%% file: main.tex
\documentclass{article}

\usepackage{microtype}
\usepackage{graphicx}
\usepackage{subfigure}
\usepackage{booktabs} 
\usepackage{pgfplots}
\usepgfplotslibrary{fillbetween}
\usetikzlibrary{patterns}
\pgfplotsset{compat=1.5.1}
\usetikzlibrary{calc}
\usepackage{nicefrac}

\usepackage{hyperref}

\usepackage[algo2e, inoutnumbered, algoruled,vlined]{algorithm2e} %

\usepackage{macros}
\usepackage{thm-restate}

\usepackage[accepted]{icml2021}

\newif\iftodo
\todofalse

\newif\ifappendix
\appendixfalse

\icmltitlerunning{On Limited-Memory Subsampling Strategies for Bandits}

\begin{document}

\twocolumn[
\icmltitle{On Limited-Memory Subsampling Strategies for Bandits}

\icmlsetsymbol{equal}{*}

\begin{icmlauthorlist}
\icmlauthor{Dorian Baudry}{equal,lille} %
\icmlauthor{Yoan Russac}{equal,ens} %
\icmlauthor{Olivier Cappé}{ens} %
\end{icmlauthorlist}

\icmlaffiliation{lille}{Univ. Lille, CNRS, Inria, Centrale Lille, UMR 9198-CRIStAL, F-59000 Lille, France}
\icmlaffiliation{ens}{DI ENS, CNRS, Inria, ENS, Université PSL, Paris, France}

\icmlcorrespondingauthor{Dorian Baudry}{dorian.baudry@inria.fr}
\icmlcorrespondingauthor{Yoan Russac}{yoan.russac@ens.fr}

\icmlkeywords{Bandits, Non-stationarity, sliding window, sub-sampling strategies}

\vskip 0.3in
]


\printAffiliationsAndNotice{\icmlEqualContribution} 

\begin{abstract}
There has been a recent surge of interest in non-parametric bandit algorithms
based on subsampling. One drawback however of these approaches is the
additional complexity required by random subsampling and the storage
of the full history of rewards. Our first contribution is to show 
that a simple deterministic subsampling rule, proposed in the recent work of
\citet{baudry2020sub} under the name of “last-block subsampling”, is 
asymptotically optimal in one-parameter exponential families. In 
addition, we prove that these guarantees also hold when limiting 
the algorithm memory to a polylogarithmic function of the 
time horizon. These findings open up new perspectives, in 
particular for non-stationary scenarios in which the arm 
distributions evolve over time. We propose a variant of the 
algorithm in which only the most recent observations are 
used for subsampling, achieving optimal regret guarantees 
under the assumption of a known number of abrupt changes. Extensive 
numerical simulations highlight the merits of this approach, particularly 
when the changes are not only affecting the means of the rewards.
\end{abstract}

\input{intro}
\input{preliminaries}
\input{lb_stationary_main}
\input{lb_nonstat_main}
\input{xp}

\section*{Acknowledgements}
The PhD of Dorian Baudry is funded by a CNRS80 grant. 

\clearpage

\bibliography{biblio}
\bibliographystyle{icml2021}

\onecolumn
\newpage
\appendix

\input{appendix_annonce}
\input{app_LBSDA_asopt}
\input{beg_proof_RBSDA}
\input{app_proof_switching}

\end{document}

%% file: intro.tex
\section{Introduction}
\label{introduction}

In the $K$-armed stochastic bandit model, the learner repeatedly picks an action
among $K$ available alternatives and only observes
the rewards associated with her actions. By interacting
with the environment, the learner aims at maximizing her expected sum
of rewards and needs to sequentially adapt her decision
strategy in light of the information gained 
up to now. 
In this model, over-confident policies are provably
suboptimal and a proper trade-off between exploitation and
exploration has to be found.

Multi-armed bandits models have been used to address a wide range of
sequential optimization tasks under uncertainty: online recommendation
\cite{li2011unbiased, li2016collaborative}, strategic pricing
\cite{bergemann1996learning} or clinical trials \cite{zelen1969play,
  vermorel2005multi} to name a few. In its standard formulation the
multi-armed bandit model postulates that the distributions of the
rewards obtained when drawing the different arms remain constant over
time. However, in some scenarios the stationary assumption is not
realistic.  In clinical trials, the disease to defeat may mutate and
the initially optimal treatment could become suboptimal compared to
another candidate \cite{gorre2001clinical}. In strategic pricing
problems, the price maximizing the profit of a given asset can evolve 
with the introduction of a new product on the market
\cite{eliashberg1986impact}. For online recommendation systems, the
preferences of the users are likely to evolve \cite{wu2018learning}
and collected data becomes progressively obsolete. 

During the past ten years, several works have considered
non-stationary variants of the multi-armed bandit model, 
proposing methods that can be grouped into 
two main categories: they either
actively try to detect modifications in the distribution of the arms
with changepoint detection algorithms \cite{liu2017change, cao2019nearly,
auer2019adaptively,chen2019new,besson2020efficient}
or they passively forget past information \cite{garivier2011upper,raj2017taming,trovo2020sliding}. To some extent,
all of these methods require some knowledge on the distribution to
obtain theoretical guarantees.

To balance exploration and exploitation, the algorithms
mentioned so far are based on one of the two standard building blocks
introduced in the bandit literature: \textit{Upper
  Confidence Bound} (UCB) constructions \cite{auer2002finite} or
\textit{Thompson Sampling} (TS) \cite{thompson1933likelihood}.
However, there has been a recent
surge of interest for alternative non-parametric bandit strategies
\cite{kveton2019perturbed,kveton2019garbage,riou2020bandit}. Instead
of using prior information on the reward distributions as in
Thompson sampling or of building tailored
upper-confidence bounds \cite{cappe2013kullback} those methods only use
the empirical distribution of the data. 
These algorithms are non-parametric in
the sense that the \emph{exact same} implementation can be used with different 
probability
distributions, while still achieving optimal regret guarantees (in a sense to be
defined in Section~\ref{sec::prelim} below).

In particular, subsampling algorithms \cite{baransi2014sub,
  chan2020multi, baudry2020sub} have demonstrated their potential
thanks to their flexibility and strong theoretical guarantees.  From a
high level perspective, they all rely on the same two components.
\textbf{(1)}~\textit{subsampling}: the arms that have been pulled a
lot are randomized by sampling only a fraction of their
history. \textbf{(2)}~\textit{duels}: the arms are pulled based on the
outcomes of duels between the different pairs of arms.  Note that the
term \textit{duel}, which we will also use in the following, refers to the
algorithmic principle of comparing the arms two by two, based on their
subsamples. It is totally unrelated to the dueling bandit framework
introduced by \citet{yue2009dueling}.

\paragraph{Scope and contributions}
In this paper, we build on the Last-Block Subsampling Duelling
Algorithm (LB-SDA) introduced by \citet{baudry2020sub} but for which no
theoretical guarantees were provided. 
This approach is
of interest because of its simplicity and its computational
efficiency compared to other strategies based on randomized
subsampling. We first prove that for stationary environments LB-SDA is
asymptotically optimal in one-parameter exponential family models and
therefore matches the guarantees obtained by \citet{baudry2020sub} for
randomized subsampling schemes. The main technical challenge is to
devise an alternative to the \textit{diversity} condition used in their work,
 which was specifically designed for randomized
subsampling schemes.
 
Furthermore, we show that, without additional changes, these guarantees
still hold for a variant of the algorithm using a \textit{limited
  memory} of the observations of each arm. We prove that
storing $\Omega\left((\log T)^2\right)$ observations instead of $T$ is
sufficient to ensure the asymptotic guarantees, making the algorithm
more tractable for larger time horizons. To the best of our knowledge,
this paper is the first to propose an asymptotically optimal
subsampling algorithm with polylogarithmic storage of rewards under
general assumptions.

Building a subsampling algorithm based on the most recent observations
makes it an ideal candidate for a passively forgetting policy.  Our
third contribution is to propose a natural extension of the LB-SDA
strategy to non-stationary environments. By limiting the extent of
the time window in which subsampling is allowed to occur, one
obtains a passively forgetting non-parametric bandit algorithm, which
we refer to as Sliding Window Last Block Subsampling Duelling
Algorithm (SW-LB-SDA). To analyze the performance of this algorithm,
we assume an abruptly changing environment in which the reward distributions
change at unknown time instants called \textit{breakpoints}.
We show that SW-LB-SDA guarantees a regret of order
$\mathcal{O}(\sqrt{\Gamma_T T \log(T)})$ for any abruptly changing
environment with at most $\Gamma_T$ breakpoints, thus matching the
lower bounds from \citet{garivier2011upper}, up to logarithmic factors.
The only required assumption is that, during each stationary phase, the
reward distributions belong to the same one-parameter exponential
family for all arms.
Due to its
non-parametric nature, this algorithm can thus be used in many
scenarios of interest beyond the standard bounded-rewards~/
change-in-the-mean framework.  We discuss some of these scenarios in
Section~\ref{sec::experiments}, where we validate numerically the
potential of the approach by comparing it with a variety of
state-of-the-art algorithms for non-stationary bandits.


%% file: preliminaries.tex
\section{Preliminaries}\label{sec::prelim}

The algorithms to be presented below are designed for the \emph{stochastic K-armed bandit} model, which is the most studied setting in the bandit literature. We introduce in this section the two variants of this basic model that will be considered in the paper: \emph{stationary} and \emph{abruptly changing} environments.

\paragraph{Stationary environments}
When the environment is stationary, the $K$ arms are characterized
by the reward distributions $(\nu_k)_{k \leq K}$ and their associated means
$(\mu_k)_{k \leq K}$, with
$\mu^\star = \max_{k \in \{1,...,K\}} \mu_k$ denoting the
highest expected reward.  We denote by $(Y_{k,s})_{s \in \mathbb{N}}$ the
i.i.d. sequence of rewards from arm $k$. Following
\citet{chan2020multi}, our algorithm operates in successive rounds, whose
length varies between 1 and $K$ time steps. At each
round $r$, the \textit{leader} denoted $\ell(r)$ is defined and $(K-1)$
duels with the remaining arms called \textit{challengers} are
performed.  Denoting by $N_k(r)$ the number of pulls of arm $k$ up to the
round $r$ the leader is the arm that has been most pulled. Namely,
\begin{equation}
\label{eq:leader_stat}
\ell(r) = \aargmax_{k \in \{1,...,K \} } N_k(r) \;.
\end{equation}
When several arms are candidate for the maximum number of
pulls, the one with the largest sum of rewards is chosen. If this is
still not sufficient to obtain a unique arm, the leader is chosen at random
among the arms maximizing both criteria.
At round $r$, a subset
$\mathcal{A}_{r} \subset \{1,...,K\}$ is selected by the learner based
on the outcomes of the duels against $\ell(r)$. Next, all arms in
$\cA_r$ are drawn, yielding $Y_{k, N_{k}(r)}$ for $k \in \cA_r$, where
$N_k(r) = \sum_{s=1}^r \ind(k \in \mathcal{A}_s)$.

The regret is defined as the expected difference between the highest expected reward and the
rewards collected by playing the sequence of arms $(A_t)_{t \leq T}$:
\[
\mathcal{R}_T =\mathbb{E}\left[ 
\sum_{t=1}^T (\mu^\star  - \mu_{A_t}) \right] \;.
\]
For distributions in one-parameter
exponential families, the lower bound of 
\citet{lai1985asymptotically} states that no strategy can
systematically outperform the following asymptotic regret lower bound
$$
\liminf_{T \to \infty} \frac{\mathcal{R}_T}{\log(T)} \geq \sum_{k: \mu_k < \mu^\star}
\frac{\mu^\star- \mu_k}{\text{kl}(\mu_k, \mu^\star)} \;.
$$

\paragraph{Abruptly changing environments}
In Section~\ref{sec:NS-lb-sda}, we consider abruptly changing environments. The number of
breakpoints up to time $T$, denoted $\Gamma_T$, is defined by
\[
  \Gamma_T = \sum_{t=1}^{T-1} \ind \{\exists k, \nu_{k,t} \neq \nu_{k,
    t+1} \}.
\]
The time instants $(t_1,..., t_{\Gamma_T})$ associated to
these breakpoints define $\Gamma_T + 1$ stationary
phases where the reward distributions are fixed. Note that in this model,
the change do not need to affect all
arms simultaneously. In such environments, letting
$\mu_t^\star = \max_{k \in \{1,...,K\}} \mu_{k,t}$ denote the best arm at
time $t$, the performance of a policy is measured through the \emph{dynamic
regret} defined as
\[
  \mathcal{R}_T = \mathbb{E} \left[ \sum_{t=1}^T (\mu_t^\star - \mu_{A_t})
  \right].
\]
We will explain how to extend the notion of leader to this
setting in~Section~\ref{sec:NS-lb-sda}.

In the non-stationary case, the lower bound for the regret takes a
different form: for any strategy, there exists an abruptly changing
instance such that
$\mathbb{E}[\mathcal{R}_T] = \Omega(\sqrt{T \Gamma_T})$
\cite{garivier2011upper,seznec2020single}. Note that in the
bandit literature, there is also another, more general, way of characterizing
non-stationary environments based on a variational distance introduced
by \citet{besbes2014stochastic}. In this work, we however only
consider the case of abruptly changing environments.


%% file: lb_stationary_main.tex
\section{LB-SDA in Stationary Environments}
\label{sec::stationnary_lbsda}
In this section we detail the subsampling strategy used in the LB-DSA
algorithm and obtain asymptotically optimal regret guarantees for its
performance. In Section~\ref{subsection:lb_sda_memory}, we consider
the variant of LB-SDA in which the memory available to the algorithm
is strongly limited.
\subsection{Last Block Sampling}
Compared to the algorithms analyzed in \cite{baudry2020sub} where the
sampler is randomized, we consider a \textit{deterministic sampler}.
At round $r$, the duel between arm $k \neq \ell(r)$ and the leader
consists in comparing the average reward from arm $k$ with the
average reward computed only from the last $N_k(r)$ observations of the leader. 
The challenger $k$ thus wins its duel if
\begin{equation}
\label{eq:winning_duel_stat}
\bar Y_{k, N_k(r)} \geq \bar Y_{\ell(r), N_{\ell(r)}(r) - N_k(r) + 1 : N_{\ell(r)}(r)} \;,
\end{equation}
where $\bar Y_{k,i:j} = \frac{1}{j-i+1}\sum_{n=i}^{j} Y_{k, n}$ denotes 
the average computed on the $j-i+1$ observations of arm $k$
between its $i$-th and $j$-th pull, and $\bar Y_{k, n}$ is a shortcut
for $\bar Y_{k, 1:n}$.

At each round, the set $\cA_{r+1}$ includes all of the
challengers that have defeated the leader, according to
Equation \eqref{eq:winning_duel_stat},
as well as under-explored arms for which $N_k(r)\leq\sqrt{\log(r)}$.
If $\cA_{r+1}$ is empty, only the leader is pulled.
Combining these elements gives LB-SDA detailed below.

\begin{algorithm}[h]
\SetKwInput{KwData}{Input}
\KwData{$K$ arms, horizon $T$}
\SetKwInput{KwResult}{Initialization}
\KwResult{$t \leftarrow 1$, $r \leftarrow 1$, $\forall k \in \{1, ..., K \}, N_k \leftarrow 0$}
\While{$t < T$}{
$\mathcal{A} \leftarrow \{\}$,  $\ell \leftarrow \text{leader}(N, Y)$ \\
\If{$r=1$}{
$\cA \leftarrow \{1, \dots, K\}$ (Draw each arm once)}
\Else{
\For{$k \neq \ell \in \{1,...,K\}$}{
\If{$N_{k} \leq \sqrt{\log(r)}$ 
\textnormal{or} $\bar Y_{k, N_k} \geq 
\bar Y_{\ell, N_{\ell}- N_{k}+1: N_{\ell} }$ }
{$\mathcal{A} \leftarrow \mathcal{A} \cup \{ k \}$}
\If{$| \mathcal{A} | = 0$}{$\mathcal{A} \leftarrow \{\ell\}$}
}
}
\For{$k \in \mathcal{A}$}{
Pull arm $k$, observe reward 
$Y_{k,N_k+1}$, $N_k \leftarrow N_k+ 1$, $t \leftarrow t+1$
}
$ r \leftarrow r+1$
}
\caption{LB-SDA}
\label{alg:lb-sda}
\end{algorithm}

\citet{baransi2014sub} propose interesting arguments explaining why
subsampling methods work. Essentially, if the sampler allows enough
\textit{diversity} in the duels, the probability of repeatedly
selecting a suboptimal arm is small. On the sampler side, this condition is
satisfied when out of a large number of duels between two arms there
is a reasonable amount of them with 
non-overlapping subsamples. 
We prove that last block sampling satisfies such property. 
The second requirement concerns the
distribution of the arms, and has
been formulated by \citet{baransi2014sub} who
introduced the \textit{balance function} of
a family of distributions. In 
particular, \citet{chan2020multi} shows
that introducing an asymptotically negligible
sampling obligation of $\sqrt{\log r}$ is enough
to make subsampling suitable when the arms
come from the same one-parameter exponential
family of distributions.
Namely, if each arm has at
least $\sqrt{\log r}$ samples at round $r$, 
the \textit{diversity} of duels will guarantee
 each arm to be pulled enough. 
This exploration rate does
not have to be tuned and is not detrimental
in practice : for an horizon of, say, $T=10^6$
it only forces each arm to be sampled at least $4$ times.
%
\subsection{Regret Analysis of LB-SDA}
\label{subsection::regret_analysis_stat}
We consider that the arms come from the same one-parameter exponential
family of distributions $\cP_{\Theta}$, i.e., that there exists a function
$g: \R \times \Theta \mapsto \R$ such that any arm $k$ has a density
of the form
\[g_k(x) = g(x, \theta_k) = e^{\theta_k x - \Psi(\theta_k)} g(x, 0)
  \;,\]
where
$\Psi(\theta_k) = \log \left[\int e^{\theta_k x} g(x, 0) \,dx\right]$.
This assumption is standard in the literature and covers a broad range
of bandits applications. The exact knowledge of the family of distributions of the arms (e.g
Bernoulli, Gaussian with known variance, Poisson, etc.) 
can be used to calibrate
algorithms like Thompson Sampling \cite{TS_Emilie}, KL-UCB
\cite{cappe2013kullback} or IMED \cite{Honda15IMED} in order to reach
asymptotic optimality.  Recently, subsampling algorithms like SSMC
\cite{chan2020multi} and RB-SDA \cite{baudry2020sub} have been proved
to be optimal \textit{without} knowing exactly $\cP_\Theta$. This
means that the same algorithm can run on Bernoulli or Gaussian
distributions  and achieve optimality.  
We first prove that
LB-SDA matches these theoretical guarantees.
We denote $\text{kl}(\mu,\mu')$ the Kullback-Leibler divergence 
between two distributions of mean $\mu$ and $\mu'$ in the exponential family $\cP_{\Theta}$.

\begin{theorem}[Asymptotic optimality of LB-SDA]
\label{th::regret_lbsda}
	For any bandit model $\nu=(\nu_1, \dots, \nu_K) \subset \cP_{\Theta}^K$ 
	where $\cP_\Theta$ is any one-parameter exponential family 
	of distributions, the regret of LB-SDA satisfies, for all $\epsilon >0$,
	$$
	\cR(T) \leq \sum_{k: \mu_k < \mu^\star} \frac{1+\epsilon}{
	\kl(\mu_k, \mu^\star)} \log(T) + C(\nu, \epsilon) \;,
	$$
where $C(\nu, \epsilon)$ is a problem-dependent constant.
\end{theorem}
\paragraph{Proof sketch}
We assume without loss of generality that there is a unique optimal
arm denoted $k^\star$. The analysis of \citet{chan2020multi} and
\citet{baudry2020sub} shows that for any SDA algorithm the number of
pulls of a suboptimal arm may be bounded as follow.

\begin{lemma}[Lemma 4.1 in \citet{baudry2020sub}]
\label{lem::dec_statio}
	For any suboptimal arm $k \neq k^\star$,
	 the expected number of pulls of $k$ is upper bounded by
	 \begin{multline}
	\bE[N_k(T)] \leq \frac{1+\epsilon}{\kl(\mu_k, \mu^\star)} \log(T) + C_k(\nu, \epsilon) \\
	+ 32 \sum_{r=1}^T \bP(N_{k^\star}(r)\leq (\log r)^2)\;,
	\end{multline}
	where $C_k(\nu, \epsilon)$ is a problem-dependent constant.
\end{lemma}

The next step consists in upper bounding the probability that the best
arm is not pulled "enough" during a run of the algorithm.  This part
is more challenging and relies on the notion of \textit{diversity} in
the subsamples provided by the subsampling algorithm. This notion
was introduced by \citet{baransi2014sub} to analyze the Best Empirical
Sampled Average (BESA) algorithm. 
Intuitively, random
block sampling \cite{baudry2020sub} or sampling without replacement
\cite{baransi2014sub} explore different part of the history thus bringing
diversity in the duels.
Unfortunately, this property is not satisfied by deterministic
samplers.
Nonetheless, with a careful examination of the relation implied by the
deterministic nature of last-block subsampling it is possible to
prove that the number of pulls of the optimal arm is large enough with
high probability.

\begin{restatable}{lemma}{lemmacontrolpulllbsda}
The probability that the optimal arm is not pulled enough by LB-SDA
can be upper bounded as follows
\[\sum_{r=1}^{+\infty} \bP\left(N_{k^\star}(r) \leq (\log r)^2\right) \leq C_{k^\star}(\nu) \;,\]
for some constant $C_{k^\star}(\nu)$.
\label{lem::control_pull_lbsda_s}
\end{restatable}

Plugging the result of Lemma~\ref{lem::control_pull_lbsda_s} in 
Lemma~\ref{lem::dec_statio} gives the asymptotic optimality of LB-SDA 
(Theorem~\ref{th::regret_lbsda}).
The proof of Lemma~\ref{lem::control_pull_lbsda_s} is reported in 
Appendix~\ref{app::proof_s}.
%
\subsection{Memory-Limited LB-SDA}
\label{subsection:lb_sda_memory}
One of our main motivations for studying LB-SDA is its simplicity and
efficiency. Yet, all existing subsampling algorithms
\cite{baransi2014sub, chan2020multi, baudry2020sub} as well as the
vanilla version of LB-SDA have to store the entire history of rewards
for all the arms.
In this section, we explain how to modify LB-SDA to reduce the storage
cost while preserving the theoretical guarantees.

The fact that LB-SDA is asymptotically optimal means that, when $T$ is
large, the arm with the largest mean is most often the leader with all
of its challengers having a number of pulls that is of order
$O(\log T)$ only. With duels based on the last block, this would mean in
particular that only the last $O(\log T)$ observations from the optimal
arm should be stored and that previous observations will \textit{never}
be used again in practice.  Based on this intuition, one might
think that keeping only $\log(T)/(\mu^\star- \mu_k)^2$ observations is
enough for LB-SDA. However, this could only be done with the knowledge of
the gaps that are unknown.

We propose instead to limit the storage memory of each arm at round
$r$ to a value of the form
\[m_r = \max\left(M, \left\lceil C(\log r)^2\right\rceil\right) \;,
\] 
where
$C>0$ and $M \in \N$. $M$ ensures that a minimum number of samples are
stored during the first few rounds.
Following the definition of \citet{agrawal2012analysis}, we then define the set of
\textit{saturated arms} at a round $r$ as
\[
\cS_r = \{k \in \K: N_k(r)\geq m_r\} \;.
\]
The only modification of LB-SDA is the following: at each round
$r$, if a saturated arm is pulled then the newly collected observation
replaces the oldest observation in its history. The
pseudo code of LB-SDA with Limited Memory (LB-SDA-LM) is given in Appendix
\ref{app::lbsda_limited_memory} and the following result shows that it keeps
the same asymptotical performance as LB-SDA under general assumptions on $m_r$.

\begin{theorem}[Asymptotic optimality of LB-SDA with Limited Memory]\label{th::memory_limited_regret}
	For any bandit model 
	$\nu=(\nu_1, \dots, \nu_K) \subset \cP_{\Theta}^K$
	where $\cP_\Theta$ 
	is any one-parameter exponential family of distributions,
	if $m_r/\log(r) \to \infty$, the 
	regret of memory-limited LB-SDA
	satisfies, for all $\epsilon >0$,
	\[\cR_T \leq \sum_{k: \mu_k < \mu^\star} 
	\frac{1+\epsilon}{\kl(\mu_k, \mu^\star)}
	 \log(T) + C'(\nu, \epsilon, \cM) \;,\]
	where $\cM=(m_1, m_2, \dots, m_T)$ 
	denotes the sequence $(m_r)_{r \in \N}$ 
	and $C'(\nu, \epsilon, \cM)$ is a problem-dependent constant.
\end{theorem}

The proof of this theorem is reported in
Appendix~\ref{app::lbsda_limited_memory}, which provides precise estimates of
the dependence of $C'(\nu, \epsilon, \cM)$ with respect to the parameters, and in particular, with respect to the sequence $\cM$. Note that LB-SDA-LM
remains an anytime algorithm because the storage constraint does not
depend on the time horizon $T$ but only on the current round.
%
\subsection{Storage and Computational Cost}
To the best of our knowledge, LB-SDA-LM is the only subsampling bandit
algorithm that does not require to store the
full history of rewards. We report in Table~\ref{tab:samp}
estimates of the computational cost of  LB-SDA-LM and its competitors.
\begin{table}[h]
	\caption{Storage and computational 
	cost at round $T$ for existing subsampling algorithms.}
	\label{tab:samp}
	\begin{center}
		\begin{small}
			\begin{tabular}{ccc}
				\toprule
				Algorithm & Storage & Comp. cost\\
				& & Best-Worst case\\ 
				\midrule
				\begin{tabular}{c} BESA \\
					\cite{baransi2014sub}\end{tabular}
				&  $O(T)$ & $O((\log T)^2)$ \\
				\midrule
				\begin{tabular}{c} SSMC \\
					\cite{chan2020multi}\end{tabular}
				& $O(T)$ & {\color{blue} $O(1)$}-{\color{red} $O(T)$} \\
				\midrule
				\begin{tabular}{c} RB-SDA \\ \cite{baudry2020sub}\end{tabular}
				& $O(T)$ & $O(\log T)$\\
				\midrule
				\begin{tabular}{c} LB-SDA \\ \textbf{(this paper)} \end{tabular}
				& $O(T)$ & {\color{blue}$O(1)$-$O(\log T)$} \\
				\midrule
				\begin{tabular}{c} LB-SDA-LM \\ \textbf{(this paper)} \end{tabular} &  
				\textbf{$O((\log T)^2)$}
				& {\color{blue}$O(1)$-$O(\log T)$}  \\
				\bottomrule
			\end{tabular}
		\end{small}
	\end{center}
\end{table}

The computational cost
can be broken into two parts: (a) the subsampling cost 
and (b) the computation of the means of the samples.
We assume that drawing a sample of size $n$ without 
replacement has $O(n)$ cost and that computing 
the mean of this subsample costs another $O(n)$.
Furthermore, at round $T$, each challenger to the best arm
has about $O(\log T)$ samples. This gives an estimated cost of
$O\left((\log T)^2\right)$ for BESA \cite{baransi2014sub}. For 
RB-SDA \cite{baudry2020sub} the estimated cost is 
$O(\log(T))$, because the sampling cost for random block 
sampling is $O(1)$ and only the sample mean has to be recomputed at each round.

For the three deterministic algorithms 
 (namely SSMC \cite{chan2020multi}, LB-SDA, LB-SDA-LM), when
the leader arm wins all its duels, its sample mean can
be updated sequentially at cost $O(1)$.
This is the \textit{best case} in terms of computational 
cost. However, when a challenger arm is pulled,
SSMC requires a full screening of the 
leader's history, with $O(T)$ cost, while
LB-SDA and LB-SDA-LM only need the computation of the mean
of the last $O(\log T)$ samples from the leader.

%% file: lb_nonstat_main.tex
\section{LB-SDA in Non-Stationary Environments}
\label{sec:NS-lb-sda}
In stationary environments,
LB-SDA achieves optimal regret rates, even when
its decisions are constrained to use at most $O((\log T)^2)$
observations. One might think that this argument itself is sufficient
to address non-stationary scenarios as the duels are performed mostly
using recent observations. However, the latter is only true for the best
arm and in the case where an arm that has been bad for a long period of time
suddenly becomes the best arm, adapting to the change would still be
prohibitively slow.
For this reason, LB-SDA has to be equipped with an
additional mechanism to perform well in non-stationary environments.
%
\subsection{SW-LB-SA: LB-SDA with a Sliding-Window}
\begin{algorithm}[h]
\SetKwInput{KwData}{Input}
\KwData{$K$ arms, horizon $T$, $\tau$ length of sliding window}
\SetKwInput{KwResult}{Initialization}
\KwResult{$t \leftarrow 1$, $r \leftarrow 1$, $\forall k \in \{1, ..., K \}, N_k \leftarrow 0$, $N_k^\tau \leftarrow 0$}
\While{$t < T$}{
$\mathcal{A} \leftarrow \{\}$,  $\ell \leftarrow \text{leader}(N,  Y, \tau)$ \\
\If{$r=1$}{
$\cA \leftarrow \{1, \dots, K\}$ (Draw each arm once)
}
\Else{
\For{$k \neq \ell \in \{1,...,K\}$}{
\If{$N_{k}^\tau \leq \sqrt{\log(\tau)} $ \text{or} $D_k^\tau(r) = 1$}{$\mathcal{A} \leftarrow \mathcal{A} \cup \{ k \}$}
\Else{
\hspace{0.1cm}
$\widehat{\mu}_k^\tau = \bar{Y}_{k, N_k- N_k^\tau +1: N_k}$\\
$N = \min(N_k^\tau, N_\ell^\tau )$\\
$\widehat \mu_{\ell,k}^\tau =\bar Y_{N_\ell -N+1: N_\ell}$\\
\If{$\widehat{\mu}_k^\tau \geq \widehat \mu_{\ell,k}^\tau$}
{$\mathcal{A} \leftarrow \mathcal{A} \cup \{ k \}$}
}
}
\If{$| \mathcal{A} | = 0$}{$\mathcal{A} \leftarrow \{\ell \}$}
}
\For{$k \in \mathcal{A}$}{
\hspace{0.1cm} Pull arm $k$, observe reward $Y_{k,N_k +1}$\\
Update $N_k \leftarrow N_k + 1$, $N_k^\tau \leftarrow N_k^\tau + 1$, $t \leftarrow t+1$
}
\For{$k \in \{1,...,K\}$}{
\If{$k \in \mathcal{A}_{r-\tau+1}$}
{$N_k^\tau \leftarrow N_k^\tau - 1$}
}
$ r \leftarrow r+1$
}
\caption{SW-LB-SDA}
\label{alg:duel_SW}
\end{algorithm}
We keep a \textit{round-based} structure for the algorithm, where, at
each round $r$, duels between arms are performed 
and the algorithm subsequently selects the subset of arms
$\cA_r$ that will be pulled.
In contrast to Section \ref{subsection:lb_sda_memory}, where a constraint on
storage related to the number of pulls was added, here, we use a
sliding window of length $\tau$ to limit the historical data available to
the algorithm to that of the last $\tau$ rounds.

\input{figure}

\paragraph{Modified leader definition} The introduction of a sliding window
requires a new definition for the \textit{leader}. By analogy with the
stationary case, the leader could be defined as the arm that has been pulled the
most during the $\tau$ last rounds.
However, with the inclusion of the sliding window, a new
phenomenon, which we call \textit{passive leadership takeover}, can
occur. Let us define $N_k^\tau(r) = \sum_{s=r-\tau}^{r-1} \ind\left(k \in
  \cA_{s+1}\right)$, the number of times arm $k$ has been pulled
during the last $\tau$ rounds and
consider a situation with 3 arms $\{1,2,3\}$. Assume that the
leader is arm $1$ and at a round $(r-1)$ we have
$N_1^\tau(r-1) = N_2^\tau(r-1)$.  If the leader has been pulled $\tau$
rounds away and wins its duel against arm 2 but looses
against arm $3$, only arm 3 will be pulled at round $r$.
Consequently, at round $r$, arm $2$ will have a strictly larger
number of pulls than arm $1$ without having actually defeated the leader.
This situation, illustrated on Figure~\ref{fig:test}, is not desirable
as it can lead to spurious leadership changes.  We fix
this by imposing that any arm has to defeat the current leader to
become the leader itself.  Define,
\[
\cB_r = \left\{k \in \cA_{r+1}\cap \{
N_k^\tau(r+1) \geq \min(r,\tau)/K\} \right\} \;.
\]
Then for any $r \in \N$, the leader at
round $r+1$ is defined as
$\ell^\tau(r+1) = \aargmax_{k \in \{1,...,K\}} N_k^\tau(r+1)$ if
$ N^{\tau}_{\ell^\tau(r)}(r+1)< \min(r,\tau)/(2K)$ and the argmax is
taken over $\cB_r\cup \{\ell^\tau(r) \}$ otherwise.  This modified
definition of the leader ensures that an arm can become the leader
only after earning at least $\tau/K$ samples and winning a duel
against the current leader, or if the leader loses a lot of duels and
its number of samples falls under a fixed threshold.  Thanks to this
definition it holds that
$N_{\ell^\tau(r) }^\tau(r) \geq \min(r,\tau)/(2K)$. More details are
given in Appendix~\ref{app::NS_lbsda}.

\paragraph{Additional diversity flags}
As in the vanilla LB-SDA, we use a sampling obligation
to ensure that each arm has a minimal number of samples.
However, in contrast to the stationary case,
this very limited number of forced samples may not
be sufficient to guarantee an adequate variety of duels, due to the forgetting window.
To this end, the sampling obligation is coupled with
a \textit{diversity flag}.
We define it as a binary random variable $D_k^\tau(r)$,
satisfying $D_k^\tau(r)=1$ only when, for the last
$\lceil (K-1) (\log \tau)^2 \rceil$ rounds the three following conditions are satisfied:
1) some arm $k'\neq k$ has been leader during all these rounds,
2) $k'$ has not been pulled, and
3) $k$ has not been pulled and satisfy $N_{k}^\tau(r)\leq (\log \tau)^2$.
In practice, there is a very low probability that these conditions are
met simultaneously but this additional mechanism is required for the
theoretical analysis. Note that the diversity flags have no impact on
the computational cost of the algorithm as they require only
to store the number of rounds since the last draw of the
different arms (which can be updated recursively) as well as the
last leader takeover. Arms that raise their diversity flag are automatically
added to the set of pulled arms.

Bringing these parts together, gives the pseudo-code of SW-LB-SDA in
Algorithm~\ref{alg:duel_SW}.
%
\subsection{Regret Analysis in Abruptly Changing Environments}
In this section we aim at upper bounding the dynamic regret in
abruptly changing environments, as defined in
Section~\ref{sec::prelim}. 
Our main result is the proof that the regret of SW-LB-SDA
matches the asymptotic lower bound of \citet{garivier2011upper}.
\begin{theorem}[Asymptotic optimality of SW-LB-SDA]
\label{th::regret_SWLB}
	If the time horizon $T$ and number of breakpoint $\Gamma_T$ are known, 
	choosing $\tau=O(\sqrt{T \log (T) /\Gamma_T})$ 
	ensures that the dynamic regret of SW-LB-SDA satisfies
	\[ \cR_T =O(\sqrt{T\Gamma_T \log T)} \;.\]
\end{theorem}
To prove this result we only need to assume that, during each stationary
period, the rewards come from the same one-parameter exponential family
of distributions. In contrast, current state-of-the-art algorithms for
non-stationary bandits typically require the assumption that the
rewards are \textit{bounded} to obtain similar guarantees. Hence, this
result is of particular interest for tasks involving unbounded reward
distributions that can be discrete (e.g Poisson) or continuous (e.g
Gaussian, Exponential). SW-LB-SDA can also be used for general bounded
rewards with the same performance guarantees by using the
\textit{binarization trick} \cite{agrawal2013further}. Note however,
that the knowledge of the horizon $T$ and the estimated number of
change point $\Gamma_T$ is still required to obtain optimal rates, which is
an interesting direction for future works on this approach
\cite{auer2019adaptively, besson2020efficient}.
We provide a high-level outline of the analysis behind Theorem~\ref{th::regret_SWLB} and
the complete proof is given in Appendix~\ref{app::NS_lbsda}.

\paragraph{Regret decomposition} For
the $\Gamma_T+1$  stationary phases
$[t_\phi, t_{\phi+1}-1]$ with $\phi \in \{1, \dots, \Gamma_T\}$,
we 
define $r_\phi$ as the first round where an observation 
from the phase $\phi$ was pulled.
Introducing the gaps
$\Delta_k^\phi=\mu_{t_\phi}^*-\mu_{t_\phi, k}$ and
denoting the optimal arm $k_\phi^\star$, we can
rewrite the regret as
\begin{align*}
\mathcal{R}_T &= 
\bE\left[\sum_{\phi=1}^{\Gamma_T} \sum_{r=r_{\phi}-1}^{r_{\phi+1}-2}
\sum_{k \neq k_\phi^\star} \ind\left(k \in \cA_{r+1} \right)
\Delta_{k}^\phi\right] \\
&= \sum_{\phi=1}^{\Gamma_T}  
\sum_{k \neq \best}\bE [N_k^\phi] \Delta_{k}^\phi \;,
\end{align*}
where we define
$N_k^\phi = \sum_{r=r_{\phi}-1}^{r_{\phi+1}-2} \ind(k \in \cA_{r+1})$
the number of pulls of an arm $k$ during a phase $\phi$ when it is
suboptimal. 

Note that the quantities $t_\phi$, $r_\phi$ and $\Delta_k^\phi$
for the different stationary phases $\phi$
are only required for the theoretical analysis 
and the algorithm has no access to those values.
We highlight that the sequence $(r_\phi)_{\phi \geq 1}$ is a random variable that 
depends on the trajectory of the algorithm. 
However, we show in
Appendix~\ref{app::NS_lbsda} that this causes no additional difficulty
for upper bounding the regret.
We introduce
$\delta_\phi = t_{\phi+1}-t_\phi $
the length of a phase $\phi$. 
Combining elements from the proofs of 
\citet{garivier2011upper} and that of 
Theorem~\ref{th::regret_lbsda}, we first provide an upper 
bound on $\bE[N_k^\phi]$ for any suboptimal arm $k$ during the phase $\phi$ as
\[
	\bE[N_k^\phi] \leq 2 \tau + \frac{\delta_\phi A_k^{\phi,\tau}}{\tau} + c_{k,1}^{\phi, \tau}
+ c_{k,2}^{\phi, \tau} + c_{k,3}^{\phi, \tau} \;.
\]
In this decomposition we define $A_k^{\phi,\tau} = b_{k}^\phi \log(\tau)$
for some constant $b_{k}^\phi > 0$,
along with the terms $c_{k,1}^{\phi,\tau}$,
$c_{k,2}^{\phi,\tau}$ and $c_{k,3}^{\phi,\tau}$, which
all represents a different technical aspect of the regret
decomposition of SW-LB-SDA.
Before interpreting them we start with their formal definition,
\begin{align*}
	 c_{k,1}^{\phi,\tau} &= \bE\left[\sum_{r=r_\phi+2\tau-2}^{r_{\phi+1}-2} \ind\left(\cG_k^\tau(r, A_k^{\phi,\tau}) \right)\right] \;, \\
	 c_{k,2}^{\phi,\tau} &= \bE\left[\sum_{r=r_\phi+2\tau-2}^{r_{\phi+1}-2} \ind\left(\ell^\tau(r) = \best, D_k^\tau(r) =1 \right) \right] \;,  \\
	 c_{k,3}^{\phi, \tau} &= \bE\left[\sum_{r=r_\phi+2\tau-2}^{r_{\phi+1}-2}
	 \ind\left(\ell^\tau(r)\neq \best \right) \right] \;,
\end{align*}
where $\cG_k^{\tau}(r,n)$ is equal to
\[
\{k \in \cA_{r+1}, \ell^{\tau}(r)=
\best, N_k^{\tau}(r)\geq n, D_k^\tau(r) =0 \} \;.
\]
\paragraph{Bounding individual terms} The three
terms have intuitive interpretation
and summarize well the technical contributions
behind
Theorem~\ref{th::regret_SWLB}. To some extent they all
rely on the notion of \textit{saturated} arms 
defined in Section~\ref{subsection:lb_sda_memory}
and that we refine in Appendix~\ref{app::NS_lbsda} for the problems
considered in this section (mainly by properly tuning $A_k^{\phi,\tau}$ 
in the theoretical analysis).

First, $c_{k,1}^{\phi,\tau}$ is an upper bound on the expectation
of the number of times a \textit{saturated suboptimal arm}
can defeat the \textit{optimal} leader (i.e $\ell^\tau(r)=\best$).
To prove this result we establish a new concentration inequality
for Last-Block Sampling in the context of SW-LB-SDA.

The second term $c_{k,2}^{\phi,\tau}$ controls
the probability that the \textit{diversity flag}
is activated when the optimal arm $\best$ is the
leader. We prove that if this event happen,
then $\best$ has necessarily lost at least one
duel against a saturated \textit{sub-optimal}
arm, and that this event has only a low probability.

The term $c_{k,3}^{\phi,\tau}$ is the most difficult to handle,
the main challenge is to upper bound the probability
that the \textit{optimal arm} is \textit{not saturated}
after a large number of rounds. 

In Appendix~\ref{app::NS_lbsda} we provide the complete
analysis of each of these terms and a full description
of all the technical results that led to Theorem~\ref{th::regret_SWLB}.

%% file: figure.tex
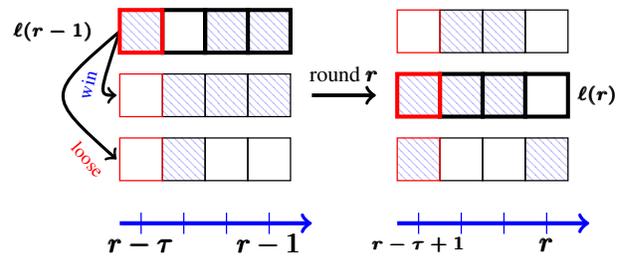
\begin{figure}[!h]
\centering
    \vspace{-30pt}
\scalebox{1.18}{
    \begin{tikzpicture}
     \begin{axis}[
     	axis equal,
        axis x line=none,
        axis y line=none,
        xtick=\empty,
        ytick=\empty,
        scaled ticks=false,
        xmin=-28,
        xmax=29,
        ymin=-20,
        ymax=0,
        xlabel=,
        ylabel=,
   ]
   
	\draw[->, color=blue, line width=1.3] (axis cs:-18,-16)  -- (axis cs:0,-16);
	\draw[color=blue] (axis cs:-16,-17)  -- (axis cs:-16,-15);
	\node[align=right, rotate=0] at (axis cs: -16,-18) {\footnotesize $\pmb{r- \tau} $};
	\draw[color=blue] (axis cs:-12,-17)  -- (axis cs:-12,-15);
    \draw[color=blue] (axis cs:-8,-17)  -- (axis cs:-8,-15);
    \draw[color=blue] (axis cs:-4,-17)  -- (axis cs:-4,-15);
    \node[align=right, rotate=0] at (axis cs: -4,-18) {\footnotesize $\pmb{r-1} $};a

	\draw[color=black, ->,line width=1]    (axis cs: -18,2)   .. controls (axis cs:-25,-4) .. (axis cs: -18.3,-10);
	\node[color=red, rotate=-50] at (axis cs:-21,-10) {\tiny loose}; 	
	\draw[color=black, ->,line width=1]    (axis cs: -18,2)   .. controls (axis cs:-20,-4) .. (axis cs: -18.3,-4);
	\node[color=blue, rotate = 75] at (axis cs: -21, -3) {\tiny win};
	
	\node[align=left] at (axis cs:-24,2) {\tiny $\pmb{\ell(r-1)} $ };
	\node[align=right] at (axis cs:27,-4) {\tiny $\pmb{\ell(r)} $ };
    
     \fill[pattern=north west lines, pattern color=blue, opacity=0.5] 
  (axis cs:-14,-12) rectangle (axis cs: -10, -8);
    \draw[color=black] (axis cs: -14, -12) rectangle (axis cs: -10, -8); 
	\draw[color=red] (axis cs: -18, -12) rectangle (axis cs: -14, -8); 
	\draw[color=black] (axis cs: -10, -12) rectangle (axis cs: -6, -8); 
	\draw[color=black] (axis cs: -6, -12) rectangle (axis cs: -2, -8); 
	\draw[color=black] (axis cs: -14, -6) rectangle (axis cs: -10, -2); 
	 \fill[pattern=north west lines, pattern color=blue, opacity=0.5] 
      (axis cs:-14,-6) rectangle (axis cs: -2, -2);
	\draw[color=red] (axis cs: -18, -6) rectangle (axis cs: -14, -2); 
	\draw[color=black] (axis cs: -10, -6) rectangle (axis cs: -6, -2); 
	\draw[color=black] (axis cs: -6, -6) rectangle (axis cs: -2, -2); 
	 \fill[pattern=north west lines, pattern color=blue, opacity=0.5] 
      (axis cs:-18,0) rectangle (axis cs: -14, 4);
      \fill[pattern=north west lines, pattern color=blue, opacity=0.5] 
      (axis cs:-10,0) rectangle (axis cs: -2, 4);
	\draw[color=black,line width=1.3]
	 (axis cs: -14, 0) rectangle (axis cs: -10, 4); 
	\draw[color=red,line width=1.3] (axis cs: -18, 0) rectangle (axis cs: -14, 4); 
	\draw[color=black,line width=1.3] (axis cs: -10, 0) rectangle (axis cs: -6, 4); 
	\draw[color=black,line width=1.3] (axis cs: -6, 0) rectangle (axis cs: -2, 4); 

	 \draw[->, color=black, line width=1.2] (axis cs:0,-4)  -- (axis cs: 6,-4);
	 \node[align=right] at (axis cs:3,-2)  {\scriptsize round $\pmb{r}$};

	\draw[->, color=blue, line width=1.3] (axis cs:8,-16)  -- (axis cs:26,-16);
	\draw[color=blue] (axis cs:10,-17)  -- (axis cs:10,-15);
	\node[align=right, rotate=0] at (axis cs: 10,-18) {\tiny $\pmb{r-\tau+1} $};
	\draw[color=blue] (axis cs:14,-17)  -- (axis cs:14,-15);
    \draw[color=blue] (axis cs:18,-17)  -- (axis cs:18,-15);
    \draw[color=blue] (axis cs:22,-17)  -- (axis cs:22,-15);
    \node[align=right, rotate=0] at (axis cs: 22,-18) {\footnotesize $\pmb{r} $};

	\fill[pattern=north west lines, pattern color=blue, opacity=0.5] 
      (axis cs:8,-12) rectangle (axis cs: 12, -8);
    \fill[pattern=north west lines, pattern color=blue, opacity=0.5]
      (axis cs:20,-12) rectangle (axis cs: 24, -8);
	\draw[color=black] (axis cs: 12, -12) rectangle (axis cs: 16, -8); 
	\draw[color=red] (axis cs: 8, -12) rectangle (axis cs: 12, -8); 
	\draw[color=black] (axis cs: 16, -12) rectangle (axis cs: 20, -8); 
	\draw[color=black] (axis cs: 20, -12) rectangle (axis cs: 24, -8); 
	\fill[pattern=north west lines, pattern color=blue, opacity=0.5] 
      (axis cs:8,-6) rectangle (axis cs: 20, -2);
	\draw[color=black,line width=1.3] (axis cs: 12, -6) rectangle (axis cs: 16, -2); 
	\draw[color=red,line width=1.3] (axis cs: 8, -6) rectangle (axis cs: 12, -2); 
	\draw[color=black,line width=1.3] (axis cs: 16, -6) rectangle (axis cs: 20, -2); 
	\draw[color=black,line width=1.3] (axis cs: 20, -6) rectangle (axis cs: 24, -2); 

	 \fill[pattern=north west lines, pattern color=blue, opacity=0.5] 
      (axis cs:12,0) rectangle (axis cs: 20, 4);
	\draw[color=black] (axis cs: 12, 0) rectangle (axis cs: 16, 4); 
	\draw[color=red] (axis cs: 8, 0) rectangle (axis cs: 12, 4); 
	\draw[color=black] (axis cs: 16, 0) rectangle (axis cs: 20, 4); 
	\draw[color=black] (axis cs: 20, 0) rectangle (axis cs: 24, 4); 
    \end{axis}
    \end{tikzpicture}}
    \vspace{-70pt}
    \caption{Illustration of a \textit{passive leadership takeover} with a sliding window 
    $\tau=4$ when the standard definition of leader is used. The bold rectangle correspond to the leader. A blue square is added when an arm has an observation for the corresponding round and the red square correspond to the information that will be lost at the end of the round due to the sliding window.}
    \label{fig:test}
\end{figure}

%% file: xp.tex
\section{Experiments}
\label{sec::experiments}

\paragraph{Limiting the storage in stationary environments.}
In our first experiment\footnote{The code for obtaining the different figures reported in the paper is available at \url{https://github.com/YRussac/LB-SDA}.} reported on Figure~\ref{fig::bernoulli_stat}, 
we compare LB-SDA and LB-SDA-LM on a stationary instance with $K=2$
arms with Bernoulli distributions for a horizon $T=10000$. We add natural competitors
(Thompson Sampling \cite{thompson1933likelihood}, kl-UCB \cite{cappe2013kullback}),
that know ahead of the experiment that the reward distributions are Bernoulli
and are tuned accordingly. 
The arms satisfy $(\mu_1, \mu_2) = (0.05, 0.15)$ with a gap $\Delta = 0.1$.
We run LB-SDA-LM with a memory limit 
$m_r = \log(r)^2 + 50$, which gives a storage 
ranging from 50 to 150 samples (much smaller than the horizon T = 10000).
The regret are averaged on 2000 independent 
replications and the upper and lower quartiles are reported.
In this setup LB-SDA-LM performs similarly to KL-UCB, 
and the impact of limiting the memory is mild, when compared to LB-SDA. 
This illustrates that even with relatively 
small gaps (here 0.1), a substantial reduction
 of the storage can be done with only minor loss of performance with LB-SDA-LM.
 
 \begin{figure*}[hbt]
        \centering
        \includegraphics[width=0.4\textwidth]{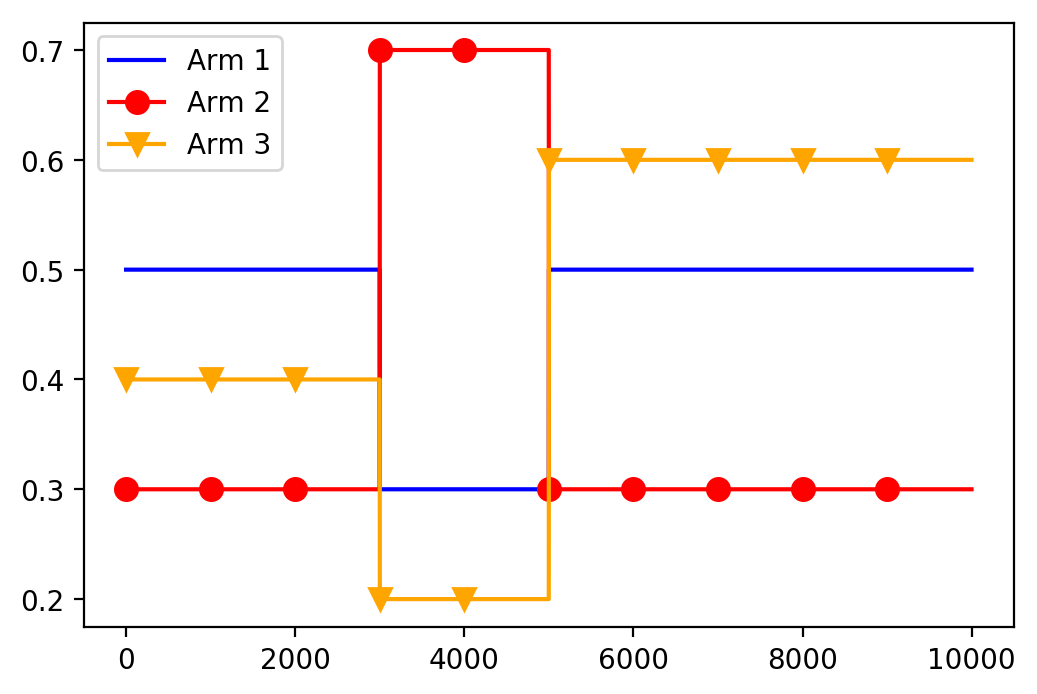} \qquad \qquad \includegraphics[width=0.4\textwidth]{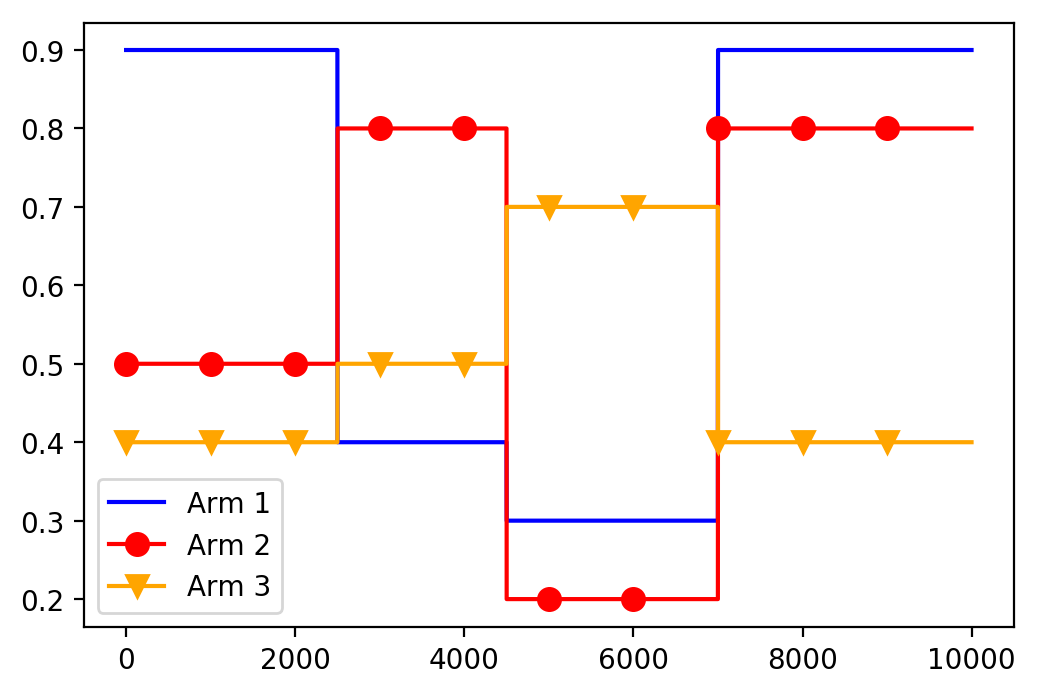}
        \caption{Evolution of the means: Left, Bernoulli arms (Fig.~\ref{fig::bernoulli_exp}); Right, Gaussian arms (Figs.~\ref{fig::100_rep_gauss_variance_0_5} and~\ref{fig::100_rep_gauss_variance_evolving}).}
        \label{fig::bernoulli_distrib_exp}
        \label{fig::distribution_arms}
\end{figure*}

\begin{figure}[hbt]
        \centering
        \includegraphics[width=0.45\textwidth]{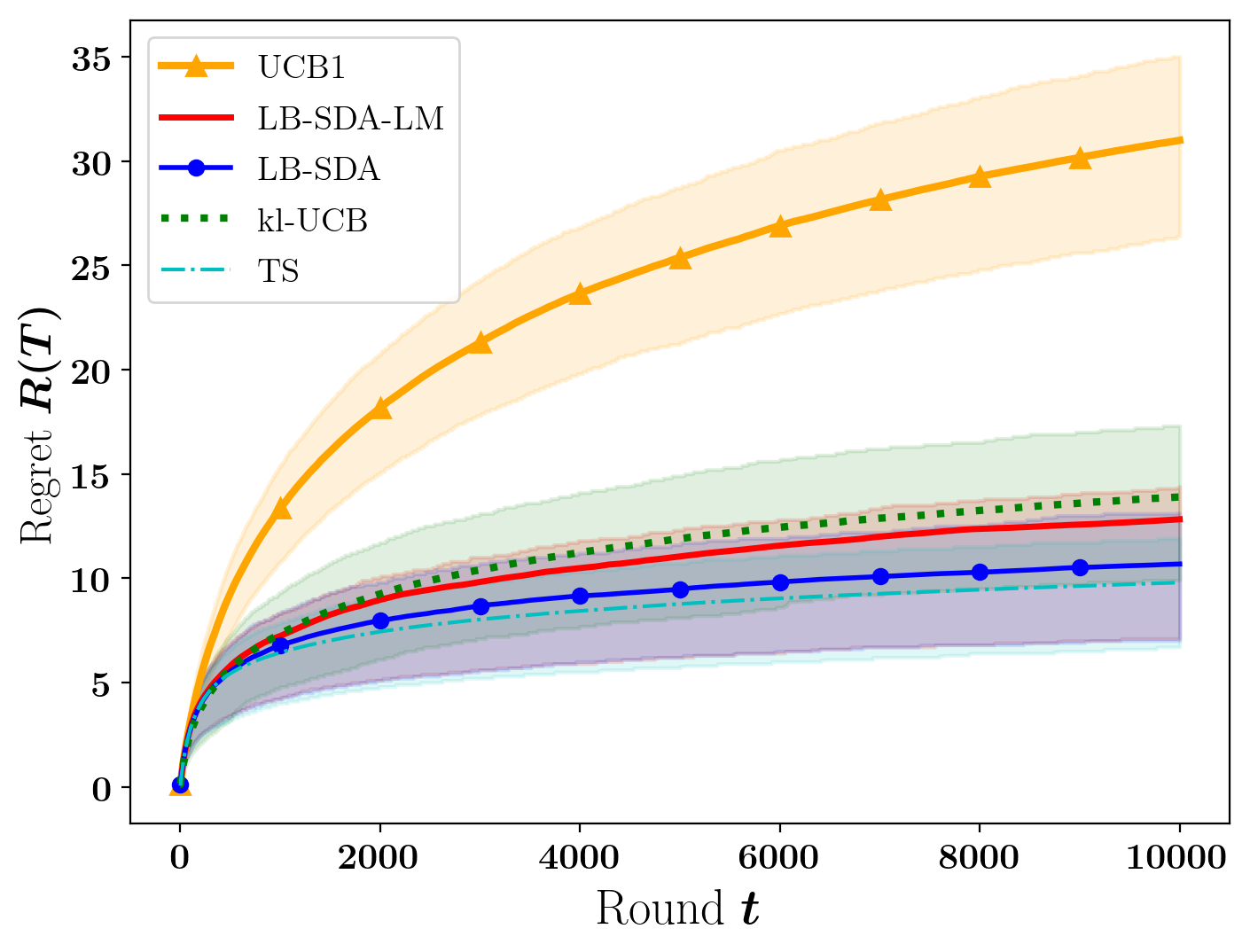}
        \caption{Cost of storage limitation on a Bernoulli instance. 
        The reported regret are averaged over 2000 independent replications.}
                \label{fig::bernoulli_stat}
\end{figure}

\paragraph{Empirical performance in abruptly changing environments.}
In the second experiment, we compare different state-of-the-art
algorithms on a problem with $K=3$ Bernoulli-distributed arms. The
means of the distributions are represented on the left hand side of
Figure~\ref{fig::bernoulli_distrib_exp} and the performance 
averaged on 2000 independent replications are
reported on Figure~\ref{fig::bernoulli_exp}. Two changepoint detection
algorithms, CUSUM \cite{liu2017change} and M-UCB \cite{cao2019nearly}
are compared with progressively forgetting policies based on upper
confidence bound, SW-klUCB and 
D-klUCB adapted from \citet{garivier2011upper}, or Thompson
sampling, DTS \cite{raj2017taming} and SW-TS \cite{trovo2020sliding}.
We also add EXP3S \cite{auer2002finite} designed for adversarial
bandits and our SW-LB-SDA algorithm for the comparison. 
The different algorithms make use of the knowledge of $T$ and
$\Gamma_T$.

\begin{figure}[hbt]
        \centering
        \includegraphics[width=0.45\textwidth]{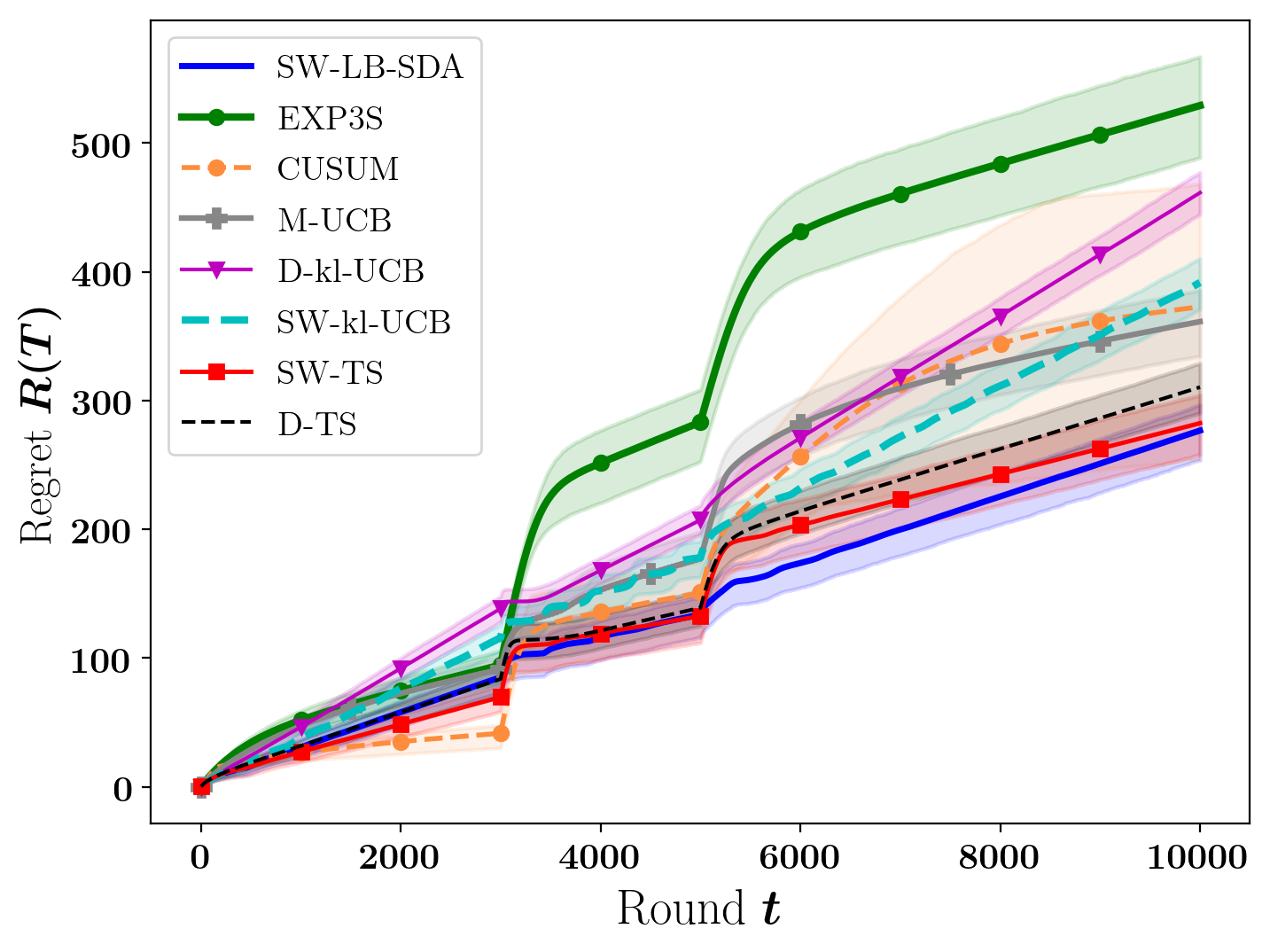}
        \caption{Performance on a Bernoulli instance averaged
        on 2000 independent replications.}
        \label{fig::bernoulli_exp}
\end{figure}

To allow for fair comparison, we use for SW-LB-SDA, the same value of
$\tau = 2 \sqrt{T \log(T) /\Gamma_T}$ that is recommended for SW-UCB
\cite{garivier2011upper}.
D-UCB uses the discount factor suggested 
by \citet{garivier2011upper}, $1/(1-\gamma) = 4 \sqrt{T/\Gamma_T}$.
The changepoint detection algorithms need extra
information such has the minimal gap for a breakpoint and the minimum
length of a stationary phase.
For M-UCB, we set $w = 800$ and $b = \sqrt{w/2 \log(2 KT^2)}$
as recommended by \citet{cao2019nearly} but set 
the amount of exploration 
to $\gamma = \sqrt{K \Gamma_T \log(T) /T}$ 
following \citet{besson2020efficient}.
In practice, using this value rather than 
the theoretical suggestion from 
\citet{cao2019nearly} improved significantly
the empirical performance of M-UCB for the horizon considered here.
For CUSUM, $\alpha$ and $h$ are 
tuned using suggestions from \citet{liu2017change}, 
namely $\alpha = 
\sqrt{\Gamma_T / T \log(T/\Gamma_T)}$ 
and $h = \log(T/\Gamma_T)$.
On this specific instance, 
using $\epsilon = 0.05$ (to satisfy 
Assumption 2 of \citet{liu2017change}) and $M=50$
gives good performance. 
For the EXP3S algorithm, 
following \cite{auer2002finite} the parameters $\alpha$ and
$\gamma$ are tuned as follows: $\alpha = 1/T$ and 
$\gamma = \min(1, \sqrt{K(e + \Gamma_T \log(K T)/((e-1)T)}$.

This problem is challenging because a policy that focuses on arm 1 to
minimize the regret in the first stationary phase also has to explore
sufficiently to detect that the second arm is the best in the second
phase.  SW-LB-SDA has performance comparable to the forgetting TS
algorithms and is the best performing algorithm in this scenario.
Note that both TS algorithms use the assumption that the arms are
Bernoulli whereas SW-LB-SDA does not. SW-klUCB performs better than 
D-klUCB and its regret closely matches the one from the changepoint
detection algorithms. 
By observing the lower and the upper quartiles, one sees that 
the performance of CUSUM vary much more than the other 
algorithms depending on its ability to detect the breakpoints.
Finally, EXP3S, which can 
adapt to more general adversarial settings,
lags behind the other algorithms 
in this abruptly changing stochastic environment. 

\begin{figure}[hbt]
        \centering
        \includegraphics[width=0.45\textwidth]{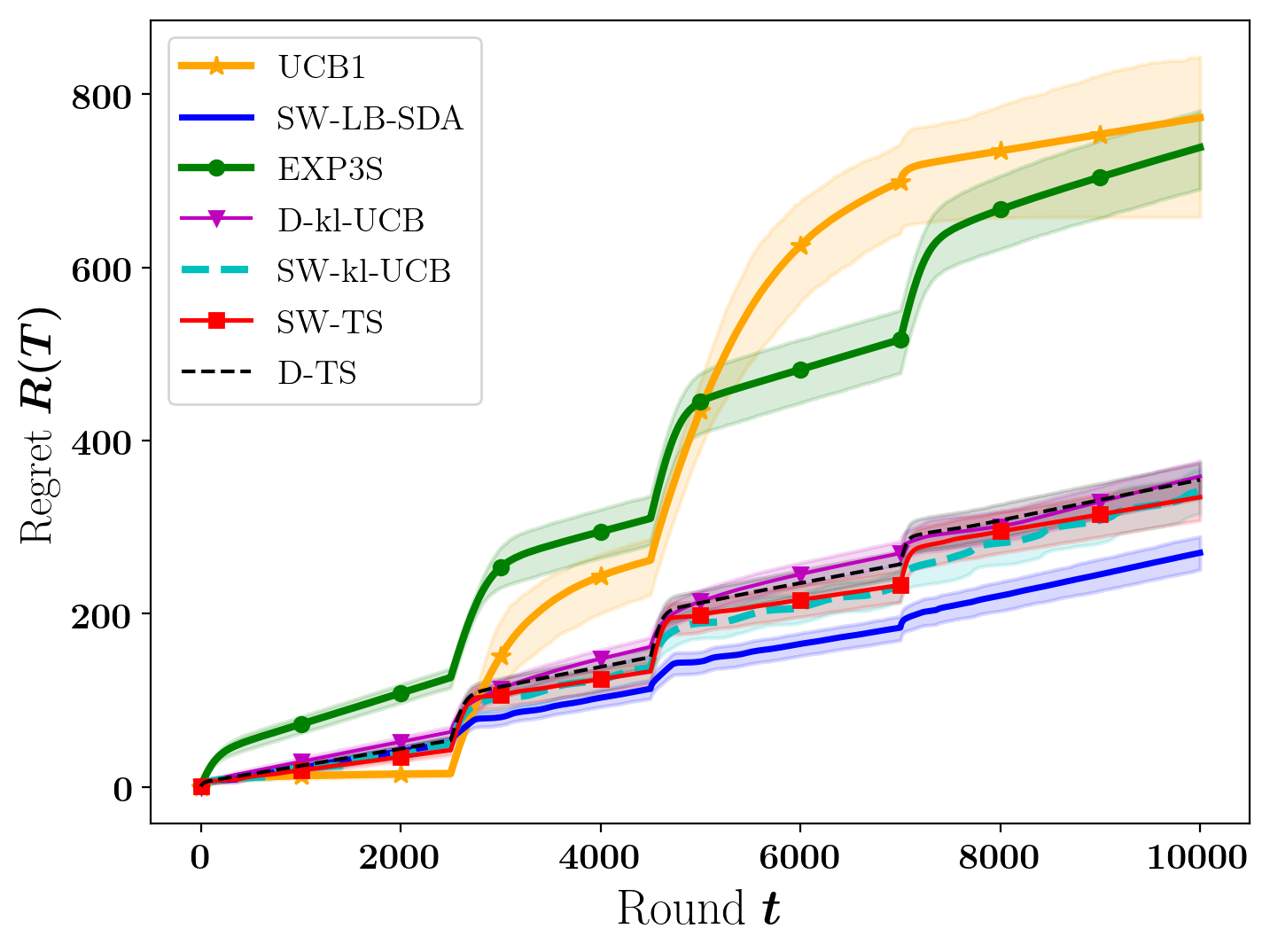}
        \caption{Performance on a Gaussian 
        instance with a constant standard deviation of $\sigma=0.5$
        averaged on 2000 independent runs.}
        \label{fig::100_rep_gauss_variance_0_5}
\end{figure}

In the third experiment with $\Gamma_T=3$ breakpoints, 
the $K=3$ arms comes from Gaussian
distributions with a fixed standard deviation of  $\sigma = 0.5$ but time
dependent means. 
The evolution of the
arm's means is pictured on the right 
of Figure~\ref{fig::distribution_arms} and
Figure~\ref{fig::100_rep_gauss_variance_0_5} displays the
performance of the algorithms. CUSUM and M-UCB can not
be applied in this setting because CUSUM is only analyzed for
Bernoulli distributions and M-UCB assume that the distributions are
bounded. 
Even if no theoretical guarantees exist for Thompson sampling with a sliding window or discount factors, when the distribution are Gaussian with known variance, we add them as competitors.
The analysis of
SW-UCB and D-UCB was done under the bounded reward assumption but the algorithms can be
adapted to the Gaussian case. Yet, the tuning of the discount factor
and the sliding window had to be adapted to obtain reasonable
performance, using $\tau = 2(1 + 2\sigma)\sqrt{T \log(T) /\Gamma_T}$
for D-UCB and $\gamma = 1 - 1/(4(1 + 2 \sigma))\sqrt{\Gamma_T/T}$ for
SW-UCB (considering that, practically, most of the rewards lie under $1 + 2 \sigma$).
For reference, Figure~\ref{fig::100_rep_gauss_variance_0_5} also displays the
performance of the UCB1 algorithm that ignores the non-stationary
structure. Clearly, SW-LB-SDA, in addition of being the only algorithm
analyzed in this setting with unbounded rewards, also has the best
empirical performance.

\begin{figure}[hbt]
        \centering
        \includegraphics[width=0.45\textwidth]{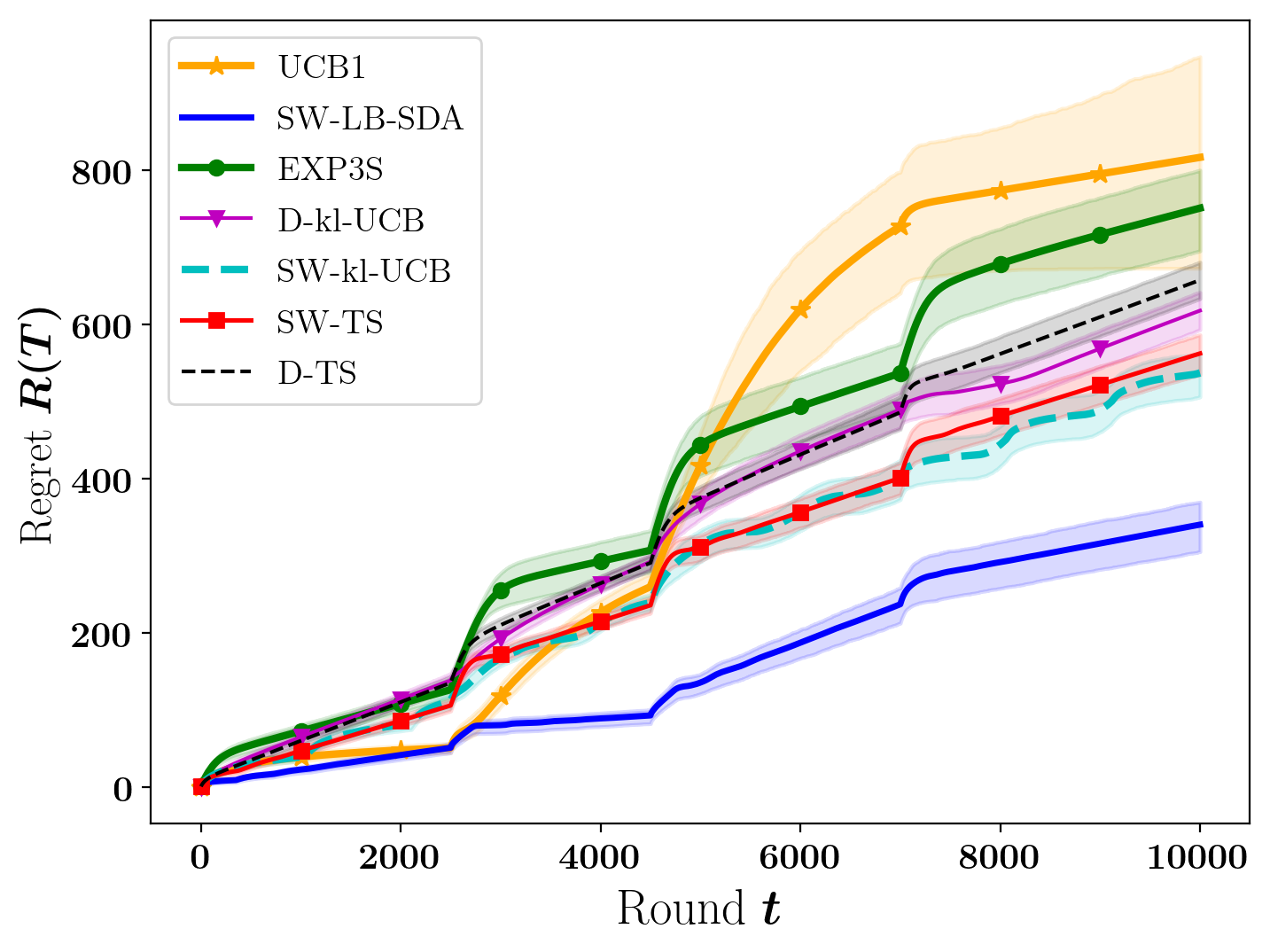}
        \caption{Performance on a Gaussian instance with time dependent standard deviations
        averaged on 2000 independent replications.}
        \label{fig::100_rep_gauss_variance_evolving}
\end{figure}

\paragraph{Changes affecting the variance.}
The last experiment features the same Gaussian means but with
different standard errors. The standard error takes the values
$0.5, 0.25, 1$ and $0.25$, respectively, in the four stationary phases.
The algorithms based on upper confidence bound are given the
maximum standard error $\sigma=1$, whereas
SW-LB-SDA is not provided with any information of this sort.
Figure~\ref{fig::100_rep_gauss_variance_evolving} shows that
the non-parametric nature of SW-LB-SDA is effective, with a
significant improvement over state-of-the-art methods in such settings.


%% file: appendix_annonce.tex
\appendix
\appendixtrue

\section*{Organization of the appendix}
The appendix is organized as follows:
\begin{itemize}
	\item In Section~\ref{app::proof_s} we provide some details on our analysis for the vanilla 
	LB-SDA algorithm.
	\item In Section~\ref{app::lbsda_limited_memory} explain how to adapt LB-SDA when a limited memory
	is used and derive an upper-bound for the regret of this variant of LB-SDA.
	\item In Section~\ref{app::NS_lbsda} a detailed analysis of LB-SDA with a sliding window in any abruptly 
	changing environment is proposed.
\end{itemize}

%% file: app_LBSDA_asopt.tex
\section{Analysis of LB-SDA}
\label{app::proof_s}

\subsection{Proof of Lemma \ref{lem::control_pull_lbsda_s}}
Before establishing our main result for LB-SDA, 
we introduce the balance function of an arm, 
which was first defined in \cite{baransi2014sub}.

Assume that the $K$ arm are characterized by the reward distributions 
$(\nu_1,..., \nu_K)$. Assume that there is a unique optimal arm associated to the 
arm $k^\star$. 

\begin{definition}
\label{def::balance}
Letting $\nu_{k,j}$ denote the distribution of the sum of $j$ 
independent variables drawn from $\nu_k$, 
and $F_{\nu_{k,j}}$ its corresponding CDF.
the balance function of arm $k$ is 
	$$
	\alpha_k(M,j) = \bE_{X \sim \nu_{k^\star, j}}\left(\left(1-F_{\nu_{k,j}}(X)\right)^M \right) \;.
	$$
\end{definition}
If we draw one sample from a distribution $\nu_{k^\star, j}$, 
and $M$ independent samples from another distribution $\nu_{k,j}$, 
the balance function $\alpha_k(M,j)$ quantifies the probability 
that each sample from $\nu_{k, j}$ is \textit{larger} than the sample from 
$\nu_{k^\star, j}$. 
The index $j$ represents itself the fact
 that these variables are built as the sum of $j$ 
 independent random variables from the same distribution 
 (respectively $\nu_{k^\star}$ and $\nu_k$).
 This function has been studied in detail in 
 \cite{baudry2020sub} (Appendix G and H), 
 and we will use its properties to prove the following result.

\lemmacontrolpulllbsda*

\begin{proof}
The main problem with the last block sampling is that if both the leader
and a given challenger are not played for some time, the index used in 
their duels remain the same due to the deterministic nature
of the sampler.
 As a consequence this challenger is never played as long 
 as the leader remains the same. 
If this situation occur too often, this would
limit the diversity for the duels played by the optimal arm $k^\star$ 
against  suboptimal leaders. 
We show that this is not possible by proving 
that the leader will be played a large number of times, which necessarily
brings some diversity. 
To measure this, we define the quantity of duels won by the leaders
at the different rounds as
\[W_r = 1+ \sum_{s=1}^{r-1} \ind(\cA_{s+1}=\{\ell(s)\}) \;,\]

where we added $1$ to consider the first round 
where every arm is pulled once. 
For any trajectory this quantity is linear in $r$.

\begin{restatable}{lemma}{lemmaWr}
\label{lem::WT}
	With $W_r = 1 + \sum_{s=1}^{r-1} \ind(\cA_{s+1}=\{\ell(s)\})$,
	for any round $r$ under LB-SDA it holds that
	\[W_r = N_{\ell(r)}(r) \geq r/K\;.\]
\end{restatable}
Before using Lemma \ref{lem::WT}, we recall 
the sampling obligation rule introduced in 
Section~\ref{sec::stationnary_lbsda}.
 and that we use to consider rounds where the optimal arm 
 has enough samples.
At any round $r$ each arm with less than 
$f(r)=\sqrt{\log r}$ samples is pulled. 
We focus on rounds where we are sure that arm $k^\star$ 
has been pulled "enough", 
and compute the probability that it has lost a lot of duels after this moment. 
In particular, we consider $a_r$ as the smallest round 
satisfying $f(a_r)\geq f(r)-1$, ensuring $N_{k^\star}(a_r) \geq \lfloor f(r)-1 \rfloor$. 
This round is exactly $\lceil f^{-1}(f(r)-1) \rceil$, that can be computed as
\begin{align*}
	f^{-1}(f(r)-1)& = \exp\left((f(r)-1)^2\right)\\
	& = \exp \left(f(r)^2 + 1 - 2 f(r)\right) \\
	& = f^{-1}(f(r))\exp(-2f(r)+1) \\
	& = r \times \exp(-2f(r)+1) \;.
\end{align*}
This means that for any $\gamma \in (0, 1)$, if $r$ 
is large enough to satisfy $f(r) \geq \frac{1-\log \gamma}{2}$ 
then $a_r \leq \gamma r$. 
For the rest of the proof we consider the
number of duels lost by the arm $k^\star$ 
after the round $a_r$ against unique subsamples
of a suboptimal leader. 
The number of duels won by the leader 
between the rounds $a_r$ and $r$ is equal to
$W_r - W_{a_r}$. 
Out of those duels, at most $(\log r)^2$ of them can concern the optimal
arm $k^\star$ because $N_{k^\star}(r) \leq \log(r)^2$.
Consequently, there is at least $W_r - W_{a_r}- (\log r)^2$ duels won by
a suboptimal leader between rounds $a_r$ and $r$.
Using Lemma \ref{lem::WT} and $W_{a_r} \leq a_r$ one has,
\begin{align*}
W_r - W_{a_r}- (\log r)^2 &\geq \frac{r}{K} - a_r - (\log r)^2 \\
& \geq \frac{r}{K} - \gamma r - (\log r)^2 \;.
\end{align*}
To simplify the expression we just write that 
for any $\beta \in (0,1)$ there exists a constant 
$r(\beta, K)$ satisfying 
$\forall r \geq r(\beta, K)$, 
\begin{equation}
\label{eq:lower_W_t}
W_r - W_{a_r} - (\log r)^2 \geq \beta \frac{r}{K} \;. 
\end{equation}
Under $N_{k^\star}(r) \leq (\log r)^2$ 
we are sure that there exists some 
$j \in \{1,...,\lfloor (\log r)^2 \rfloor \}$ 
such that a fraction $1/(\log r)^2$ 
of the duels counted above have been 
played with $N_{k^\star}(r)=j$.
Let us denote $\widetilde{W}_r = W_r - W_{a_r} - (\log r)^2$
and show this by contradiction. 
Out of those duels, we denote $\widetilde{W}_{r,j}$ the number of duels
played with $N_{k^\star}(r) = j$.
If we assume that for all $j \leq \lfloor (\log r)^2 \rfloor$,
there is strictly less than 
$\frac{\beta}{(\log r)^2} \frac{r}{K}$ 
duels played with $N_{k^\star}(r) = j$.
The following would hold,
\begin{align*}
W_r - W_{a_r} - (\log r)^2 =
\widetilde{W}_r = \sum_{j = 1}^{\lfloor (\log r)^2 \rfloor} \widetilde{W}_{r,j}
 < \sum_{j = 1}^{\lfloor (\log r)^2 \rfloor}
 \frac{\beta}{(\log r)^2} \frac{r}{K} < \beta \frac{r}{K} \;.
\end{align*}
There is a contradiction with Equation \eqref{eq:lower_W_t} and
means there is a $j  \leq \lfloor (\log r)^2 \rfloor$ 
and $\beta r/((\log r)^2K)$ duels such that 
$k^\star$ 
competes using its same block of observations of size $j$.

Furthermore, with the same argument 
we are sure that a fraction $1/(K-1)$ 
of these duels is played against the 
same leader $k \in \{2,\dots, K\}$. 
We would now like to obtain duels with non-overlapping blocks.
Even if the blocks are all consecutive, waiting for $j$ steps is enough
to ensure that they are not overlapping.
Taking a fraction $1/j$ 
of the duels from the previous subsets is hence enough to guarantee this. 

Finally, we conclude that for any $\beta \in (0,1)$ 
there exists a constant $r(\beta, K)$ such that for 
any round $r>r(\beta, K)$, under the event 
$\{N_1(r)\leq (\log r)^2 \}$  
there exists some $k \in \{2,\dots, K\}$ and 
some $j \in \{\lfloor f(r)-1 \rfloor, \lfloor (\log r)^2 \rfloor \}$ 
such as arm $k^\star$ lost at least 
$\beta \frac{r}{K(K-1)(\log r)^2 j}$ 
duels against non-overlapping blocks of arm 
$k$ while $k$ is the leader and $k^\star$ 
has exactly $j$ observations. 
This term correspond exactly to the balance function
$\alpha_k(M, j)$ from Definition~\ref{def::balance}, 
with $M=\beta\frac{r}{K(K-1)(\log r)^2j}$, hence we can upper bound 
$$
 \sum_{r=1}^T \bP\left(N_{k^\star}(r) \leq (\log r)^2 \right) 
 \leq r(\beta, K) + \sum_{k=2}^K \sum_{r=r(\beta, K)}^T 
 \sum_{j=\lfloor \log(r)-1\rfloor}^{\lfloor (\log r)^2 \rfloor} 
 \alpha_k\left(\beta\frac{r}{K(K-1)(\log r)^2 j}, j\right) \;. 
 $$
\begin{remark}
The fact that the duels concern non-overlapping blocks of arm $k$ is necessary to
obtain \textit{independent} samples. It is also important that those duels
are based on exactly $j$ observations in order to introduce the balance function.
\end{remark}
We conclude the proof using the following lemma which is proved in the next section.
\begin{restatable}{lemma}{lembalancestatio}
\label{lem::balance_statio} If the arms $k$ and $k^\star$ come from the same one-parameter exponential family of distributions it holds that
\[\sum_{r=r(\beta, K)}^T 
\sum_{j=\lfloor \log(r)-1\rfloor }^{\lfloor(\log r)^2 \rfloor} 
\alpha_k\left( \beta \frac{r}{K(K-1)(\log r)^2 j} , j\right) = O(1)\;.\]
\end{restatable}
\end{proof}

\subsection{Proof of Auxiliary Results}
\lemmaWr*

\begin{proof}
We consider any trajectory of the bandit algorithm. 
For this trajectory we consider the sequence of the 
rounds where a change of leader occurred and write
 them as the (potentially infinite) set 
 $\mathcal{Y}=[r_0, r_1, r_2, \dots]$. 
 These are basically all the rounds $r$ satisfying 
 $\ell(r)\neq \ell(r-1)$. $r_0=1$ as it is the first round
  where we start defining the leader in the 
  algorithm, and it holds that $N_{\ell(1)}(1)=1$ 
  as every arm is drawn once at the first round. 
  As the leader was not defined before it holds that $W_{1}=1=N_{\ell(1)}(1)$
   so the property holds in $r_0$. 
As a first step, we show that the property   
is valid for all $r_i$ when $i \in \N$. 
Let $i \in \N$, we assume that the property holds in $r_i$
 and we consider the round $r_{i+1}$. It holds that 
	\[W_{r_{i+1}} = W_{r_i} + \sum_{s=r_i}^{r_{i+1}-1}\ind(\cA_{s+1}=\ell(s)) \;.\] 
The sum is exactly the number of duels won
 by the arm that is leader during the interval
  $[r_i, r_{i+1}-1]$ and it holds that 
  $\sum_{s=r_i}^{r_{i+1}-1}\ind(\cA_{s+1}=
  \ell(s))= N_{\ell(r_i)}(r_{i+1})-N_{\ell(r_i)}(r_i)$. 
Furthermore, when a change of leader happens 
the number of elements of the new and former 
leader are the same, i.e. 
$N_{\ell(r_{i+1})}(r_{i+1}) = N_{\ell(r_i)}(r_{i+1})$. 
This is due to the fact that when a challenger reaches 
the history size of the leader then the arm with the largest
mean is chosen as the leader. 
In particular, if the challenger has a lower index 
than the leader at this round it cannot take
the leadership at the next round as it will otherwise lose
its duel against the leader. For this reason, 
the only possibility for a challenger to take 
the leadership is to reach to number of samples of the leader
and to have a better index at this moment. We can write
	\begin{align*}
		W_{r_{i+1}} &= W_{r_i} + \sum_{s=r_i}^{r_{i+1}-1}\ind(\cA_{s+1}=\{\ell(s)\}) \\
		&= W_{r_i} + N_{\ell(r_i)}(r_{i+1})-N_{\ell(r_i)}(r_i) \\
		&=  W_{r_i} + N_{\ell(r_{i+1})}(r_{i+1})-N_{\ell(r_i)}(r_i) \\
		&= N_{\ell(r_i)}(r_i)  + N_{\ell(r_{i+1})}(r_{i+1})-N_{\ell(r_i)}(r_i) \quad (\text{Inductive step})
		\\
		&= N_{\ell(r_{i+1})}(r_{i+1}) \;. \\
	\end{align*}
Therefore, if the property holds in $r_i$ 
then it holds in $r_{i+1}$ which gives the result. 
The extension to any round is obtained with similar arguments: 
$\forall r \notin \mathcal{Y}$, 
$\exists i: r_i < r < r_{i+1}$. 
Then we write 
	\begin{align*}
	W_r =& W_{r_i} + \sum_{s=r_i}^{r-1}\ind(\cA_{s+1}=\ell(s)) \\ 
	=& N_{\ell(r_i)}(r_i)  + (N_{\ell(r_i)}(r) - N_{\ell(r_i)}(r_i)) \\
	= & N_{\ell(r_i)}(r) 	 = N_{\ell(r)}(r) \;,
\end{align*}
where the last inequality comes from the fact 
that the leader is unchanged between the rounds $r_i$ and 
$r$.
We conclude the proof by using the property
 that as the leader always has a number of samples
  larger than $r/K$, as it is the arm with
   the largest number of pulls at each round. 
\end{proof}

\lembalancestatio*

Before proving this result we prove an intermediary
result that will also be useful to handle the balance
function in the proof for switching bandits in Appendix~\ref{app::NS_lbsda}. 
This result was already presented in \cite{chan2020multi}, but
 we provide its proof for completeness. 

\begin{lemma}
\label{lem::balance_UB}
Let $F_1$ and $F_2$ be the cdf of two distributions with
respective means $\mu_1$ and $\mu_2$, $\mu_1>\mu_2$. For any integer $j\geq 1$
 we denote $F_{1,j}$ and $F_{2, j}$ the cdf of the
 sum of $j$ independent random variables drawn
  respectively from $F_1$ and $F_2$, and 
  $\alpha(M, j) = \bE_{X\sim F_{1, j}} \left((1- F_{2,j}(X))^M\right)$ 
  the balance function of these two distributions. 
  For any $u \in \R$ it holds that 
$$
\alpha(M, j) \leq F_{1, j}(u) + (1-F_{2, j}(u))^M \;.
$$
Furthermore, 
if we assume that $F_1$ and $F_2$ come
 from the same one-parameter exponential family of distributions, 
 for any $u \in [0, 1]$ satisfying $F_2(u) \leq F_2(\mu_2)$
  the following result holds
\[\alpha(M, j) \leq e^{-j \kl(\theta_2, \theta_1)} u + (1-u)^M \;,\]
where $\kl(\theta_2, \theta_1)$ is the Kullback-Leibler divergence between $F_2$ and $F_1$, expressed with their canonical parameters $\theta_1$ and $\theta_2$.

\end{lemma}

\begin{proof}

We prove the first result, that is valid for any distribution 
$F_1$ and $F_2$ and is a direct property
 of the definition of the balance function. 
 For $u \in \mathbb{R}$, it holds that
\begin{align*} 
\alpha(M, j) &= 
\int_{-\infty}^{+\infty} (1-F_{2, j}(x))^M dF_{1, j}(x)
\\
&\leq \int_{-\infty}^{u} (1-F_{2, j}(x))^MdF_{1, j}(x)
+ \int_{u}^{+\infty} (1-F_{2, j}(x))^M dF_{1, j}(x)
\\
&\leq F_{1,j}(u)+(1-F_{2, j}(u))^M \;. \\
\end{align*}	
We now assume that $F_1$ and $F_2$ come 
from the same one-parameter exponential family of distributions. 
In this case they admit a density 
$f_\theta(y)= f(y,0)e^{\eta(\theta)y-\psi(\theta)}$ 
for some natural parameter $\theta \in \R$. 
We write $\theta_1$ the parameter of $F_1$, and $\theta_2$ 
the parameter of $F_2$. We then define some $y_1,...,y_j \in \mathbb{R}^j$. 
If the sequence $y_1, \dots, y_j$ 
satisfies $\sum_{u=1}^j y_u \leq j \mu_2$, it holds that
\begin{align*}
	\prod_{u=1}^j f_{\theta_1}(y_u) &=  
	\prod_{u=1}^j e^{(\eta(\theta_1)- \eta(\theta_2)) 
	y_u-(\psi(\theta_1)- \psi(\theta_2))} f_{\theta_2}(y_u) 
	\leq e^{-j\kl(\theta_2, \theta_1)} \prod_{u=1}^j f_{\theta_2}(y_u) \;.
\end{align*}

where we write $\kl(\theta_2, \theta_1)$ 
for the Kullback-Leibler divergence between $F_1$ and $F_2$. 
This inequality first ensures that for all $ x \leq \mu_2$
\begin{align*}
	F_{1,j}(x) \leq e^{-j \kl(\theta_2, \theta_1)} F_{2,j}(x) \;.
\end{align*}

If we insert this expression in the first result, 
we have that for any $u \in [0, 1]$ satisfying $F_2(u) \leq F_2(\mu_2)$ the following result holds
\[\alpha(M, j) \leq e^{-j \kl(\theta_2, \theta_1)} u + (1-u)^M \;.\]
\end{proof}

\begin{remark}\label{remark::tuning_balance} 
The second result is particularly interesting because 
there is a trade-off in the choice of $u$. 
If we want to upper bound $\alpha(M, j)$ by a relatively small quantity we need to choose small values for $u$, however if $u$ is too small then the second term may become too large. In particular, making the approximation $(1-u)^M \approx e^{-Mu}$ provides an optimal scaling of $u$ of the form
\[u^* = \frac{j \kl(\theta_2, \theta_1)+\log M}{M}\;,\]

and as a consequence 

\begin{align*}
	\alpha(M, j) &\leq e^{-j \kl(\theta_2, \theta_1)} u^* + (1-u^*)^M \\
& \leq \frac{j \kl(\theta_2, \theta_1)+\log M}{M} e^{-j \kl(\theta_2, \theta_1)} 
+ e^{M \log\left(1-\frac{j \kl(\theta_2, \theta_1)+\log M}{M}\right)} 
\\
& \leq \frac{j \kl(\theta_2, \theta_1)+\log M}{M} e^{-j \kl(\theta_2, \theta_1)} 
+ C_1 \frac{e^{-j \kl(\theta_2, \theta_1)}}{M} \\
& = \frac{j \kl(\theta_2, \theta_1)+\log M + C_1}{M} e^{-j \kl(\theta_2, \theta_1)} \;, 
\end{align*}
for some constant $C_1$. 
\end{remark}

With these technical results we can now finish the proof of Lemma~\ref{lem::balance_statio}
by simply replacing $M$ by its value in the double sum.

\begin{proof} 
We denote $\alpha_k$ the balance function between 
the arm $k^\star$ and an arm $k$ and want to upper bound  
\begin{align*}
\sum_{r=r(\beta, K)}^T 
\sum_{j=\lfloor \sqrt{\log r}-1\rfloor }^{\lfloor(\log r)^2 \rfloor} 
\alpha_k\left( \beta \frac{r}{K(K-1)(\log r)^2 j} , j\right) \;.
\end{align*}

We directly use the second result of Lemma~\ref{lem::balance_UB},
 and choose the tuning of $u$ from Remark~\ref{remark::tuning_balance}. 
 If we write $a_{r, j}= \alpha_k\left( \beta \frac{r}{K(K-1)(\log r)^2 j} , j\right)$ 
 and try to extract the order of $a_{r, j}$ just in terms of $r$ and $j$ we obtain 
$$
a_{k, j} = O_{r, j}\left(\frac{j^2 (\log r)^2}{r} e^{-j \kl(\theta_k, \theta_{k^\star})}
\right) \;.
$$

We then upper bound the term in $j^2$ 
by another $(\log r)^4$ using the upper limit on the sum on $j$, 
hence the only term left in $j$ is $e^{-j \kl(\theta_2, \theta_1)}$, 
which sums in a term of order $\exp(- \sqrt{\log r})$. So we then obtain a term of the form

\begin{align*} \sum_{r=r(\beta, K)}^T 
\sum_{j=\lfloor \sqrt{\log r}-1\rfloor }^{\lfloor(\log r)^2 \rfloor} 
\alpha_k\left( \beta \frac{r}{K(K-1)(\log r)^2 j} , j\right) = O\left(\sum_{r=1}^T \frac{(\log r)^6 e^{-\sqrt{\log r}}}{r} \right) \;. \end{align*}

We conclude, using that for any integer 
$k>1$, $(\log r)^k = o(e^{\sqrt{\log r}})$. 
Hence 
\[\frac{(\log r)^6 e^{-\sqrt{\log r}}}{r}= o\left(\frac{1}{r (\log r)^2}\right) \;,\]
which is the general term of a convergent series. Hence we finally obtain

\[\sum_{r=r(\beta, K)}^T 
\sum_{j=\lfloor \sqrt{\log r}-1\rfloor }^{\lfloor(\log r)^2 \rfloor} 
\alpha_k\left( \beta \frac{r}{K(K-1)(\log r)^2 j} , j\right) = O(1) \;.
\]
\end{proof}
\clearpage

%% file: beg_proof_RBSDA.tex
\section{$\LBSDA$ with a limited memory}
\label{app::lbsda_limited_memory}
In this section the variant of LB-SDA using a limited storage 
memory introduced in Section~\ref{subsection:lb_sda_memory} is analyzed. 
After introducing a few notations, 
we present a detailed version of the algorithm. 
We then provide a detailed proof of Theorem~\ref{th::memory_limited_regret}. 

\subsection{Notation for the Proof of Theorem~\ref{th::memory_limited_regret}}
\label{app::notations}

General notations for the stationary case:
\begin{itemize}
	\item $K$ number of arms
	\item $\nu_k$ distribution of the arm $k$, 
	with mean $\mu_k$. 
	We assume that $\forall k$, $\nu_k \in \cP_{\Theta}$, 
	a one-parameter exponential family.
	\item We assume that $\mu_1 = \max_{k \in [K]} \mu_k$ 
	so we call the (unique) optimal arm "arm 1".
	\item $I_k(x)$ some large deviation rate function of the arm $k$, 
	evaluated in $x$. For one-parameter exponential families this 
	function will always be the KL-divergence between $\nu_k$ 
	and the distribution from the same family with mean $x$.
	\item $N_k(r)$ number of pull of arm $k$ up to (and including) round $r$.
	\item $Y_{k, i}$ reward obtained at the $i$-th pull of arm $k$.
	\item $\bar Y_{k, i}$ mean of the $i$-th first reward 
	of arm $k$, $\bar Y_{k, n:m}$ mean of the rewards 
	of $k$ on a subset of indices $n<m$:  
	$\bar Y_{k, n:m}= \frac{1}{m-n+1} \sum_{i=n}^m Y_{k, i}$. 
	If $m-n=s$, then $\bar{Y}_{k, s}$ and $\bar{Y}_{k, n:m}$ have the same distribution.
	\item $\ell(r)$ leader at round $r$, $\ell(r)=\argmax{k \in \K} N_k(r)$.
	\item $\cA_r$ set of arms pulled at a round $r$.
	\item $\mathcal{R}_T$ regret up to (and including) round $T$.
\end{itemize}
Notations for the regret analysis, part relying on concentration:
\begin{itemize}
\item $\cZ^r = \{\ell(r) \neq 1\}$, 
the leader used at round  $r+1$ is suboptimal.
	
	\item $\cD^r = \{\exists u \in  \{ \lfloor r/4 \rfloor,..., r \} 
	\text{ such that } \ell(u-1) = 1 \}$, the optimal arm 
	has been leader at least once between $\lfloor r/4 \rfloor$ and $r$.
	\item $\cB^{u} = \{\ell(u)=1,
	 k \in \mathcal{A}_{u+1}, N_k(u)=N_1(u)-1 \text{ for some arm } k \}$, 
	 the optimal arm is leader in $u$ but loses its duel 
	 against arm $k$, that have been pulled enough to possibly take over the leadership at next round.
	\item $\cC^u = \{\exists k \neq 1, N_k(u)\geq N_1(u), 
	\hat Y_{k,S_1^u(N_k(u),N_{1}(u))} \geq \hat Y_{1,N_1(u)}\}$, the 
	optimal arm is not the leader and has lost its duel against the 
	suboptimal leader.
	\item $\cL^r= \sum_{u=\lfloor r/4 \rfloor}^r \ind_{\cC^u}$.
\end{itemize}

\subsection{The algorithm}
Before giving the algorithm, we introduce additional notations that are used in the statement of the 
algorithm. The stored history for the arm $k$ at round $r$ is denoted 
$\mathcal{H}_k(r)$. At round $r$ when comparing the leader $\ell(r)$ 
and the arm $ k \neq \ell(r)$ the last block of the history 
of $\ell(r)$ is used and is denoted $\mathcal{S}(\mathcal{H}_k(r), \mathcal{H}_{\ell}(r))$. 
In particular, when both arms are saturated their entire history of length $m_r$ is used for the duel. 
The Last Block Subsampling Duelling Algorithm with Limited Memory is reported in Algorithm 
\ref{alg:lb-sda-LM}
\begin{algorithm}[H]
\SetKwInput{KwData}{Input}
\KwData{$K$ arms, horizon $T$, $m_r$ storage limitation}
\SetKwInput{KwResult}{Initialization}
\KwResult{$t \leftarrow 1$, $r=1$ $\forall k \in \{1, ..., K \}, N_k \leftarrow 0$, $\mathcal{H}_k = \{\}$}
\While{$t < T$}{
$\mathcal{A} \leftarrow \{\}$,  $\ell \leftarrow \text{leader}(N,  t)$ \\
\If{$r=1$}{
$\cA \leftarrow \{1, \dots, K\}$ (Draw each arm once)}
\Else{
\For{$k \neq \ell \in \{1,...,K\}$}{
\If{$N_{k} \leq \sqrt{\log r}$ \text{or} $\bar Y_{k, \mathcal{H}_k} > \bar Y_{\ell, 
\mathcal{S}(\mathcal{H}_k, \mathcal{H}_{\ell} )}$ }{$\mathcal{A} \leftarrow \mathcal{A} \cup \{ k \}$}
\If{$| \mathcal{A} | = 0$}{$\mathcal{A} \leftarrow \{l \}$}
}
}
\For{$k \in \mathcal{A}$}{
\If{ $\textnormal{card}(\mathcal{H}_k) \geq m_r$}
 {$\text{pop}(\mathcal{H}_k)$ {\color{purple} \small \tcp{Removing the oldest observation}}}
Pull arm $k$, observe reward $Y_{k,N_k +1}$, $N_k \leftarrow N_k + 1$, $t \leftarrow t+1$ \\
$\mathcal{H}_k = \mathcal{H}_k \cup \{Y_{k,N_k +1} \}$ {\color{purple} \small \tcp{Append the new observation}}
}
$r \leftarrow r+1$
}
\caption{LB-SDA with Limited Memory}
\label{alg:lb-sda-LM}
\end{algorithm}

\subsection{Proof of Theorem~\ref{th::memory_limited_regret}}	
The beginning of the proof of ~\citet{baudry2020sub} is valid for \LBSDA{}, however 
it has to be rewritten completely to introduce the storage limitation. 
We use the same notation as in Section~\ref{subsection:lb_sda_memory} 
and introduce a sequence $m_r$ of allowed memory for each arm at a round $r$. 
In the beginning of the proof we do not make any assumption on the sequence $m_r$ except 
that $m_r/\log(r) \rightarrow +\infty$, 
which is required in the statement of Theorem~\ref{th::memory_limited_regret}. 
We further assume that $m_r$ is an integer 
for any round $r$, which does not change anything
for the algorithm but simplifies the notations for the proof.
In this section, without loss of generality, we assume that the arm $1$ is the unique optimal arm
$\mu_1=\max_{k \in [K]} \mu_k$. 
We also recall that the arms are assumed to come from the same one-parameter 
exponential family of distributions.

In terms of notation, 
we remark that if $N_k(r) \geq m_r$ and 
$\ell(r)\neq k$ then the duel between $k$ and $\ell(r)$ 
is the comparison between $\bar Y_{k, N_k(r)-m_r:N_k(r)}$ 
and $\bar Y_{\ell(r), N_{\ell(r)}(r)-m_r:N_{\ell(r)}(r)}$. 
Otherwise, if $N_k(r)\leq m_r$ and $\ell(r) \neq k$
then the duel is the comparison between $\bar Y_{k, N_k(r)}$ 
and $\bar Y_{\ell(r), N_{\ell(r)}(r)-N_k(r):N_{\ell(r)}(r)}$, 
which is the same as for the vanilla LB-SDA.

We recall that the set of \textit{saturated arms} at round $r$ is 
defined as 
\begin{equation}
\label{eq:def_S_u}
\cS_r=\{k \in \K: N_k(r)\geq m_r\}\;.
\end{equation}
However, we do not change the definition of the leader that is still
defined as $\ell(r)=\aargmax_{k \leq K} N_k(r)$ nor the corresponding tie-breaking rules.
All along the proof we will use the Chernoff inequality, 
that states that for any exponential family of distribution and any 
$x, y$ 
satisfying $x < \mu_k < y$, 
then $\bP(\bar Y_{k, n} \leq x) \leq e^{- \kl(x, \mu_k)}$ 
and $\bP(\bar Y_{k, n} \geq y) \leq e^{- \kl(y, \mu_k)}$. 
To simplify the notation for each arm $k$ 
we  define the real number 
$x_k=\frac{\mu_1+\mu_k}{2}\in (\mu_k, \mu_1)$, 
and write $\omega_k=\min(\kl(x_k, \mu_1), \kl(x_k, \mu_k))$. 
Hence, we will write most of our results
using concentration with this value $\omega_k$ for arm $k$. 

We write $N_k(T)$ as $N_k(T)=1 + \sum_{r=1}^{T-1} \ind(k \in \cA_{r+1}) $. The first step of the proof is to decompose the number of pulls according to the events $\{\ell(r)=1\}$ and $k \in \cS_r$,
\begin{align*}
	\bE[N_k(T)] &= 1
	+\bE\left[\sum_{r=1}^{T-1} \ind(k \in \cA_{r+1}, \ell(r)\neq 1)\right] 
	+ \bE\left[\sum_{r=1}^{T-1}\ind(k \in \cA_{r+1}, k \notin \cS_r, \ell(r)=1)\right] \\
	& \quad + \bE\left[\sum_{r=1}^{T-1}\ind(k \in \cA_{r+1}, k \in \cS_r, \ell(r)=1)\right] \\
	& \leq 1+\bE\left[\sum_{r=1}^{T-1} \ind(\ell(r)\neq 1)\right] 
	+ \bE\left[\sum_{r=1}^{T-1}\ind(k \in \cA_{r+1}, k \notin \cS_r, \ell(r)=1)\right] \\
	& \quad + \bE\left[\sum_{r=1}^{T-1}\ind(k \in \cA_{r+1}, k \in \cS_r, \ell(r)=1)\right] \;.
\end{align*}
We first study the term 
$E_1 = \bE\left[\sum_{r=1}^{T-1}\ind(k \in \cA_{r+1}, k \in \cS_r, \ell(r)=1)\right]$ 
and use that under $k \in \cS_r$ the index of both arms will be a subsample 
of size $m_r$ of their history. 
We start the sum on the rounds at $2m_1$ because two arms
cannot be saturated before this round is reached, so it holds that
\begin{align*}
	E_1 & \leq  \sum_{r=2m_1}^{T-1} \bP\left(\ell(r) = 1, 
	k \in \cA_{r+1}, N_k(r) \geq m_r, N_1(r)\geq m_r \right) \\
	& \leq \sum_{r=2m_1}^{T-1} \bP\left(\ell(r) = 1, k \in \cA_{r+1}, N_k(r) \geq m_r, N_1(r)
	\geq m_r, \bar Y_{k, N_k(r)-m_r+1: N_k(r)} \geq \bar Y_{1,N_1(r)-m_r+1: N_1(r)} \right) 
	\\	
	& \leq \sum_{r=2m_1}^{T-1} \bP\left(N_k(r) \geq m_r, \bar Y_{k, N_k(r)-m_r+1: N_k(r)} \geq x_k \right) 
	+ \sum_{r=2 m_1}^{T-1} \bP\left(N_1(r) \geq m_r, \bar Y_{1,N_1(r)-m_r+1: N_1(r)} \leq x_k \right) 
	\\
	& \leq \sum_{r=2m_1}^{T-1} \sum_{n_k=m_r}^{r}  \bP\left(\bar Y_{k, n_k-m_r+1:n_k} 
	\geq x_k, N_k(r)=n_k\right)
	+ \sum_{r=2m_1}^{T-1} \sum_{n_1=m_r}^{r}\bP\left(\bar Y_{1, n_1-m_r+1:n_1} 
	\leq x_k, N_1(r)=n_1\right)  
	\\
	& \leq \sum_{r=2m_1}^{T-1} \sum_{n_k=m_r}^{r}  \bP\left(\bar Y_{k, n_k-m_r+1:n_k} \geq x_k\right)
	+ \sum_{r=2m_1}^{T-1} \sum_{n_1=m_r}^{r}\bP\left(\bar Y_{1, n_1-m_r+1:n_1} \leq x_k\right)
	\\
	& \leq 2 \sum_{r=2m_1}^{T-1} r e^{-m_r \omega_k} \;,
\end{align*}
where we used two main elements: 
1) if two random variables $X$ and $Y$ satisfy $X\geq Y$ 
then for any threshold $\eta$ it holds that either 
$X\geq \eta$ or $Y \leq \eta$ (third line), and 
2) the empirical averages of the fixed blocks of observations 
satisfy the Chernoff concentration inequality. 
Using the notation,  we introduced 
\[\bP(\bar Y_{1, n_1-m_r+1:n_1} \leq x_k)
= \bP(\bar Y_{1, m_r} \leq x_k) \leq e^{-m_r \omega_k}
\] 
and 
\[\bP(\bar Y_{k, n_k-m_r+1:n_k} \geq x_k) =
\bP(\bar Y_{k, m_r} \geq x_k) \leq e^{-m_r \omega_k} \;.\]

Therefore, the following holds
\begin{equation}
\label{eq:saturated}
\sum_{r=1}^{T-1}\mathbb{P}(k \in \cA_{r+1}, k \in \cS_r, \ell(r)=1) \leq 2 
\sum_{r=2m_1}^{T-1} r e^{-m_r \omega_k} \;.
\end{equation}
We then study $E_2=\bE\left[\sum_{r=1}^{T-1}\ind(k \in \cA_{r+1}, k \notin \cS_r, \ell(r)=1)\right]$. 
We further distinguish two cases, whenever $N_k(r) \leq n_0(T)$ holds or not at each round, 
for some $n_0(T)$ that will be specified later.

\begin{align*}
E_2 \leq  n_0(T) 
+ \bE\left[ \sum_{r=1}^{T-1} \ind(k \in \cA_{r+1}, k 
\notin \cS_r, \ell(r)=1, N_k(r) \geq n_0(T))\right]\;.
\end{align*}

We then use that on the event $k \notin \cS_r$  the duels played between $k$ and $1$ 
will be the classical duel with the last block: 
$k$ will compete with its empirical mean and $1$ with the mean 
of its last block of size $N_k(r)$. We define some $\eta_k \in (\mu_k, \mu_1)$ and write

\begin{align*}
	E_2 & \leq   n_0(T)
	+ \bE\left[ \sum_{r=1}^{T-1} \ind(k \in \cA_{r+1}, 
	k \notin \cS_r, \ell(r)=1, N_k(r) \geq n_0(T))\right] 
	\\
	& \leq  n_0(T) 
	+ \bE\left[ \sum_{r=1}^{T-1} 
	\ind(k \in \cA_{r+1}, \bar Y_{k, N_k(r)} 
	\geq \bar Y_{1, N_1(r)-N_k(r)+1: N_1(r)}, \ell(r)=1, N_k(r) \geq n_0(T))\right] 
	\\
	& \leq n_0(T) 
	+ \sum_{r=1}^{T-1} \bP\left(k \in \cA_{r+1}, 
	\bar Y_{k, N_k(r)} \geq \eta_k, N_k(r) \geq n_0(T)\right) 
	\\
	& + \sum_{r=1}^{T-1} \bP\left(k \in \cA_{r+1}, 
	\bar Y_{1, N_1(r)-N_k(r)+1: N_1(r)}\leq \eta_k, \ell(r)=1, N_k(r) \geq n_0(T), N_1(r) \geq n_0(T) \right) \;,
\end{align*}
where we used the same trick as for $E_1$ to obtain the last result.

We then use a union bound on the values of $N_k(r)$ for the first sum 
and on both $N_k(r)$ and $N_1(r)$ for the second sum, leading to
\begin{align*}
	E_2 & \leq n_0(T)+ \sum_{r=1}^{T-1} \sum_{n_k = n_0(T)}^{T-1} 
	\bP\left(k \in \cA_{r+1}, \bar Y_{k, n_k} \geq \eta_k, N_k(r) =n_k \right) \\
	& + \sum_{r=1}^{T-1} \sum_{n_1=n_0(T)}^{T-1} \sum_{n_k=n_0(T)}^{n_1} 
	\bP\left(k \in \cA_{r+1}, \bar Y_{1, n_1-n_k+1: n_1}\leq \eta_k, N_k(r) =n_k, N_1(r)=n_1 \right) 
	\\
	& \leq n_0(T) + \sum_{n_k = n_0(T)}^{T-1} \bP\left(\bar Y_{k, n_k} \geq \eta_k \right) 
	+ \sum_{n_k = n_0(T)}^{T-1}\sum_{n_1 = n_0(T)}^{T-1} \bP\left( \bar Y_{1, n_1-n_k+1: n_1}\leq \eta_k\right)\;,
\end{align*}
where we used that $\sum_{r=1}^{T-1} \ind(k \in \cA_{r+1}, N_k(r) = n_k) \leq 1$ 
to remove the sums in $r$ (simply ignoring the event $N_1(r)=n_1$ 
in the second term). Using the Chernoff inequality, we write

\begin{align*}
	E_2 \leq n_0(T) 
	+ \frac{e^{-n_0(T)\kl(\eta_k, \mu_k)}}{1-e^{-\kl(\eta_k, \mu_k)}} 
	+ T \frac{e^{-n_0(T)\kl(\eta_k, \mu_1)}}{1-e^{-\kl(\eta_k, \mu_1)}} \;.
\end{align*}
We then calibrate $n_0(T)$ and $\eta_k$ in order to makes these terms converge properly. 
We define $\epsilon >0$ and state $n_0(T)=\frac{1+\epsilon}{\kl(\mu_k, \mu_1)} \log T$. 
We then use the continuity of the kullback-leibler divergence on $(\mu_k, \mu_1)$ to state that for any $\delta>0$, 
there exists some $\epsilon>0$ and $\eta_k \in (\mu_k, \mu_1)$ 
satisfying $\kl(\eta_k, \mu_1) \geq \kl(\mu_k, \mu_1) - \delta \geq \frac{\kl(\mu_k, \mu_1)}{1+\epsilon} $. 
This means that for any $\epsilon>0$, 
there exists some $\eta_k > 0$ satisfying 
$T e^{-n_0(T)\kl(\eta_k, \mu_1)}\leq T e^{-n_0(T)\frac{1+\epsilon}{\kl(\mu_k, \mu_1)}\log T}\leq 1$. 
Hence, for any $\epsilon >0$ it holds that
$$
E_2 \leq \frac{1+\epsilon}{I_1(\mu_k)}\log T + C_{k, \epsilon} \;,
$$
where $C_{k, \epsilon}$ is a constant.

Combining these results we can write a first decomposition of $\bE[N_k(T)]$ as
\begin{eqnarray}\label{eq::first_dec_lim_memory}
\bE[N_k(T)] \leq 1 + \frac{1+\epsilon}{I_1(\mu_k)} \log T 
+ 2 \sum_{r=2m_1}^{T-1} re^{-m_r \omega_k} 
+ C_{k, \epsilon} + \sum_{r=2m_1}^{T-1} \bP(\ell(r) \neq 1)\;.
\end{eqnarray}

We remark that this expression provides an explicit 
dependence in $m_r$ in the second term, 
that justifies the condition in 
Theorem~\ref{th::memory_limited_regret} for $m_r$ (
namely, $m_r/(\log r) \rightarrow +\infty$). 
Indeed, this condition is sufficient to ensure for instance 
that $m_r \geq \frac{3}{\omega_k} \log r$ 
for $r$ large enough, making the term inside 
the sum a $o(r^{-2})$.

The next step is to prove that 
$\sum_{r=1}^{T-1} \bP(\ell(r) \neq 1) = o(\log T)$. 
As in the proof of \cite{chan2020multi} this part causes 
a lot of technical challenges, and we need to define several 
new events to analyze the different scenarios that could lead a 
suboptimal arm to be the leader at a round $r$. 
In the next steps we will consider the same events as in the original
proof, but the storage limitation will add some complexity to 
the task. We will use the following property, 
issued from the definition of the leader
\begin{align*}
\ell(r) = k \Rightarrow N_k(r) \geq \left\lceil \frac{r}{K}\right\rceil \;.
\end{align*}
%
%
%
However, adding the storage constraint we 
have that for any $r$ satisfying 
$r \geq K m_r$ the leader has necessarily 
more than $m_r$ observations.
For this reason, its history will be truncated to the
 $m_r$ last observations. However, we leverage the
  property that when $r$ is reasonably large, $m_r$ is 
  large enough to guarantee a good concentration 
  of the empirical mean of the saturated arms around 
  their true mean.
We will explain how this can be done in this section.
We define $a_r=\left\lceil \frac{r}{4}\right\rceil$, 
and write the following decomposition
\begin{equation}
\label{dec:Z}
\bP\left(\ell(r) \neq 1\right)= 
\bP\left(\{\ell(r) \neq 1\} \cap \cD^r\right) + 
\bP\left(\{\ell(r) \neq 1\} \cap \bar\cD^r\right)
 \;.
\end{equation}
We define $\cD^r$ the event under which the optimal 
arm has been leader at least once in $[a_r,r]$.
$$
\cD^r = \{\exists u \in  [a_r, r] \text{ such that } \ell(u) = 1 \}.
$$
We now explain how to upper bound the term in the left hand side of Equation~\eqref{dec:Z}. 
We look at the rounds larger than some round $r_0$ that will be specified later in the proof.

\subsubsection{Arm $1$ has been leader $a_r$ and $r$}
We introduce a new event 
\begin{align*}
&\cB^{u} = \{\ell(u)=1, k \in \mathcal{A}_{u+1}, N_k(u)=N_1(u)-1 \text{ for some arm } k \} \;.&
\end{align*}
Under the event $\cD^r$, $\{\ell(r)\neq 1 \}$ 
can only be true only if the leadership 
has been taken over by a suboptimal arm at some round between $a_r$ and $r$, that is 
\begin{equation}
\{\ell(r)\neq 1\} \cap \cD^r \subset \cup_{u=a_r}^{r-1}\{\ell(u)=1, \ell(u+1)\neq 1 \} \subset \cup_{u=a_r}^{r-1} \cB^u \;.
\end{equation}

Indeed, a leadership takeover can only happen after a challenger has defeated the leader while having at least the same number of observations minus one (however this situation is necessary but not sufficient to cause a change of leader, hence the strict inclusion). 

We now upper bound $\sum_{r=r_0}^{T-1} \sum_{u=a_r}^r \bP(\cB^u)$. 
We use the notation $b_r=\lceil a_r/K \rceil$ representing the minimum of samples 
of the leader at the round $a_r$. Hence we are sure that under $\cB^u$ arm $1$ 
had at least $b_u$ observations 
when it lost the duel that
 cost it the leadership. 
 
 We then take an union bound on all the suboptimal arms 
 $k \in \{2, ..., K\}$, defining
 \[\cB^u = \cup_{k=2}^K \cB_k^u := 
 \left\{\ell(u)=1, k \in \mathcal{A}_{u+1}, N_k(u)=N_1(u)-1\right\}\;,\]
  
which fixes the specific suboptimal arm that could have taken the leadership.

Choosing $x_k, \omega_k$ as in the previous section we can write 
	\begin{align*}
	\sum_{r=r_0}^{T-1} \sum_{u=a_r}^r \bP(\cB_k^u)  
	&= \bE \left[\sum_{r=r_0}^{T-1} \sum_{u=a_r}^{r} \ind(\ell(u)=1, k \in \cA_{u+1}, N_1(u)=N_k(u)+1)
	\right] 
	\\
	& \leq \underbrace{\bE \left[\sum_{r=r_0}^{T-1}\sum_{u=a_r}^{r} 
	\ind(\ell(u)=1, k \in \cA_{u+1}, N_1(u)=N_k(u)+1, k \notin \cS_u) \right]}_{B_1}  
	\\
	& + \underbrace{ \bE \left[\sum_{r=r_0}^{T-1}\sum_{u=a_r}^{r} \ind(\ell(u)=1, k \in \cA_{u+1}, 
	N_1(u)=N_k(u)+1, k \in \cS_u) \right]}_{B_2}  \;.
	\end{align*}

We proceed similarly as in the previous part, analyzing separately 
the case $k \in \cS_u$ and the case $k \notin \cS_u$ with $\cS_u$
 defined in Equation~\eqref{eq:def_S_u}. 
 We start with the term $B_1$,
\begin{align}
	B_1& \leq \bE \left[ \sum_{r=r_0}^{T-1} \sum_{u=a_r}^r \ind(N_1(u) \geq  b_r, 
	\bar Y_{k, N_k(u)} \geq \bar Y_{1, N_1(u)-N_k(u)+1:N_1(u)}, N_1(u)=N_k(u)+1, k\in \cA_{u+1}, k 
	\notin \cS_u)\right] \nonumber
	\\
	& \leq \bE \left[
	\sum_{r=r_0}^{T-1} \sum_{u=a_r}^r \ind(N_1(u) 
	\geq  b_r, \bar Y_{k,N_k(u)} \geq x_k, N_1(u)=N_k(u)+1, k\in \cA_{u+1}, k \notin \cS_u) 
	\right]
	\label{eq:FirstPart}\\ 
	& + \bE \left[\sum_{r=r_0}^{T-1} \sum_{u=a_r}^r \ind(N_1(u) 
	\geq  b_r, \bar Y_{1,N_1(u)-N_k(u)+1:N_1(u)} \leq x_k, N_1(u)=N_k(u)+1, k\in \cA_{u+1}, k \notin \cS_u)
	\right]
	 \;. \label{eq:SecondPart}
	\end{align}
We now separately upper bound each of these two terms. First, 
\begin{align*}
\eqref{eq:FirstPart} & \leq  
\bE \left[\sum_{r=r_0}^{T-1} \sum_{u=a_r}^r \sum_{n_k = b_r -1}^{m_u-1} 
\ind(N_k(u) = n_k, k\in \cA_{u+1}, \bar Y_{k,n_k} \geq x_k) \right]
\\
& \leq \bE
\left[ \sum_{r=r_0}^{T-1} \sum_{u=a_r}^r 
\sum_{n_k = b_r -1}^{r} \ind(N_k(u) = n_k, k\in \cA_{u+1}, \bar Y_{k,n_k} \geq x_k) \right]
\\
& \leq  \bE \left[\sum_{r=r_0}^{T-1}  
\sum_{n_k = b_r -1}^{r} \ind( \bar Y_{k,n_k} \geq x_k) 
\underbrace{\sum_{u=a_r}^r\ind(N_k(u) = n_k) \ind(k\in \cA_{u+1})}_{\leq 1}  \right]
\\
& \leq \sum_{r=r_0}^{T-1}  \sum_{n_k = b_r -1}^{r} \bP( \bar Y_{k,n_k} \geq x_k) 
\\
& \leq \sum_{r=r_0}^{T-1} \sum_{n_k = b_r -1}^{r} \exp\left(-n_k \omega_k \right)  
\\
& \leq \sum_{r=r_0}^{T-1}\frac{e^{-(b_r -1) \omega_k}}{1-e^{-\omega_k}} \;.
\end{align*}
We remark that by definition $b_r \geq a_r/K \geq r/(4K)$ and using $r_0 \geq 8$, we conclude that
\begin{align*}
\eqref{eq:FirstPart}& \leq \frac{e^{(1- \frac{2}{K})\omega_k}}{(1-e^{-\omega_k})(1-e^{-\omega_k/(4K)})} \;.
\end{align*}
As the subsampling in $\LBSDA$ is deterministic, 
thanks to $N_1(r)=N_k(u)+1$ we obtain the same result for \eqref{eq:SecondPart}, 
\begin{align*}
\eqref{eq:SecondPart} & \leq 
\bE \left[\sum_{r=r_0}^{T-1} \sum_{u=a_r}^r 
\sum_{n_k = b_r -1}^{r} 
\ind( \bar Y_{1,2:n_k+1} \leq x_k) 
\ind(N_k(u) = n_k) \ind(k\in \cA_{u+1}) \right]
\\
& \leq \bE 
\left[ \sum_{r=r_0}^{T-1} \sum_{n_k = b_r -1}^{r} 
\ind( \bar Y_{1,2:n_k+1} 
\leq x_k) \underbrace{\sum_{u=a_r}^r \ind(N_k(u) = n_k) \ind(k\in \cA_{u+1})}_{\leq 1} \right]
\\
& \leq \sum_{r=r_0}^{T-1} \sum_{n_k = b_r -1}^{r} \bP( \bar Y_{1, n_k} \leq x_k )
\\
& \leq \frac{e^{(1- \frac{2}{K})\omega_k}}{(1-e^{-\omega_k})(1-e^{-\omega_k/(4K)})} \;.
\end{align*}

We then control $B_2$. 
For $B_2$ the condition $N_1(u)=N_k(u)+1$ 
will not be used but instead 
we use Equation \eqref{eq:saturated} 
already established in the previous section.

$$\sum_{u=1}^r \bP(k \in \cA_{u+1}, k \in \cS_u, \ell(u)=1) 
\leq 2 \sum_{u=2m_1}^{r} u e^{-m_u\omega_k} \;,
$$ 

which leads to
\begin{align*}
	B_2 &= \bE \left[\sum_{r=r_0}^{T-1}\sum_{u=a_r}^{r} \ind(\ell(u)=1, k \in \cA_{u+1}, N_1(u)
	= N_k(u)+1, k \in \cS_u) \right]
	\\
	& \leq  \sum_{r=r_0}^{T-1}\sum_{u=\max(a_r, 2m_1)}^{r} 2u e^{-m_u \omega_k} \;.
\end{align*}

Then, if consider $r_0=\min \{r: a_r \geq 2m_1\}$
 we can further upper bound $B_2$ by 

\begin{align*}
	B_2 & \leq   \sum_{r=r_0}^{T-1} \sum_{u=a_r}^{r} 2u e^{-m_u \omega_k} \\
	& \leq  2 \sum_{r=r_0}^{T-1} r \sum_{u=a_r}^{r} 2 e^{-m_u \omega_k} \\
	& \leq  2 \sum_{r=r_0}^{T-1} r^2 e^{-m_{a_r} \omega_k} \;.
\end{align*}

We first use this result without commenting its dependence in the
 sequence $(m_{r})_{r\geq 1}$. 
 Summing on all suboptimal arms $k$ we obtain
\begin{equation}\label{eq::bound_B}
\sum_{r=r_0}^{T-1}\bP\left(\{\ell(r)\neq 1\} \cap \cD^r\right) 
\leq 2 \sum_{k=2}^K  
\left[ \frac{e^{(1- \frac{2}{K})\omega_k}}{(1-e^{-\omega_k})(1-e^{-\omega_k/(4K) })}
+ \sum_{r=r_0}^{T-1} r^2 e^{-m_{a_r} \omega_k}\right]\;.
\end{equation}

Hence, the sums of the probability that arm $1$ 
is not the leader while it has already been before 
is upper bounded by two terms: 
a problem-dependent constant, and a term that 
depends of the sequence of memory limits $(m_r)_{r\geq 1}$. 
We can further analyze this second term. 
First, we remark that contrarily to the term 
in $m_r$ in Equation~\eqref{eq::first_dec_lim_memory} 
this time we have both $r^2$ and $m_{a_r}$ instead of $m_r$, 
with $a_r = \lceil r/4 \rceil$. 
Hence, for a fixed $r$ the term of the sum 
is larger in this case. 
However, the constraint $m_r/\log(r) \rightarrow +\infty$ 
is again sufficient to ensure a proper convergence 
of this sum to a constant with the same arguments. 
This is mainly because the choice of $a_r$ 
as a fraction of $r$ ensures that $m_{a_r}$ will be sufficiently large.

\subsubsection{Arm $1$ has never been leader between $a_r$ and $r$}

The idea in this part is to leverage the fact that if the optimal arm is not leader between 
$\lfloor r/4 \rfloor$ and $r$, then it has necessarily lost a lot of duels 
against the current leader at each round. 
We then use the fact that when the leader has been drawn "enough", concentration 
prevents this situation with large probability. We introduce
$$
\cL^r= \sum_{u=a_r}^r \ind_{\cC^u}\;,
$$
with $\cC^u$ defined as $\cC^u = \{\exists k \neq 1, \ell(u)=k, 1 \notin \cA_{u+1} \}$.
The following holds 
\begin{equation}
\label{eq:upper_z}
\bP(\ell(r) \neq 1 \cap \bar \cD^r) \leq \bP(\cL^r\geq r/4)\;. 
\end{equation}
This result comes from \cite{chan2020multi}, along with the direct use of the Markov inequality to provide the upper bound
\begin{equation}
\label{eq:upper_z_2}
\bP(\cL^r\geq r/4) \leq \frac{\bE(\cL^r)}{r/4}= \frac{4}{r} \sum_{u=a_r}^r \bP(\cC^u)\;.
\end{equation}
We further decompose the probability of 
$\bP(\cC^u)$ in two parts depending 
on the value of the number of selections of arm $1$. 
For the next steps we define the following events,  
$\{N_1(u) \leq C/4 \log(u) \}$ 
and $\{N_1(u) \geq C/4 \log(u)\}$, for some constant 
$C$ that is not known by the algorithm and 
that we will define later. 
This idea handle the memory limit through 
this parameter $C$. 
Indeed, we only know that the 
sequence $(m_r)_{r \geq 1}$ satisfies 
$m_r/(\log(r)) \rightarrow +\infty$. 
For this reason, we know that for any 
$C>0$ there exists a round $r_C$ such that for any 
$r\geq r_C$ then $m_r \geq C \log(r)$. 
  
Using Equation \eqref{eq:upper_z} and Equation \eqref{eq:upper_z_2}, we have
\begin{align*}
\sum_{r=r_0}^{T-1}\bP(\{\ell(r)\neq 1\} \cap \overline{\cD}^r) &\leq \underbrace{ \sum_{r=r_0}^{T-1} \frac{4}{r}\sum_{u=a_r}^r \bP \left( N_1(u) \leq \frac{C}{4} \log(u)\right) }_{B} \\
&+ \underbrace{\sum_{r=r_0}^{T-1} \frac{4}{r}\sum_{u=a_r}^r \bP\left(\cC^u,  N_1(u) \geq \frac{C}{4}\log(u) \right)}_{D}
\;.
\end{align*}

Again, $D$ can be upper bounded by 
splitting the cases when the optimal
arm is saturated or not. 
We also introduce $\cC_k^u= \{\ell(u)=k, 1 \notin \cA_{u+1} \}$ 
for any $k \in \{2, \dots, K\}$ and  obtain 
\[D \leq  \sum_{k=2}^K
\left[ \underbrace{\sum_{r=r_0}^{T-1} \frac{4}{r}\sum_{u=a_r}^r 
\bP\left(\cC_k^u,  N_1(u) 
\geq \frac{C}{4}\log(u) , 1 \in \cS_u \right)}_{D_{k, 1}} 
+ \underbrace{\sum_{r=r_0}^{T-1} 
\frac{4}{r}\sum_{u=a_r}^r 
\bP\left(\cC_k^u,  N_1(u) \geq 
\frac{C}{4} \log(u), 1 \notin \cS_u \right)}_{D_{k, 2}} \right]\;.
\]
For the event featuring $\{1 \in \cS_u\}$ we can use the result of 
the previous sections because in the event we consider there is no difference between 
$\ell(r)=1$ and $\ell(r)=k$ 
when both arms are saturated.
Following the proof for obtaining Equation \eqref{eq:saturated}, one has
\begin{equation}
\label{eq:saturated_2}
\sum_{u=a_r}^r \bP(1 \notin \cA_{u+1}, 1 \in \cS_u, \ell(u)=k) 
\leq 2 \sum_{u=a_r}^{r} u e^{-m_u\omega_k} \;.
\end{equation}
With this result we then obtain
\begin{align*}
	D_{k, 1} &= 	\sum_{r=r_0}^{T-1}\frac{4}{r}\sum_{u=a_r}^r \bP\left(\cC_k^u, 1 \in \cS_u \right) \\
	& \leq   \sum_{r=r_0}^{T-1}\frac{4}{r}\sum_{u=a_r}^r 
	\bP\left(1 \notin \cA_{u+1}, 1 \in \cS_u, \ell(u)=k \right) \\
	& \leq    \sum_{r=r_0}^{T-1}\frac{4}{r}
	\sum_{u=a_r}^r 2ue^{-m_u \omega_k} 
	\quad (\text{Equation } \eqref{eq:saturated_2})
	\\
	& \leq   8 \sum_{r=r_0}^{T-1} \sum_{u=a_r}^r e^{-m_u \omega_k} \\
	& \leq   8 \sum_{r=r_0}^{T-1} r e^{-m_{a_r} \omega_k}  \;,
\end{align*}


\begin{align*}
D_{k, 2} &\leq \sum_{r=r_0}^{T-1}
\frac{4}{r}\sum_{u=a_r}^r \bP(\cC_k^u, N_1(u)\geq \frac{C}{4} \log(u), 1 \notin \cS_u)\\
& \leq \sum_{r=r_0}^{T-1} \frac{4}{r} 
\sum_{u=a_r}^r \bP(\bar{Y}_{k,N_k(u)-N_1(u)+1:N_k(u)} > 
\bar Y_{1,N_1(u)}, N_1(u)\geq \frac{C}{4} \log(u), 1 \notin \cS_u, N_k(u)>N_1(u)) \\
& \leq \sum_{r=r_0}^{T-1} \frac{4}{r} \left[\frac{1}{1-e^{-\omega_k}} 
e^{-\frac{C}{4} \log(a_r) \omega_k} + \frac{r}{1-e^{-\omega_k}} 
e^{-\frac{C}{4} \log(a_r) \omega_k} \right] \\
 &\leq \sum_{r=r_0}^{T-1} \frac{4(r+1)}{r(1-e^{-\omega_k})} e^{-\frac{C}{4} \log(a_r) \omega_k} \\
 & \leq \sum_{r=r_0}^{T-1} \frac{6}{1-e^{-\omega_k}} e^{-\frac{C}{4} \log(a_r) \omega_k}\;.
\end{align*}
So finally
\[D \leq \sum_{k=2}^K \left[8 \sum_{r=r_0}^{T-1} r e^{-m_{a_r} \omega_k} + \sum_{r=r_0}^{T-1} \frac{6}{1-e^{-\omega_k}} e^{-\frac{C}{4} \log(a_r) \omega_k}   \right] \;. \]
At this step we remark that we need to choose the constant $C$ large enough in order to make this sum converge to a constant. 
We remind here, that $C$ is only an analysis parameter.
We then consider the term $B$. 
As in \citet{baudry2020sub} we transform the double sum 
in a simple sum by simply counting the number of times 
each term is included.
 For any integer $s$ and any round $r$, the term $\frac{4}{s}$
  only if $a_s\leq r \leq s$. 
  With the value $a_r=\left\lceil \frac{r}{4} \right\rceil$ we obtain

$$B= \sum_{r=r_0}^T\frac{4}{r}\sum_{u=a_r}^r 
\bP \left( N_1(u) \leq \frac{C}{4} \log(u) \right) 
= \sum_{r=r_0}^T \left( \sum_{t=1}^r \frac{4}{t} 
\ind(t \in [r, 4r]) \right) \bP\left(N_1(r)\leq \frac{C}{4} \log(u) \right)\;.$$ 
If we remark that $\sum_{t=1}^r \frac{4}{t} \ind(t \in [s, 4s]) \leq (4s-s+1)\times \frac{4}{s}\leq 16$, we finally get:
\begin{equation}
\label{eq:termD}
\sum_{r=r_0}^{T}\bP(\{\ell(r)\neq 1\} \cap \overline{\cD}^r)  \leq  r_0 + 16 \sum_{r=r_0}^{T} \bP\left(N_1(r) \leq \frac{C}{4} \log(r) \right) + D(\bm\nu). 
\end{equation}

Combining \eqref{eq::bound_B} and \eqref{eq:termD} yields 
\[\sum_{r=r_0}^{T} \bP\left(\ell(r)\neq 1 \right) \leq r_0 
+ 16 \sum_{r=r_0}^{T} \bP\left(N_1(r) \leq \frac{C}{4} \log(r) \right) + D'_k(\bm\nu)\]
for some constant $D'_k(\bm\nu)$ that depends on $k$ and $\bm\nu$. 
Hence, the storage limit may introduce larger constant 
terms in the proof, but asymptotically the dominant terms are the same 
as in the proof of the vanilla LB-SDA algorithm.

The last step is to show that we can upper the last term as we did in
 Appendix~\ref{app::proof_s}. To do so, we only need to prove that if 
 $r_0$ is large enough and $\{N_1(r) \leq C/4 \log(r) \}$, 
 then the arm $1$ has not been saturated 
  for a long time. This way we would handle the saturation
  exactly as we handled the forced exploration (which is still present here) 
  in the proof for the vanilla LB-SDA. 
  To do so, 
  we define the function $m^{-1}(x)=\inf \{r: m_r \geq x\}$. 
  If we had exactly $m_r = C\log r$ then this function would 
  be $m^{-1}(x)=\exp(x/C)$. 
  Up to choosing a slightly larger $r_0$, we consider that 
  for any $r>r_0$ we also have $m^{-1}(C/4 \log r) \leq \exp(C/4 \log(r) C^{-1}) = r^{1/4}$. 
  Hence, after the round $r_0$ we are sure that arm $1$ has 
  never been saturated since the round $r^{1/4}$, 
  hence we can apply the same sketch of proof as in Appendix~\ref{app::proof_s} to conclude that
\[\sum_{r=r_0}^{T} \bP\left(N_1(r) \leq \frac{C}{4} \log(r) \right) = O(1) \;. \]

%% file: app_proof_switching.tex
\section{Proof for Switching Bandits}
\label{app::NS_lbsda}
As explained in the main paper bounding
$\mathbb{E}[N_k^\phi]$, the number of pulls of a suboptimal arm $k$ during a
\textit{phase} $\phi$ is sufficient to control the \textit{dynamic regret}.
During the phase $\phi$ the best arm is denoted $\best$.
We consider the SW-LB-SDA policy with a sliding
window of size $\tau$. We also define $\hat{\delta}_\phi = r_{\phi+1}- r_{\phi}$,
the random number of rounds in the phase $\phi$.
Due to the sliding window, we use the definition of the leader
introduced in Section~\ref{sec:NS-lb-sda} and recall that
$N_k^\tau(r)= \sum_{s=r-
\tau}^{r-1} \ind\left(k \in \cA_{s+1}\right)$,
i.e. number of times arm $k$ has been pulled
during the $\tau$ last rounds.

Then for any $r \in \N$, the leader at round $r+1$ is defined as
\[
\ell^\tau(r+1) = \begin{cases}
\aargmax_{k \in \K} N_k^\tau(r+1) \text{ if } N_{\ell^\tau(r)}^\tau(r+1) < \min(r,\tau)/(2K) \\
\aargmax_{k \in \cB_r \cup \{\ell^\tau(r) \} } N_k^\tau(r+1) \text{ otherwise}
\end{cases}\;
\]

\subsection{Details for SW-LB-SDA Implementation}
With our new definition of the leader, it could happen that for some rounds the leader is not
the arm with the largest number of samples when $K \geq 3$.
We give an example of such a behavior: assume that the first round is $r=1$,
there are $2n+m$ rounds and $K=3$ arms drawn in
the following order (1 arm per round):
$m$ pulls of arm $1$, followed by $n>m$ pulls of arm $3$
 and then $n-m$ pulls of arm $1$.
 If the length of the sliding window is $\tau = 2n$
 and the leader at the round ($m + n + (n-m) = 2n$) is $1$,
 then we see that $1$ will lose samples during the next $m$ rounds.
If for those $m$ successive rounds only the arm $2$ is pulled,
then $1$ will stay leader with $n-m$ samples while $3$ still have $n$ samples.
At the end (round $2n+m$),
the leader is arm $1$, we have $N_1^\tau(2n+m) = n-m < N_3^\tau(2n+m)= n$.
This example highlights that is it possible that the leader is not the arm that
has been played the most
with a sliding window.

For this reason, the duels are slightly different to the stationary case.
The index of the leader for duels against an arm with a larger number of samples is simply the
mean of its observations collected during the last $\tau$ rounds. Indeed, in this case both
arms have a large number of samples hence subsampling is not necessary. This explain why the term
$\hat{\mu}_{\ell, k}^\tau$ is used in Algorithm \ref{alg:duel_SW}.

\subsection{Analysis}
We use the notation introduced in Section~\ref{sec:NS-lb-sda}.
The beginning of the proof takes elements from \citet{garivier2008upper} and
\citet{baudry2020sub}. For $k \neq \best$ and an arbitrary function
 $A_k^{\phi,\tau}$, we write
 \begin{align*}
N_k^\phi & =  \sum_{r=r_\phi-1}^{r_{\phi+1}-2} \ind\left(k \in \cA_{r+1}\right)
\\
& \leq 2\tau
 + \sum_{r=r_\phi+2\tau-2}^{r_{\phi+1}-2} \ind\left(k \in \cA_{r+1}\right)
 \\
& \leq  2\tau
+ \sum_{r=r_\phi+2\tau-2}^{r_{\phi+1}-2}
\ind\left(k \in \cA_{r+1}, \ell^\tau(r)= \best, N_k^\tau(r) \geq A_k^{\phi,\tau} \right) \\
& \quad + \sum_{r=r_\phi+2\tau-2}^{r_{\phi+1}-2} \ind\left(k \in \cA_{r+1}, N_k^\tau(r) < A_k^{\phi,\tau}
\right)
+ \sum_{r=r_\phi+2\tau-2}^{r_{\phi+1}-2} \ind\left(k \in \cA_{r+1}, \ell^\tau(r) \neq \best \right)
\\
& \leq  2\tau + \sum_{r=r_\phi+2\tau-2}^{r_{\phi+1}-2}
\ind\left(k \in \cA_{r+1}, \ell^\tau(r)= \best, N_k^\tau(r)\geq A_k^{\phi,\tau}, D_k^\tau(r)=0 \right)
+ \sum_{r=r_\phi+2\tau-2}^{r_{\phi+1}-2} \ind\left(\ell^\tau(r)= \best, D_k^\tau(r)=1\right)
\\
& \quad + \sum_{r=r_\phi+2\tau-2}^{r_{\phi+1}-2} \ind\left(k \in \cA_{r+1}, N_k^\tau(r) < A_k^{\phi,\tau}
\right)
+ \sum_{r=r_\phi+2\tau-1}^{r_{\phi+1}-2} \ind\left(k \in \cA_{r+1}, \ell^\tau(r) \neq \best \right) \;.
\end{align*}
We then use the following lemma.
\begin{lemma}[Adaptation of Lemma 25 from \cite{garivier2008upper}]
\label{lemma:garivier}
$$
\sum_{r=r_\phi+2\tau-2}^{r_{\phi+1}-2} \ind\left(k \in \cA_{r+1}, N_k^\tau(r) < A \right)
\leq \frac{\widehat \delta_\phi A }{\tau}\;.
$$
\end{lemma}
Therefore,
\begin{align*}
N_k^\phi &\leq 2 \tau + \frac{\widehat \delta_\phi A_k^{\phi,\tau}}{\tau}
+
\underbrace{\sum_{r=r_\phi+2\tau-2}^{r_{\phi+1}-2} \ind\left(k \in \cA_{r+1}, \ell^\tau(r)=
\best, N_k^\tau(r) \geq A_k^{\phi,\tau}, D_k^\tau(r)=0 \right)}_{c_{k,1}^{\phi, \tau}} \\
& \quad
+ \underbrace{\sum_{r=r_\phi+2\tau-2}^{r_{\phi+1}-2} \ind\left(\ell^\tau(r)=
\best, D_k^\tau(r)=1\right)}_{c_{k,2}^{\phi, \tau}}
+ \underbrace{\sum_{r=r_\phi+2\tau-1}^{r_{\phi+1}-2} \ind\left(\ell^\tau(r)\neq \best \right)}_{
c_{k,3}^{\phi, \tau}} \;.
\end{align*}

We control the expectation of these terms separately.

\subsubsection{Upper bounding $\bE [c_{k,1}^{\phi, \tau}]$}
\label{app::control_good_event}
We recall that
$$
\bE [c_{k,1}^{\phi, \tau}] =
\bE \left[
\sum_{r=r_\phi+2\tau-2}^{r_{\phi+1}-2} \ind\left(k \in \cA_{r+1}, \ell^\tau(r)=
\best, N_k^\tau(r) \geq A_k^{\phi,\tau}, D_k^\tau(r)=0 \right)
\right] \;.
$$
We start by stating a lemma on the concentration of subsample means
in Last Block sampling that is crucial for the proof.

\begin{lemma}
\label{lem::concentration_SWLB}
We consider a stationary phase $\phi$ and the multi-arm
bandit model characterized by $(\nu_1^\phi,\dots, \nu_K^\phi)$.
Let $\best$ denote the arm with the largest mean.
For each arm we assume there exists a continuous
rate function $I_k$ satisfying $I_k(x)=0$
if $x=\bE_{X \sim \nu_k^\phi(X)}=\mu_k^\phi$ and $I_k(x) \geq 0$ otherwise. Furthermore,
	\begin{align*}
		\forall x>\mu_k^\phi \text{, }\bP\left(\bar Y_n \geq x \right) & \leq e^{-n I_k(x)} \;, \\
		\forall y < \mu_k^\phi \text{, } \bP\left(\bar Y_n \leq y \right) & \leq e^{-n I_k(y)} \;.
	\end{align*}

Then, for any constant $n \in \N$ satisfying $n \geq f(\tau) = \sqrt{\log \tau}$,
by letting $\tilde n= \min(n, \floor*{\tau/(2K)})$ it holds that
\begin{equation}
\label{eq::conc_LB}
\mathbb{E}
\left[\sum_{r=r_\phi+2 \tau-2}^{r_{\phi+1}-2}
\ \ind\left( k \in \cA_{r+1}, \ell^\tau(r) = \best, N_k^\tau(r)\geq n, D_k^\tau(r)=0 \right) \right]
\leq  \delta_\phi (\tau+1) \frac{e^{-\tilde n \omega_k}}{1-e^{- \omega_k}} \;,
\end{equation}
where we defined $\omega_k = \min \left(I_k\left(\frac{1}{2}(\mu_k^\phi +
\mu_{\best}^\phi) \right),
I_{\best}\left(\frac{1}{2}(\mu_k^\phi + \mu_{\best}^\phi)\right) \right)$, and
$\delta_\phi$ is the length of the phase and $\tau$ the size of the sliding window.
Similarly,
\begin{equation}
\label{eq::conc_LB2}
\mathbb{E} \left[\sum_{r=r_\phi+\tau-2}^{r_{\phi+1}-2}
 \ind \left( \best \notin \cA_{r+1}, \ell^\tau(r)= k, N^\tau_{\best}(r)\geq n\right)  \right]
 \leq  \delta_\phi (\tau+1) \frac{e^{-\tilde n \omega_k}}{1-e^{- \omega_k}} \;.
 \end{equation}
\end{lemma}

\begin{proof}
We start with the first claim.
Under the considered event, an arm $k$ can be drawn for three reason:
1) $D_k^\tau(r)=1$, the diversity flag of this arm is raised
2) $N_k^\tau(r) \leq \sqrt{\log \tau}$, the forced exploration is used,
or
3) $k$ has won its duel against the leader
$\best$.
In our case, as $D_k^\tau(r)=0$ and
 $N_k^\tau(r) \geq n \geq \sqrt{\log \tau}$,
 if $k$ is pulled while $\best$ is leader then $k$ has won its duel against $\best$.

Under this event, the duel between $k$ and $\best$
is a comparison between the mean of two blocks
containing at least $\min(n, \tau/(2K))$ observations because of the definition of the leader.
As in \citet{baudry2020sub} we use that for any threshold
$\xi_k$, $k$ wins the duel only if either
$\widehat{\mu}^\tau_k(r) \geq \xi_k$
or $\widehat{\mu}_{\ell,k}^\tau(r) \leq \xi_k$.
For the sake of simplicity in our results we choose
$\xi_k$ as the number satisfying $\xi_k = \frac{1}{2}(\mu_k^\phi + \mu_{\best}^\phi)$,
and this choice will remain the same for the rest of the paper. We then write
\begin{align*}
	A &= \bE\left[\sum_{r=r_\phi+2\tau-2}^{r_{\phi+1}-2} \ind\left(k \in \cA_{r+1},
	\ell^\tau(r)= \best, N_k^\tau(r)\geq n, D_k^\tau(r) =0 \right)\right]
	\\
	& \leq \bE\left[\sum_{r=r_\phi+2\tau-2}^{r_{\phi+1}-2} \ind\left(k \in \cA_{r+1},
	\{\widehat{\mu}_k^\tau(r) \geq \xi_k \cup \widehat{\mu}_{\best,k}^\tau(r) \leq \xi_k \},
	N_{\best}^\tau(r)\geq \tau/(2K), N_k^\tau(r)\geq n \right)\right]
	\\
	& \leq \bE\left[\sum_{r=r_\phi+2\tau-2}^{r_{\phi+1}-2}
	\ind\left(k \in \cA_{r+1}, \widehat{\mu}_{\best,k}^\tau(r) \leq \xi_k,
	N_{\best}^\tau(r)\geq \tau/(2K), N_k^\tau(r)\geq n \right)\right]
	\\
	& \quad + \bE\left[\sum_{r=r_\phi+2\tau-2}^{r_{\phi+1}-2}
	\ind\left(k \in \cA_{r+1}, \widehat{\mu}_k^\tau(r) \geq \xi_k,
	N_{\best}^\tau(r)\geq \tau/(2K), N_k^\tau(r) \geq n \right)\right] \;.
\end{align*}
%
%
First note that for a given arm $k$
all possible blocks of observations are uniquely
 described by two quantities: $N_k^\phi(r)$
 the number of observations of arm $k$ from
 the beginning of the phase $\phi$ and $N_k^\tau(r)$
 number of observations of arm $k$ over the last $\tau$ rounds.
We will use this property to bound the two previous sums.

Starting by the simpler term featuring the arm $k$, we use
\begin{equation}
\label{eq:major_ind}
\ind\left(k \in \cA_{r+1}, \widehat{\mu}_k^\tau(r) \geq \xi_k, N_{\best}^\tau(r) \geq
\frac{\tau}{2K}, N_k^\tau(r)\geq n \right)
\leq \ind\left(k \in \cA_{r+1}, \widehat{\mu}_k^\tau(r) \geq \xi_k, N_k^\tau(r)\geq n \right)\;.
\end{equation}
$N_k^\phi$ is defined by $N_k^\phi(r)=\sum_{s=r_\phi-1}^{r-1} \ind(k \in \cA_{s+1})$. For a given round $r$ if the indicator from the RHS of Equation \eqref{eq:major_ind} is equal to 1, it implies that there is a block of length at least $n$ with a mean at least $\xi_k$. More formally, when introducing

$$
S^{n, m}_k(r) = \{ k \in \mathcal{A}_{r+1},\, \widehat{\mu}_k^\tau(r) \geq \xi_k, \, N_k^\phi(r)=m+n-1, N_k^\tau(r)=n\} \;,
$$
the following holds,
\begin{equation}
\label{eq:incl_S}
\{ k \in \cA_{r+1}, \widehat{\mu}_k^\tau(r) \geq \xi_k, N_k^\tau(r)\geq n
\} \subset \bigcup_{n_k = n}^{\hat{\delta}_\phi}
\bigcup_{m_k = 1}^{\hat{\delta}_\phi} S_k^{n_k, m_k}(r) \;.
\end{equation}
For the sake of clarity, we denote $Y_{k,1},..., Y_{k, \hat{\delta}_\phi}$ the set of possible rewards for the arm $k$ for the phase $\phi$.
If the indicator function equals one for a given round $r_0$,
then $\{k \in \mathcal{A}_{r_0+1}\}$ holds.
The same block (same value for both $n$ and $m$)
can not be used for upcoming rounds because
$N_k^\phi(r_0+1)$ will satisfy
$N_k^\phi(r_0+1) = 1 + N_k^\phi(r_0)$.
More specifically, for the arm $k$ for any possible block
there is at most one round for which the indicator function can be 1., i.e.
$$
\sum_{n_k=n}^{\hat{\delta}_\phi}  \sum_{m_k=1}^{\hat{\delta}_\phi}
\sum_{r=r_\phi+2\tau-2}^{r_{\phi+1}-2} \ind\left(S_k^{n_k, m_k}(r)\right) \leq
\sum_{n_k=n}^{\hat{\delta}_\phi}  \sum_{m_k=1}^{\hat{\delta}_\phi}
\ind\left(\bar Y_{k, m_k:m_k+n_k-1}\geq \xi_k \right)\;.
$$

Similarly,
we denote $Y_{k_\phi^\star,1},..., Y_{k_\phi^\star, \hat{\delta}_\phi}$
the set of possible rewards for the arm $k_\phi^\star$ and let
$$
S^{n, m}_{k_\phi^\star}(r) = \{ k \in \mathcal{A}_{r+1},\,
\widehat{\mu}_{k_\phi^\star, k}^\tau(r) \leq \xi_k, \, N^\phi_{\best}(r)=m+n-1, N_{\best}^\tau(r)=n\} \;.
$$
We also have
\begin{equation}
\label{eq:incl_S}
\{ k \in \cA_{r+1}, \widehat{\mu}_{k_\phi^\star,k}^{\tau}(r) \leq \xi_k, N_{k_\phi^\star}^\tau(r)
\geq n'
\} \subset \bigcup_{n^\star = n'}^{\hat{\delta}_\phi}
\bigcup_{m^\star = 1}^{\hat{\delta}_\phi} S_{k_\phi^\star}^{n^\star,m^\star}(r) \;.
\end{equation}
The main difference here is that several rounds can use the same block of observations of
$k_\phi^\star$. This can be explained because when the indicator function equals
1 the arm $k$ is drawn instead of $k_\phi^\star$ and the previous argument do not hold anymore.
Yet, $N_{\best}^\tau(r)$ can not remain unchanged
for more than $\tau$ steps because of the sliding window.
This implies in particular,
$$
\sum_{n^\star=n'}^{\hat{\delta}_\phi}  \sum_{m^\star=1}^{\hat{\delta}_\phi}
\sum_{r=r_\phi+2\tau-2}^{r_{\phi+1}-2}
\ind(S_{k_\phi^\star}^{n^\star, m^\star}(r))
\leq \tau \sum_{n^\star=n'}^{\hat{\delta}_\phi}  \sum_{m^\star=1}^{\hat{\delta}_\phi}
\ind\left(\bar Y_{\best, m^\star:m^\star+n^\star-1}\leq \xi_k \right)
\;.
$$
Bringing things together and applying the previous inequality with $n' =  \floor*{\tau/(2K)}$ we obtain
\begin{align*}
	A \leq \bE\left[\sum_{m^\star=1}^{\hat{\delta}_\phi}
	\sum_{n^\star=n'}^{\hat{\delta}_\phi} \tau \ind
	\left(\bar Y_{\best, m^\star:m^\star+n^\star-1}\leq \xi_k \right) +
	\sum_{m_k=1}^{\hat{\delta}_\phi} \sum_{n_k=n}^{\hat{\delta}_\phi} \ind\left(
	\bar Y_{k, m_k:m_k + n_k -1}\geq \xi_k \right)\right] \;.
\end{align*}
%
%
We then have to handle carefully the fact that
$\widehat \delta_\phi$ is actually a random variable
depending on the bandit algorithm. Indeed,
as several arms can be pulled at each round
we don't know what will be the length of a phase
in terms of rounds. However,
this quantity is upper bounded by the actual length
of the phase in terms of arms pulled $\delta_\phi$.

Thus, using the concentration inequality corresponding
to the family of distributions for an appropriate rate function we can write
\begin{align*}
A &\leq  \sum_{m^\star=n}^{\delta_\phi}
\sum_{n^\star=n'}^{\delta_\phi} \tau
\bP\left(\bar Y_{\best, m^\star:m^\star+n^\star-1}\leq \xi_k \right)
+ \sum_{m_k=1}^{\delta_\phi} \sum_{n_k=n}^{\delta_\phi}
\bP\left(\bar Y_{k, m_k:m_k+n_k-1}\geq \xi_k \right)
\\
& \leq \sum_{m^\star=1}^{\delta_\phi} \sum_{n^\star=n'}^{\delta_\phi} \tau
e^{-n^\star I_{\best}(\xi_k)} + \sum_{m_k=n}^{\delta_\phi} \sum_{n_k=n}^{\delta_\phi}
 e^{-n_k I_k(\xi_k)} \\
& \leq \delta_\phi \left(\tau \frac{e^{-n' I_{\best}(\xi_k)}}{1-e^{- I_{\best}(\xi_k)}}
+ \frac{e^{-n I_k(\xi_k)}}{1-e^{- I_k(\xi_k)}}\right)\\
& \leq \delta_\phi (\tau+1) \frac{e^{-\widetilde{n} \omega_k}}{1-e^{- \omega_k}} \;,
\end{align*}
where in the last inequality we have introduced
$\tilde n = \min(n, n') =
\min(n,\lfloor \tau/(2K) \rfloor) $.

Finally, the proof of the second statement is
a direct adaptation of this proof by inverting $k$ and $\best$.
We don't need the event $D_k^\phi(r)=0$ because if $\best$ is not drawn
it has necessarily lost its duel against the leader $k$.
\end{proof}

We then remark that Equation~\eqref{eq::conc_LB} in Lemma~\ref{lem::concentration_SWLB} can
be used to upper bound term $c_{k,1}^{\phi,\tau}$, by replacing $n$ by $A_{k}^{\phi,\tau}$.
Assuming that $A_{k}^{\phi,\tau}\leq \tau/(2K)$ it holds that
\begin{equation}
\label{eq:c_k_1}
\bE [c_{k,1}^{\phi,\tau}]
 \leq \delta_\phi (\tau+1)
 \frac{e^{-A_{k}^{\phi,\tau}\omega_k}}{1-e^{-\omega_k}}\;.
\end{equation}

\subsubsection{Upper bounding $\bE [c_{k,2}^{\phi, \tau}]$}
\label{app::bound_diversity_flag}
We recall that,
$$
\bE [c_{k,2}^{\phi, \tau}] =
\bE \left[
\sum_{r=r_\phi+2\tau-2}^{r_{\phi+1}-2} \ind\left(\ell^\tau(r)=
\best, D_k^\tau(r)=1\right)
\right]\;.
$$
To upper bound $\bE [c_{k,2}^{\phi, \tau}]$
we have to study the probability that the
optimal arm for the phase $\phi$ loses
$\lceil (K-1) (\log \tau)^2 \rceil$ successive
duels while being leader.
We derive in Lemma~\ref{lem::nb_duels_lost}
an intuitive consequence of this property:
the optimal arm has necessarily lost at least one duel against a concentrated arm.

\begin{lemma}
\label{lem::nb_duels_lost}
Consider $K$ arms, and assume that some arm $k$
has been leader for $M$ consecutive rounds, $M \leq \tau$.
For any $m$ satisfying $(K-1)m \leq M$,
if $k$ has lost more than $(K-1)m$ duels then
it has lost at least one duel against an arm with more than $m$ samples.
\end{lemma}

\begin{proof}
We assume that arm $k$ has been leader for $M$ consecutive rounds and that arm $k$
lost strictly more than $(K-1)m$ duels. We also assume that all the challengers that have won against the
arm $k$ have less than $m$ samples. There exists an arm $k' \neq k$ such that
$k'$ won at least $m+1$ duels against arm $k$ while having less than $m$ samples by assumption.
We denote the rounds corresponding to the first $m+1$ wins $r_1, \dots, r_{m+1}$.
The following holds,
$$
N_{k'}^\tau(r_{m+1}) = N_{k'}^\tau(r_1) + m - \sum_{s=r_1}^{r_{m+1}} \ind(k' \in \cA_{s-\tau+1}) \;.
$$

As the number of rounds where $k'$ wins against $k$ is smaller than $\tau$,
we have $\sum_{s=r_1}^{r_{m+1}} \ind(k' \in \cA_{s-\tau+1}) \leq N_{k'}^\tau(r_1)$.
Plugging this in the previous equation gives,
$$
N_{k'}(r_{m+1}, \tau) \geq m \;.
$$
We have the contradiction and it concludes the proof.
\end{proof}

Under the event $c_{k,2}^{\phi, \tau}$,
the optimal arm $\best$ is the leader and the diversity
flag for the arm $k$ is raised.
If $D_k^\tau(r)=1$, and $\best$ is the leader, it means that the leader has not changed for
$\lceil (K-1) (\log \tau)^2 \rceil$ successive rounds and hast lost more than $(K-1) (\log \tau)^2$
duels. All the conditions for applying Lemma~\ref{lem::nb_duels_lost} are met.
 Using Lemma~\ref{lem::nb_duels_lost}
 and the fact that the diversity flag cannot be activated in $r$ if it has already been
 activated in the last $\lceil (K-1) (\log \tau)^2 \rceil$ rounds it holds that
 \begin{equation}
 \label{eq:utile_D_1}
 \ind\left(\ell^\tau(r) = \best, D_k^\tau(r)=1\right)
 \leq \sum_{k' \neq \best} \sum_{s=r- \lceil (K-1) (\log \tau)^2 \rceil}^{r-1}
 \ind(\ell^\tau(s)=\best, N_{k'}^\tau(s) \geq (\log \tau)^2, k'\in \cA_{s+1}, D_{k'}^\tau(s)=0)\;.
 \end{equation}
Furthermore, we can add that an event
$\{\ell^\tau(r)=\best, N_k^\tau(s)\geq (\log \tau)^2, k\in \cA_{s+1}, D_k^\tau(s)=0\}$
 can only be associated with at most one event
 $D_k^\tau(r)=1$ for some $r$.
 Indeed, if the diversity flag is activated it cannot
  be anymore before at least $\lceil (K-1) (\log \tau)^2 \rceil$ rounds.
  Hence, combining these results we obtain
$$
\sum_{r=r_\phi+2\tau-2}^{r_{\phi+1}-2} \ind(\ell^\tau(r)=\best, D_k^\tau(r)=1)
\leq \sum_{k' \neq \best} \sum_{r=r_\phi+2\tau-2}^{r_{\phi+1}-2} \ind(k' \in \cA_{r+1},
\ell^\tau(r)=\best, N_{k'}^\tau(r)\geq (\log \tau)^2, D_{k'}^\tau(r)=0)
\;.
$$
Applying Lemma~\ref{lem::concentration_SWLB} with $n=(\log \tau)^2$ gives,
\begin{equation}
\label{eq:c_k_2}
\bE [c_{k,2}^{\phi, \tau}]
\leq \sum_{k' \neq \best}\delta_\phi(\tau+1) \frac{e^{-(\log \tau)^2 \omega_{k'}}}{1-e^{- \omega_{k'} }} \;.
\end{equation}

\subsubsection{Upper bounding $c_{k,3}^{\phi, \tau}$}
\label{app::upper_bound_leader_not_opt_SW}
We recall that,
$$
\bE [c_{k,3}^{\phi, \tau}] =
\bE \left[
\sum_{r=r_\phi+2\tau-1}^{r_{\phi+1}-2} \ind\left(\ell^\tau(r)\neq \best \right)
\right] \;.
$$
As for the stationary case the trickiest part is to prove
that the leader is the best arm with high probability.
We will first look at the terms involving the event that
the best arm has already been leader after the first $\tau$ rounds of the phase,
and then analyze the situation where it has never been leader.
As the upper bound for $c_{k,3}^{\phi, \tau}$ is difficult to obtain, we
break this section into different parts.

\textbf{\underline{Part 1:}} the optimal arm has been leader between $r-\tau$ and $r-1$

If the best arm has already been leader between $r-\tau$
and $r-1$ then it has necessarily lost its leadership
 at some intermediate round.
Loosing the leadership can be done in two different ways.
The first one called the \textit{active leadership takeover}
corresponds to the case where an arm takes the leadership by winning against the leader.
The second one, \textit{passive leadership takeover}
is simply the case where the leader loses so many duels
that its number of samples falls below $\tau/(2K)$.
We handle the first case similarly as in \citet{baudry2020sub},
while for the second we use Lemma~\ref{lem::nb_duels_lost}.

We denote
$\cD(r) =\{\exists s \in [r-\tau, r-1]: \ell^\tau(s)=\best\}$ and
we will upper bound
$\bP(\ell^\tau(r) \neq \best, \cD(r) )$.
We introduce,
\begin{align*}
\cB(r) =& \left\{\exists s \in [r-\tau, r-1]: \ell^\tau(s)=\best, \ell^\tau(s+1)\neq \best \right\}
=  \cup_{s=r-\tau}^{r-1} \left\{\ell^\tau(s)=\best, \ell^{\tau}(s+1) \neq \best \right\} \;.
\end{align*}
One has,
$$\ind(\ell^\tau(r)\neq \best, \cD(r) ) \leq \ind(\cB(r)) \;.$$
The change of leader can happen under three different scenarios:
1) some arm $k$ takes the leadership after winning against $\best$ (active takeover),
2) arm $\best$ loses the leadership because its number
 of samples falls below the threshold $\tau/(2K)$ and
3) some arm takes the leadership after being pulled
because of the diversity flag.
We remark that the activation of the diversity flag
 for some arm $k$ cannot lead to a leadership takeover by arm $k$
 if $(\log \tau)^2\leq \tau/K$,
 so this scenario can only happen for
 relatively small values of $\tau$.
 These properties can be formulated as
\begin{align*}
\left\{\ell^\tau(s)=\best, \ell^\tau(s+1) \neq \best \right\}
\subset
&\cup_{k \neq \best }\left\{\ell^\tau(s) =\best,
\ell^\tau(s+1)=k, k \in \cA_{s+1}, D_k^\tau(s)=0 \right\}
\\ & \cup \left\{ \ell^\tau(s)= \best, N_{\ell^\tau(s)}^{\tau}(s+1) \leq \tau/(2K) \right\}
\\ & \cup \left\{\ell^\tau(s)= \best, \exists k \neq \best: \ell^\tau(s+1)=k, D_k^\tau(s)=1 \right\} \;.
\end{align*}

Using this property it holds that
\begin{align*}
\sum_{r = r_\phi+2\tau-1}^{r_{\phi+1}-2}
		\ind(\ell^\tau(r) \neq \best, \cD(r) )
& \leq \sum_{r = r_\phi+2\tau-1}^{r_{\phi+1}-2}  \ind(\cB(r) ) \\
& \leq \sum_{r = r_\phi+2\tau-1}^{r_{\phi+1}-2}
\sum_{s=r-\tau}^{r-1} \sum_{k \neq \best} \ind \left(k \in \cA_{s+1},
\ell^\tau(s)=\best, \ell^\tau(s+1)= k, D_k^\tau(s)=0\right)  \\
& + \sum_{r = r_\phi+2\tau-1}^{r_{\phi+1}-2}
\sum_{s=r-\tau}^{r-1} \ind\left(\ell^\tau(s)=\best, N^\tau_{\ell^\tau(s)}(s+1) \leq \tau/(2K) \right)
\\
& + \sum_{r = r_\phi+2\tau-1}^{r_{\phi+1}-2}
\sum_{s=r-\tau}^{r-1}
\sum_{k \neq \best} \ind\left(\ell^\tau(s)=\best, \ell^\tau(s+1)=k, D_k^\tau(s)=1 \right) \;. \\
\end{align*}
We remark that if we reorganize the sums in $s$ and $r$ each element in the range
$[r_\phi+2\tau-1,r_{\phi+1}-2]$
will appear at most $\tau$ times, which leads to
\begin{align*}
\sum_{r = r_\phi+2\tau-1}^{r_{\phi+1}-2}
\ind(\ell^\tau(r) \neq \best, \cD(r) )
& \leq \underbrace{\sum_{r = r_\phi+2\tau-2}^{r_{\phi+1}-2}  \tau
\sum_{k\neq \best} \ind\left(\ell^\tau(r)=\best,
\ell^\tau(r+1)=k, k \in \cA_{r+1}, D_k^\tau(r)=0\right)}_{C_1} \\
&+ \underbrace{\sum_{r = r_\phi+2\tau-2}^{r_{\phi+1}-2}  \tau
 \ind\left(\ell^\tau(r)=\best,
N^\tau_{\ell^\tau(r) }(r+1) \leq \tau/(2K) \right)}_{C_2}  \\
& + \underbrace{\sum_{r = r_\phi+2\tau-1}^{r_{\phi+1}-2}
\tau \sum_{k \neq \best}
\ind\left(\ell^\tau(r)=\best, \ell^\tau(r+1)=k, D_k^\tau(r) =1 \right)}_{C_3} \;.
\end{align*}
We then upper bound separately the three terms.
We can upper bound $C_1$ using
Lemma~\ref{lem::concentration_SWLB} replacing $n$ by the value $\tau/K-2$,
\begin{align*}
\mathbb{E}[C_1] \leq & \sum_{k\neq \best} \tau
\mathbb{E} \left[\sum_{r = r_\phi+2\tau-2}^{r_{\phi+1}-2}  \ind\left(
k \in \cA_{r+1}, \ell^\tau(r)=\best, N_k^\tau(r)\geq \frac{\tau}{K}-2, D_k^\tau(r)=0\right)
\right]
\\
\leq & \sum_{k\neq \best} \delta_\phi \tau(\tau+1)
\frac{e^{-(\tau/K-2)\omega_k}}{1-e^{-\omega_k}} \;.
\end{align*}
To handle $C_2$ we will use Lemma~\ref{lem::nb_duels_lost}.
The definition of the leader ensures that when
one arm takes the leadership is does it with at least
$\tau/K$ observations.
Hence, to make this number go below the threshold
$\tau/(2K)$, $\best$ has to lose at least $\tau/(2K)$
duels between the moment this arm took the leadership and the round $r$.
There are two possibilities.
The first one is that $\best$ was leader
for at least $\tau$ rounds: as the index of
each arms are computed from observations
that have been all drawn under the leadership of $\best$
then at least one arm has to beat $\best$
while having more than $\tau/K-1$ observations,
which results in an active leadership takeover by this arm.
Hence, a \textit{passive} change of leader can
only happen if $\best$ was leader
for less than $\tau$ rounds.
In this case, we apply
Lemma~\ref{lem::nb_duels_lost},
it ensures that $\best$ lost at least one duel
with an arm with more than
$\lfloor \frac{\tau}{2K(K-1)} \rfloor$
observations during the time it was leader. Formally,
 $$
 \left\{\ell^\tau(r) =\best, N_{\best}^\tau(r+1)
 \leq \tau/(2K) \right\} \subset
 \cup_{s=r-\tau}^{r-1} \left\{\exists k
 \neq \best: k \in \cA_{s+1}, \ell^\tau(s) =\best, N_k^\tau(s)
 \geq \left\lfloor \frac{\tau}{2K(K-1)}\right\rfloor \right\} \;.
 $$
 We can write
\begin{align*}
\bE [C_2] &= \tau
\bE \left[\sum_{r = r_\phi+2\tau-2}^{r_{\phi+1}-2}
\ind(\ell^\tau(r)=\best, N_{\best}^\tau(r+1) \leq \tau/(2K)) \right]
\\
&  \leq \tau \sum_{k \neq \best}
\bE\left[ \sum_{r = r_\phi+2\tau-2}^{r_{\phi+1}-2}
\sum_{s= r- \tau}^{r-1}
 \ind\left(k \in \cA_{s+1}, \ell^\tau(s)=\best, N_k^\tau(s)
 \geq \left\lfloor \frac{\tau}{2K(K-1)}\right\rfloor,
 D_k^\tau(s) =0\right)  \right]
 \\
 &  \leq \tau^2 \sum_{k \neq \best}
\bE\left[ \sum_{r = r_\phi+2\tau-2}^{r_{\phi+1}-2}
 \ind\left(k \in \cA_{r+1}, \ell^\tau(r)=\best, N_k^\tau(r)
 \geq \left\lfloor \frac{\tau}{2K(K-1)}\right\rfloor,
 D_k^\tau(r) =0\right)  \right]
 \\
& \leq \sum_{k\neq \best} \delta_\phi \tau^2 (\tau+1)
\frac{e^{-\left\lfloor \frac{\tau}{2K(K-1)}\right\rfloor \omega_k}}{1-e^{-\omega_k}}\;.
\end{align*}
In the second to last inequality, we have used that the terms can appear at most $\tau$ times and
the last inequality result from Lemma~\ref{lem::concentration_SWLB}.

We now focus on the term $C_3$.
We use that $\{\ell^\tau(s+1)=k, D_k^\tau(s)=1\}$
can happen only if $\tau/K\leq (\log \tau)^2$ because if
$(\log \tau)^2 \leq \tau/K$, the activation of the diversity flag is not sufficient to
take over the leadership.
We recall that,
\begin{align*}
	\bE [C_3]
	= \bE\left[ \sum_{r = r_\phi+2\tau-2}^{r_{\phi+1}-2}
	\tau \sum_{k \neq \best} \ind \left(\ell^\tau(r)=\best, \ell^\tau(r+1)=k, D_k^\tau(r)=1 \right) \right] \;.
\end{align*}
Using Equation \eqref{eq:utile_D_1}, and letting $b =\lceil (K-1) (\log \tau)^2 \rceil$, one has
 \begin{align*}
	\bE [C_3]
	&\leq \tau \sum_{k \neq \best} \bE
	\left[\sum_{r = r_\phi+2\tau-2}^{r_{\phi+1}-2}
	\sum_{k' \neq \best} \sum_{s=r-b}^{r-1}
	\ind(k' \in \cA_{s+1}, \ell^\tau(s)=\best, N_{k'}^\tau(s)
	\geq (\log \tau)^2, D_{k'}^\tau(s)=0) \ind\left(\tau/K \leq (\log \tau)^2\right) \right]
	\\
& \leq
\tau (K-1) \sum_{k' \neq \best} \ind\left(\tau/K \leq (\log \tau)^2\right)
\bE \left[\sum_{r = r_\phi+2\tau-2}^{r_{\phi+1}-2} \ind(k' \in \cA_{r+1}, \ell^\tau(r)=\best, N_{k'}^\tau(r)
	\geq (\log \tau)^2, D_{k'}^\tau(r)=0) \right]\;.
\end{align*}
As $\ind\left(\tau/K \leq (\log \tau)^2\right)$
is deterministic, we conclude by applying Lemma~\ref{lem::concentration_SWLB}.
$$
\bE [C_3] \leq (K-1) \sum_{k \neq \best}\delta_\phi \tau (\tau+1)
\frac{e^{-(\log \tau)^2 \omega_k}}{1-e^{- \omega_k}} \ind\left(\tau/K \leq (\log \tau)^2\right) \;.
$$
We then use the condition on $\tau$ to simply upper bound $C_3$ by
$$
\bE [C_3] \leq (K-1) \sum_{k \neq \best}\delta_\phi \tau (\tau+1)
\frac{e^{-(\tau/K) \omega_k}}{1-e^{- \omega_k}} \;.$$
We observe that the three terms $\bE [C_1]$, $\bE [C_2]$ and $\bE [C_3]$
have very similar upper bounds,
so we finally regroup them in a single term using
$\left\lfloor \frac{\tau}{2K(K-1)} \right\rfloor \leq \tau/K-2 \leq \tau/K$.
$$
\bE [C_1]+ \bE [C_2]+ \bE [C_3]
\leq 3\delta_\phi \tau^2 (\tau+1) (K-1)
\sum_{k\neq \best}
\frac{e^{-\left\lfloor \frac{\tau}{2K(K-1)}\right\rfloor \omega_k}}{1-e^{-\omega_k}} \;.
$$

\textbf{\underline{Part 2:}} the optimal arm has never been
the leader after the $2\tau$ first observations of the phase.

We now aim at upper bounding
$\bE \left[\sum_{r=r_{\phi}+2\tau-2}^{r_{\phi+1}-2} \ind(\cD(r)^c )\right]$,
where $\cD(r)^c$ is the event that $\best$
has never been the leader between $r-\tau$ and $r-1$. To do so, we use that
\[\cD(r)^c \subset \left\{\sum_{s = r-\tau}^{r-1}
\ind\left( \best \notin \cA_{s+1}, \ell^\tau(s)\neq \best \right) \geq \frac{\tau}{2}\right\} \;,\]
and as in \citet{chan2020multi} we would like to handle this term using the Markov inequality. However, the problem in non-stationary environment is that the index of the sum is a random variable. Hence, to get back to a sum with a deterministic number of terms we introduce the set $\cR_\phi = [r_\phi + 2\tau -1, r_{\phi+1}-2]$ and write
\begin{align*} \bE \left[\sum_{r=r_{\phi}+2\tau-1}^{r_{\phi+1}-2} \ind(\cD(r)^c )\right] &=
\bE \left[\sum_{r=2\tau}^{T} \ind(\cD(r)^c, r \in \cR_\phi)\right] \\
&\leq \sum_{r=2\tau}^T \bE \left[\ind(\cD(r)^c, r \in \cR_\phi)\right] \\
& \leq \sum_{r=2\tau}^T \bP \left(\sum_{s = r-\tau}^{r-1}
\ind\left( \best \notin \cA_{s+1}, \ell^\tau(s)\neq \best \right) \geq \frac{\tau}{2}, r \in \cR_\phi\right) \\
& \leq \sum_{r=2\tau}^T \bP \left(\sum_{s = r-\tau}^{r-1}\ind(r \in \cR_\phi)
\ind\left( \best \notin \cA_{s+1}, \ell^\tau(s)\neq \best \right) \geq \frac{\tau}{2} \right)\;.
\end{align*}
At this step we can use the Markov inequality, and obtain
\begin{align*}
	\bE \left[\sum_{r=r_{\phi}+2\tau-1}^{r_{\phi+1}-2} \ind(\cD(r)^c )\right] &\leq
	\sum_{r=2\tau}^T \frac{2}{\tau} \bE\left[\sum_{s = r-\tau}^{r-1}\ind(r \in \cR_\phi)
	\ind\left( \best \notin \cA_{s+1}, \ell^\tau(s)\neq \best \right) \right] \\
	& \leq \bE\left[\sum_{r=2\tau}^T \ind(r \in \cR_\phi) \frac{2}{\tau} \sum_{s = r-\tau}^{r-1}
	\ind\left( \best \notin \cA_{s+1}, \ell^\tau(s)\neq \best \right) \right] \\
	& \leq \bE\left[\sum_{r\in \cR_\phi}\frac{2}{\tau} \sum_{s = r-\tau}^{r-1}
	\ind\left( \best \notin \cA_{s+1}, \ell^\tau(s)\neq \best \right) \right] \\
	& = \bE\left[\sum_{r=r_\phi+2\tau-1}^{r_{\phi+1}-2}\frac{2}{\tau} \sum_{s = r-\tau}^{r -1}
	\ind\left( \best \notin \cA_{s+1}, \ell^\tau(s)\neq \best \right) \right] \;.
\end{align*}
Hence,
\begin{align*}
\bE \left[ \sum_{r= r_\phi+2\tau-1}^{r_{\phi+1}-2} \ind(\cD(r)^c)  \right]
\leq \bE\left[\sum_{r=r_\phi+2\tau-1}^{r_{\phi+1}-2}\frac{2}{\tau} \sum_{s = r-\tau}^{r-1}
\ind\left( \best \notin \cA_{s+1}, \ell^\tau(s)\neq \best \right) \right]
\leq D_1 + D_2 \;,
\end{align*}
where,
\begin{align*}
D_1 &= \bE
\left[
\sum_{r= r_\phi+2\tau-1}^{r_{\phi+1}-2} \frac{2}{\tau} \sum_{s=r-\tau}^{r-1}
\ind \left( \best \notin \cA_{s+1}, \ell^\tau(s)
\neq \best, N_{\best}^\tau(s) \geq  A_{\best}^{\phi,\tau} \right) \right]
\\
D_2 &= \bE
\left[
\sum_{r= r_\phi+2\tau-1}^{r_{\phi+1}-2}
\frac{2}{\tau}\sum_{s=r -\tau}^{r -1}
\ind\left(N_{\best}^\tau(s)\leq  A_{\best}^{\phi,\tau} \right) \right] \;.
\end{align*}
%
The different rounds can appear at most $\tau$ times in the double sum.
Using this and the second equation of
Lemma~\ref{lem::concentration_SWLB}, $D_1$ can be upper bounded
\[D_1
\leq 2 \bE \left[ \sum_{r= r_\phi + 2\tau -2}^{r_{\phi +1} -2}
\ind \left( \best \notin \cA_{r+1}, \ell^\tau(r)
\neq \best, N_{\best}^\tau(r) \geq  A_{\best}^{\phi,\tau} \right) \right]
\leq 2 \delta_\phi (\tau+1) \sum_{k\neq \best}
\frac{e^{-A_{\best}^{\phi,\tau}
\omega_k}}{1- e^{-\omega_k}} \;.\]
Contrarily to the stationary case, we cannot work directly with $D_2$
and have to further decompose $\ind(N_{\best}(r, \tau) \leq A_{\best}^{\phi,\tau})$.
Indeed, the proof in the stationary case use the sparsity of the observations of $\best$
 when it has not been pulled a lot, and the fact that in this
case it has necessarily lost a lot of duel while having a fixed sample size.
 This is not the case in the non stationary environment,
 as for instance if $\best$ has been pulled a lot in the previous
  windows its index may change a lot.
  To avoid this we split the event according to the values of $N_k^\tau(r-\tau)$.
\[\ind\left(N_{\best}^\tau(r) \leq A_{\best}^{\phi,\tau} \right) \leq  \ind\left(N_{\best}^\tau(r) \leq
A_{\best}^{\phi,\tau}, N_{\best}^\tau(r-\tau) > A_{\best}^{\phi,\tau}\right)
 +  \ind\left(N_{\best}^\tau(r) \leq
A_{\best}^{\phi,\tau}, N_{\best}^\tau(r-\tau) \leq A_{\best}^{\phi,\tau}
\right)\;.\]
We then write $D_2= 2 (D_3+D_4)$,
with
\begin{align*}
D_3 & = \bE \left[ \sum_{r=r_{\phi}+2\tau-1}^{r_{\phi+1}-1}
\ind
\left(N_{\best}^\tau(r) \leq A_{\best}^{\phi,\tau}, N_{\best}^\tau(r- \tau)>
A_{\best}^{\phi,\tau}\right)\right] \;,
\\
D_4 &= \bE \left[ \sum_{r=r_{\phi}+2\tau-1}^{r_{\phi+1}-1}
\ind \left(N_{\best}^\tau(r) \leq
A_{\best}^{\phi,\tau}, N_{\best}^\tau(r- \tau)
\leq A_{\best}^{\phi,\tau} \right)\right] \;.
\end{align*}

$D_3$ can be upper bounded using Equation~\eqref{eq::conc_LB2}
in Lemma~\ref{lem::concentration_SWLB}. Indeed, if $N_{\best}^\tau(r) \leq A_{\best}^{\phi,\tau}$ and
$N_{\best}^\tau(r-\tau, \tau) > A_{\best}^{\phi,\tau}$, for large enough values of $\tau$, $\best$
can not be the leader and lost at least one duel against a suboptimal
leader
while having exactly $A_{\best}^{\phi,\tau}$ samples between
round $r-\tau$ and round $r-1$, thus

\[
\left\{N_{\best}^\tau(r) \leq A_{\best}^{\phi,\tau},
N_{\best}^\tau(r-\tau) > A_{\best}^{\phi,\tau} \right\}
\subset \cup_{s=r-\tau}^{r-1} \left\{\best \notin \cA_{s+1},
\ell^\tau(s)\neq \best, N_{\best}^\tau(s) = A_{\best}^{\phi,\tau} \right\}
\;.
\]
We use the same trick as for $D_1$ and $D_2$ to handle the sums and write
\begin{align*}
D_3 & \leq \bE \left[
\sum_{r=r_{\phi}+2\tau-1}^{r_{\phi+1}-1} \sum_{s=r-\tau}^{r-1} \ind \left(\best
\notin \cA_{s+1},
\ell^\tau(s) \neq \best,N_{\best}^\tau(s)= A_{\best}^{\phi,\tau} \right)  \right]
\\
& \leq \tau \bE \left[\sum_{r=r_{\phi}+2\tau-1}^{r_{\phi+1}-1} \ind \left(\best \notin \cA_{r+1},
\ell^\tau(r) \neq \best, N_{\best}^\tau(r)= A_{\best}^{\phi,\tau} \right) \right]\;.
\end{align*}
We can directly use Lemma~\ref{lem::concentration_SWLB},
however we remark that as we do not have to use an union bound
 on the values of $N^\tau_{\best}$ we can remove the factor
 $1/(1-e^{-\omega_k})$. Hence, we finally get
\[D_3 \leq \delta_\phi \tau(\tau+1) e^{- A_{\best}^{\phi,\tau} \omega_k}
\;.\]

We then handle $D_4$ by using the arguments introduced by \citet{baransi2014sub}
with some novelty due to the sliding window.
Indeed, we remark that if both $N_{\best}^\tau(r-\tau) \leq A_{\best}^{\phi,\tau}$
and $N_{\best}^\tau(r) \leq A_{\best}^{\phi,\tau}$, then $\best$ competes with at most
$2 A_{\best}^{\phi,\tau}$
\textit{different index} in the entire window $[r-\tau, r-1]$.
This is due to the fact that the index change only if $\best$
is pulled (can happen at most $A_{\best}^{\phi,\tau}$ times)
or if $\best$ loses one observation from the window $[r-2\tau, r-\tau-1]$ due to the sliding window
(which can also happen at most $A_{\best}^{\phi,\tau}$ times).
Thanks to these properties we know that during the interval
$[r-\tau, r-1]$ we are sure that $\best$ lost at least
$\tau - A_{\best}^{\phi,\tau}$ duels, and that a fraction
$1/2 A_{\best}^{\phi,\tau}$
of them occurred while the index of $\best$ remained the same.

Our objective is to highlight a property similar to the balance condition.
To do so we need to identify the fraction of the duels played by $\best$
with the \textit{same index} and
against \textit{non-overlapping} blocks
(i.e of mutually independent means)
of any suboptimal arm $k \in \K, k \neq \best$.
To avoid cumbersome notations
we summarize the elements that allow this conclusion,
first recalling the arguments of the previous paragraph:
\begin{itemize}
	\item $\best$ lost at least $\tau - A_{\best}^{\phi,\tau}$ duels in the window $[r- \tau, r-1]$
	\item A fraction $1/(2A_{\best}^{\phi,\tau})$
	of them has been played with a fixed index for $\best$,
	i.e with the subsample mean of the same block.
	With a forced exploration $B(\tau) = \sqrt{\log \tau}$ this block
	can have any size between $\sqrt{\log \tau}$ and $A_{\best}^{\phi,\tau}$.
	\item Among those duels,  a fraction of at least $1/(K-1)$ of them
	has been played against the same suboptimal arm $k \neq \best$.
\end{itemize}

The next step is to identify the proportion of these duels that have been played
 against non-overlapping blocks of $k$.
 As in the proof for the stationary case 
we proceed in $2$ steps. First we identify
the number of \textit{different} duels
(i.e the index of $k$ is not based on the same block of observations of $k$)
played by $\best$ against $k$.
However, thanks to the \textit{diversity flag} we know a new duel
happens after at most each $(K-1) (\log \tau)^2$ rounds.
So we further process the set of duels previously identified stating that:

\begin{itemize}
	\item A fraction of $\frac{1}{(K-1)(\log \tau)^2}$
	has been played against different index of $k$ based
	on different blocks of observations from the history of $k$, thanks to the diversity flag.
	\item As the blocks are of maximum size $A_{\best}^{\phi,\tau}$ a fraction
	at least $1/A_{\best}^{\phi,\tau}$ of them are \textit{non-overlapping}.
\end{itemize}

We put all these elements together to state that there exist some
$\beta \in (0,1)$ such that for any value of $\tau$
large enough $\best$ lost at least
$C^\tau=\left\lfloor
\frac{\beta \tau}{2(K-1)^2 (\log \tau)^2 (A_{\best}^{\phi,\tau})^2} \right\rfloor$
duels against non-overlapping blocks of some challenger $k$,
with a fixed index. We write this event $E_j^\tau$. Summing on all the arms, rounds, possible interval (index $n$) and size of the history of $\best$ (index $j$), we obtain
\begin{align*}
D_4 \leq& \bE\left[\sum_{k\neq \best}
\sum_{r=r_\phi+2\tau-1}^{r_{\phi+1}-1}
\sum_{n=1}^{2\left\lfloor A_{\best}^{\phi,\tau} \right\rfloor}
\sum_{j=\sqrt{\log \tau}}^{\left\lfloor A_{\best}^{\phi,\tau} \right\rfloor} \ind(E_j^\tau)\right] \;.
\end{align*}
As these events do not depend on $r$ and on $n$ we have
\begin{align*}
D_4 \leq& 2 \delta_\phi A_{\best}^{\phi,\tau}\sum_{k\neq \best} \sum_{j=\sqrt{\log \tau}}^{\left\lfloor A_{\best}^{\phi,\tau} \right\rfloor}\bE\left[\ind(E_j^\tau)\right] \\
\leq & 2 \delta_\phi A_{\best}^{\phi,\tau}\sum_{k\neq \best} \sum_{j=\sqrt{\log \tau}}^{\left\lfloor A_{\best}^{\phi,\tau} \right\rfloor}\alpha_k^\phi(C^\tau, j) \;.
\end{align*}
Here $\alpha_k$ is the balance function, as defined in Appendix~\ref{app::proof_s}. We index these functions by $\phi$ and $k$ in order to denote the balance function between $\best$ and $k$ in the phase $\phi$.
We recall the definition of $\alpha_k$, for any integer $M$
$$
\alpha_k^\phi(M, j) = \bE_{X \sim \nu_{\best}^\phi} \left((1-F_{k, j}^\phi(X))^M\right)\;,
$$
where $\nu_{k'}^\phi$ is the distribution of the sum of $j$
random variables drawn from the distribution of an
arm $k'$ in the phase $\phi$, and
$F_{k', j}^\phi$ its cdf.
 We then use the Lemma~\ref{lem::balance_UB},
 introduced and proved in Appendix~\ref{app::proof_s}.
 We recall that this result state that for any $u \leq \mu_k^\phi$ it holds that
\[\alpha_k(C^\tau, j) \leq e^{-j \kl(\mu_k^\phi, \mu_{\best})} u + (1-u)^{C^\tau} \;.\]
We write  $\kl(\mu_k^\phi, \mu_{\best})= \omega_k^\phi$, and choose the value
$u=\frac{3\log \tau}{C^\tau}$. Thanks to this choice, there exist a constant $\gamma>1$ such that
\begin{align*}
(1-u)^{C^\tau} & = \exp\left(C^\tau \log(1-u)\right) \\
& = \exp\left(C^\tau \log\left(1-\frac{3 \log \tau}{C^\tau}\right)\right) \\
& \leq \gamma \exp\left(-3 \log \tau\right)\\
& \leq \frac{\gamma}{\tau^3} \;.
\end{align*}
If we plug this expression to upper bound the sums we obtain
\begin{align*}
	D_4 & \leq 2 \delta_\phi A_{\best}^{\phi,\tau}\sum_{k\neq \best} \sum_{j=\sqrt{\log \tau}}^{\left\lfloor A_{\best}^{\phi,\tau} \right\rfloor} \left[e^{-j \omega_k^\phi} \frac{3 \log \tau}{C^\tau} + \frac{\gamma}{\tau^3} \right] \\
	& \leq 2 \delta_\phi A_{\best}^{\phi,\tau}\sum_{k\neq \best} \left[ \frac{e^{-\sqrt{\log \tau} \omega_k^\phi}}{1-e^{-\omega_k^\phi}}\frac{3 \log \tau}{C^\tau} + \frac{\gamma A_{\best}^{\phi,\tau}}{\tau^3} \right] \\
	& \leq 2 \delta_\phi A_{\best}^{\phi,\tau} (K-1) \left[\frac{e^{-\sqrt{\log \tau} \omega^\phi}}{1-e^{-\omega^\phi}}\frac{3 \log \tau}{C^\tau} + \frac{\gamma A_{\best}^{\phi,\tau}}{\tau^3} \right]
	\;,
\end{align*}

where $\omega^\phi =\min_{k \neq \best} \omega_k^\phi$.
Even if these terms look impressive we explain
 in the next section that they are not first order terms in the regret analysis.
 Indeed, if we only look at the order of $A_{\best}^{\phi,\tau}$,
 $C^\tau$, we can use the same argument as in the proof of
 Lemma~\ref{lem::balance_statio}. Considering that for any integer
 $k>1$, $(\log \tau)^k = o\left(e^{-\sqrt{\log r} \omega}\right)$
 we obtain that asymptotically
 $D_4$ is a $o\left(\frac{\delta_\phi}{\tau \log \tau^{k'}}\right)$
 for any integer $k'\geq 1$.

\subsection{Summary: Upper Bound on the Dynamic Regret}

\paragraph{Objective} Due to the many terms introduced in the analysis
we provide in this section a clarification of the final terms in the regret.
First of all we recall the decomposition introduced in
the Section~\ref{sec:NS-lb-sda} to control the number of pulls
of a suboptimal arm during a phase $\phi \in [1, \Gamma_T]$,
\begin{align*}
	\bE[N_k^\phi] &\leq 2 \tau + \frac{\delta_\phi A_k^{\phi,\tau}}{\tau}
	+ \bE [c_{k,1}^{\phi, \tau}]
	+ \bE[c_{k,2}^{\phi, \tau}] + \bE[c_{k,3}^{\phi, \tau}] \;.
\end{align*}

\paragraph{Results of Section~\ref{app::NS_lbsda}}
We first provide the results we obtained in Appendix~\ref{app::NS_lbsda},
that are true for any value of the sliding window $\tau$ and the function
$A_k^{\phi,\tau}$, that we will properly calibrate later.
We also recall that for any sub-optimal arm $k$ in a phase $\phi$
we defined a constant $\omega_k^\phi$ (written $\omega_k$ in
the proof as the phase is explicit), satisfying
$\omega_k^\phi= \min\left(\kl\left(\mu_k^\phi,
\frac{1}{2}(\mu_k^\phi + \mu_{\best}^\phi)\right),
\kl\left(\mu_{\best}^\phi, \frac{1}{2}(\mu_k^\phi
+ \mu_{\best}^\phi)\right) \right)$.

We first obtained an upper bound on $\bE [c_{k, 1}^{\phi,\tau}]$,
which controls the probability that a "concentrated" suboptimal arm $k$ is pulled
when the best one is leader, and $\bE [c_{k, 2}^{\phi,\tau}]$,
that represents the expectation of the number of pulls
of the arm $k$ because of the diversity flag when the best arm is leader.
These upper bounds are
\begin{align*}
	 \bE[c_{k, 1}^{\phi,\tau}]
	 \leq \delta_\phi (\tau+1)\frac{e^{-A_k^{\phi,\tau} \omega_k}}{1-e^{-\omega_k}} \;,
	 \quad
	 \bE[c_{k, 2}^{\phi,\tau}]
	 \leq \delta_\phi (\tau+1) \sum_{k' \neq \best}\frac{e^{-(\log \tau)^2 \omega_{k'}}}{1-e^{-\omega_{k'}}}\;.
\end{align*}

We then provided an upper bound of $\bE [c_{k, 3}^{\phi, \tau}]$
composed of multiple terms.
This is because this term represents the expectation of the number of rounds
 when the best arm is not leader.
 To provide a general overview, this term is composed
 of two parts:
 the first one for the cases when the best arm \textit{has already}
 been leader in the last $\tau$ rounds,
 and the case when the best arm \textit{has never} been leader in the last $\tau$ round.
 The first general scenario was handled by the constants $C_1$, $C_2$ and $C_3$,
 that we upper bounded in expectation by,
\[\bE\left[C_1 + C_2 + C_3\right] \leq 3 \delta_\phi \tau^2 (\tau+1)(K-1)\sum_{k' \neq \best} \frac{e^{-\left\lfloor \frac{\tau}{2K(K-1)}\omega_{k'} \right\rfloor }}{1-e^{-\omega_{k'}}} \;.\]
We observe that this term has a larger order in $\tau$
than the previous one before the exponential, but as a larger term
in the exponential that compensates.
After that, we handled the cases when the best arm has
never been leader in . We distinguish again different cases.
The terms $D_1$ and $D_3$ provide terms that share similar
order with the ones we obtained before, namely:
\begin{align*}
	D_1 \leq 2 \delta_\phi (\tau+1) \sum_{k' \neq \best}\frac{e^{-A_{k'}^{\phi,\tau} \omega_{k'}}}{1-e^{-\omega_{k'}}}  \quad \text{ and} \quad
	D_3 \leq  \delta_\phi \tau (\tau+1) e^{-(\log \tau)^2 \omega_k}
\end{align*}
The last term is the one that corresponds to the balance condition in the stationary case. Its adaptation to the non-stationary case was non trivial but we could provide an upper bound, leveraging on the properties detailed in Appendix~\ref{app::proof_s}. We obtained
\[D_4 \leq 2 \delta_\phi A_{\best}^{\phi,\tau} (K-1) \left[\frac{e^{-\sqrt{\log \tau} \omega^\phi}}{1-e^{-\omega^\phi}}\frac{3 \log \tau}{C^\tau} + \frac{\gamma A_{\best}^{\phi,\tau}}{\tau^3} \right] \;,\]
where $C^\tau= \left\lfloor
\frac{\beta \tau}{2(K-1)^2 (\log \tau)^2 (A_{\best}^{\phi,\tau})^2} \right\rfloor$.

\paragraph{Tuning of the parameters} The previous results allow to control precisely the dynamic regret of SW-LB-SDA for general values of $\tau$ and the constants of the problem. We first remark that one could tune each of the constants $A_{\best}^{\phi, \tau}$ to optimize the term in each phase. However, in this paragraph we propose a more general asymptotic analysis that proves that an optimal tuning of $\tau$ allows the algorithm to reach optimal guarantees. To catch this generality we will simply define $A_{\best}^{\phi, \tau}= A(\tau) = B \log \tau$ for some constant $B$, and define $\omega = \min_{\phi \in [1, \Gamma_T]} \{\min_{k \neq \best} \omega_k^\phi\}$. With these new definitions we can regroup several terms together, and obtain for $\tau>K$
\begin{align*}
\bE[N_k^\phi] \leq & 2 \tau
+ \frac{\delta_\phi A(\tau)}{\tau} +
\frac{2\delta_\phi(\tau+1)K}{1-e^{-\omega}} e^{-A(\tau) \omega}
+ \frac{K\delta_\phi \tau (\tau+1)}{1-e^{-\omega}}e^{-(\log \tau)^2 \omega} \\
&
+ 3 \delta_\phi \tau^2 (\tau+1)(K-1)^2
\frac{e^{-\left\lfloor \frac{\tau}{2K(K-1)}\omega \right\rfloor }}{1-e^{-\omega}}
+ 2 \delta_\phi A(\tau) (K-1)
\left[\frac{e^{-\sqrt{\log \tau} \omega}}{1-e^{-\omega}}\frac{3 \log \tau}{C^\tau}
+ \frac{\gamma A(\tau)}{\tau^3} \right]
\end{align*}
As the only term that depends on the phase is $\delta_\phi$
it is now straightforward to sum on the phases and the arms to obtain the dynamic regret,
recalling that $\sum_{\phi=1}^{\Gamma_T} \delta_\phi = T$.
Without loss of generality, we also assume that for all $\phi$ and for all
$k \neq \best$, $\Delta_{k}^\phi \leq 1$.

%
\begin{align*}
	\cR_T =& \sum_{\phi=1}^{\Gamma_T} \sum_{k \neq \best} \bE[N_k^\phi] \Delta_k^\phi \\
	\leq &\underbrace{ 2 (K-1) \tau \Gamma_T
	+ \frac{(K-1) T A(\tau)}{\tau}}_{E_1} +
	\underbrace{\frac{2T(\tau+1)K(K-1)}{1-e^{-\omega}} e^{-A(\tau) \omega}}_{E_2} \\
	& + \underbrace{\frac{T K (K-1) \tau (\tau+1)}{1-e^{-\omega}}e^{-(\log \tau)^2 \omega}}_{E_3}
	+ \underbrace{\frac{3 T(K-1) \tau^2 (\tau+1)(K-1)^2}{1-e^{-\omega}} e^{-\left\lfloor \frac{\tau}{2K(K-1)}\omega \right\rfloor }}_{E_4} \\
	& +  \underbrace{2 T A(\tau) (K-1)^2
	\left[\frac{e^{-\sqrt{\log \tau} \omega}}{1-e^{-\omega}} \frac{3 \log \tau}{C^\tau} + (K-1)\frac{\gamma A(\tau)}{\tau^3} \right]}_{E_5}
	\end{align*}

Knowing the horizon $T$ and an order of the number of breakpoints $\Gamma_T$ we propose a tuning for $\tau$ in $\sqrt{\frac{T\log T}{\Gamma_T}}$. We then prove that the only first order terms in the decomposition are the terms in $E_1$.

First, as $\log \tau$ is of order $\log T$, choosing $A(\tau) = \frac{6}{\omega}\log \tau$ ensures that $E_2$ is upper bounded by a constant. Then, the terms $E_3$ and $E_4$ are also both upper bounded by constants as the term in the exponent dominates the polynomial in $\tau$ before it. The term $E_5$ is a bit more touchy. Indeed, its second component causes no difficulty and is upper bounded by a constant. However, for the first term we need to use the fact $C^\tau$ is of order $\tau/\log(\tau)^j$, hence there exists some integer $j'$ such that the dominant term in $E_5 $ is of order $\frac{T}{\tau}\times (\log \tau)^{j'} e^{-\sqrt{\log \tau} \omega}$. As in Appendix~\ref{app::proof_s} we use that $(\log \tau)^{j'} e^{-\sqrt{\log \tau} \omega}=o(\log(\tau)^{-1})$ (for instance). Hence, thanks to the log terms $E_5$ is of lower order than $E_1$. Finally, we obtain
$$
\cR_T =O(\sqrt{T\Gamma_T \log T}) \;.
$$
This concludes the proof of Theorem~\ref{th::regret_SWLB}.

%% file: main.bbl
\begin{thebibliography}{34}
\providecommand{\natexlab}[1]{#1}
\providecommand{\url}[1]{\texttt{#1}}
\expandafter\ifx\csname urlstyle\endcsname\relax
  \providecommand{\doi}[1]{doi: #1}\else
  \providecommand{\doi}{doi: \begingroup \urlstyle{rm}\Url}\fi

\bibitem[Agrawal \& Goyal(2012)Agrawal and Goyal]{agrawal2012analysis}
Agrawal, S. and Goyal, N.
\newblock Analysis of thompson sampling for the multi-armed bandit problem.
\newblock In \emph{Conference on learning theory}, pp.\  39--1. JMLR Workshop
  and Conference Proceedings, 2012.

\bibitem[Agrawal \& Goyal(2013)Agrawal and Goyal]{agrawal2013further}
Agrawal, S. and Goyal, N.
\newblock Further optimal regret bounds for thompson sampling.
\newblock In \emph{Artificial intelligence and statistics}, pp.\  99--107.
  PMLR, 2013.

\bibitem[Auer et~al.(2002)Auer, Cesa-Bianchi, and Fischer]{auer2002finite}
Auer, P., Cesa-Bianchi, N., and Fischer, P.
\newblock Finite-time analysis of the multiarmed bandit problem.
\newblock \emph{Machine learning}, 47, 2002.

\bibitem[Auer et~al.(2019)Auer, Gajane, and Ortner]{auer2019adaptively}
Auer, P., Gajane, P., and Ortner, R.
\newblock Adaptively tracking the best bandit arm with an unknown number of
  distribution changes.
\newblock In \emph{Conference on Learning Theory}, pp.\  138--158, 2019.

\bibitem[Baransi et~al.(2014)Baransi, Maillard, and Mannor]{baransi2014sub}
Baransi, A., Maillard, O.-A., and Mannor, S.
\newblock Sub-sampling for multi-armed bandits.
\newblock In \emph{Joint European Conference on Machine Learning and Knowledge
  Discovery in Databases}, pp.\  115--131. Springer, 2014.

\bibitem[Baudry et~al.(2020)Baudry, Kaufmann, and Maillard]{baudry2020sub}
Baudry, D., Kaufmann, E., and Maillard, O.-A.
\newblock Sub-sampling for efficient non-parametric bandit exploration.
\newblock \emph{Advances in Neural Information Processing Systems}, 33, 2020.

\bibitem[Bergemann \& V{\"a}lim{\"a}ki(1996)Bergemann and
  V{\"a}lim{\"a}ki]{bergemann1996learning}
Bergemann, D. and V{\"a}lim{\"a}ki, J.
\newblock Learning and strategic pricing.
\newblock \emph{Econometrica: Journal of the Econometric Society}, pp.\
  1125--1149, 1996.

\bibitem[Besbes et~al.(2014)Besbes, Gur, and Zeevi]{besbes2014stochastic}
Besbes, O., Gur, Y., and Zeevi, A.
\newblock Stochastic multi-armed-bandit problem with non-stationary rewards.
\newblock In \emph{Advances in neural information processing systems}, pp.\
  199--207, 2014.

\bibitem[Besson et~al.(2020)Besson, Kaufmann, Maillard, and
  Seznec]{besson2020efficient}
Besson, L., Kaufmann, E., Maillard, O.-A., and Seznec, J.
\newblock Efficient change-point detection for tackling piecewise-stationary
  bandits.
\newblock Prepint, December 2020.

\bibitem[Cao et~al.(2019)Cao, Wen, Kveton, and Xie]{cao2019nearly}
Cao, Y., Wen, Z., Kveton, B., and Xie, Y.
\newblock Nearly optimal adaptive procedure with change detection for
  piecewise-stationary bandit.
\newblock In \emph{The 22nd International Conference on Artificial Intelligence
  and Statistics}, pp.\  418--427. PMLR, 2019.

\bibitem[Capp{\'e} et~al.(2013)Capp{\'e}, Garivier, Maillard, Munos, Stoltz,
  et~al.]{cappe2013kullback}
Capp{\'e}, O., Garivier, A., Maillard, O.-A., Munos, R., Stoltz, G., et~al.
\newblock Kullback--leibler upper confidence bounds for optimal sequential
  allocation.
\newblock \emph{The Annals of Statistics}, 41\penalty0 (3):\penalty0
  1516--1541, 2013.

\bibitem[Chan(2020)]{chan2020multi}
Chan, H.~P.
\newblock The multi-armed bandit problem: An efficient nonparametric solution.
\newblock \emph{The Annals of Statistics}, 48\penalty0 (1):\penalty0 346--373,
  2020.

\bibitem[Chen et~al.(2019)Chen, Lee, Luo, and Wei]{chen2019new}
Chen, Y., Lee, C.-W., Luo, H., and Wei, C.-Y.
\newblock A new algorithm for non-stationary contextual bandits: Efficient,
  optimal and parameter-free.
\newblock In \emph{Conference on Learning Theory}, pp.\  696--726. PMLR, 2019.

\bibitem[Eliashberg \& Jeuland(1986)Eliashberg and
  Jeuland]{eliashberg1986impact}
Eliashberg, J. and Jeuland, A.~P.
\newblock The impact of competitive entry in a developing market upon dynamic
  pricing strategies.
\newblock \emph{Marketing Science}, 5\penalty0 (1):\penalty0 20--36, 1986.

\bibitem[Garivier \& Moulines(2008)Garivier and Moulines]{garivier2008upper}
Garivier, A. and Moulines, E.
\newblock On upper-confidence bound policies for non-stationary bandit
  problems.
\newblock \emph{arXiv preprint arXiv:0805.3415}, 2008.

\bibitem[Garivier \& Moulines(2011)Garivier and Moulines]{garivier2011upper}
Garivier, A. and Moulines, E.
\newblock On upper-confidence bound policies for switching bandit problems.
\newblock In \emph{International Conference on Algorithmic Learning Theory},
  pp.\  174--188. Springer, 2011.

\bibitem[Gorre et~al.(2001)Gorre, Mohammed, Ellwood, Hsu, Paquette, Rao, and
  Sawyers]{gorre2001clinical}
Gorre, M.~E., Mohammed, M., Ellwood, K., Hsu, N., Paquette, R., Rao, P.~N., and
  Sawyers, C.~L.
\newblock Clinical resistance to sti-571 cancer therapy caused by bcr-abl gene
  mutation or amplification.
\newblock \emph{Science}, 293\penalty0 (5531):\penalty0 876--880, 2001.

\bibitem[Honda \& Takemura(2015)Honda and Takemura]{Honda15IMED}
Honda, J. and Takemura, A.
\newblock Non-asymptotic analysis of a new bandit algorithm for semi-bounded
  rewards.
\newblock \emph{Journal of Machine Learning Research}, 16:\penalty0 3721--3756,
  2015.

\bibitem[Kaufmann et~al.(2012)Kaufmann, Korda, and Munos]{TS_Emilie}
Kaufmann, E., Korda, N., and Munos, R.
\newblock Thompson sampling: An asymptotically optimal finite-time analysis.
\newblock In \emph{Algorithmic Learning Theory - 23rd International Conference,
  ALT}, 2012.

\bibitem[Kveton et~al.(2019{\natexlab{a}})Kveton, Szepesvari, Ghavamzadeh, and
  Boutilier]{kveton2019perturbed}
Kveton, B., Szepesvari, C., Ghavamzadeh, M., and Boutilier, C.
\newblock Perturbed-history exploration in stochastic multi-armed bandits.
\newblock \emph{arXiv preprint arXiv:1902.10089}, 2019{\natexlab{a}}.

\bibitem[Kveton et~al.(2019{\natexlab{b}})Kveton, Szepesvari, Vaswani, Wen,
  Lattimore, and Ghavamzadeh]{kveton2019garbage}
Kveton, B., Szepesvari, C., Vaswani, S., Wen, Z., Lattimore, T., and
  Ghavamzadeh, M.
\newblock Garbage in, reward out: Bootstrapping exploration in multi-armed
  bandits.
\newblock In \emph{International Conference on Machine Learning}, pp.\
  3601--3610. PMLR, 2019{\natexlab{b}}.

\bibitem[Lai \& Robbins(1985)Lai and Robbins]{lai1985asymptotically}
Lai, T.~L. and Robbins, H.
\newblock Asymptotically efficient adaptive allocation rules.
\newblock \emph{Advances in applied mathematics}, 6\penalty0 (1):\penalty0
  4--22, 1985.

\bibitem[Li et~al.(2011)Li, Chu, Langford, and Wang]{li2011unbiased}
Li, L., Chu, W., Langford, J., and Wang, X.
\newblock Unbiased offline evaluation of contextual-bandit-based news article
  recommendation algorithms.
\newblock In \emph{Proceedings of the fourth ACM international conference on
  Web search and data mining}, pp.\  297--306, 2011.

\bibitem[Li et~al.(2016)Li, Karatzoglou, and Gentile]{li2016collaborative}
Li, S., Karatzoglou, A., and Gentile, C.
\newblock Collaborative filtering bandits.
\newblock In \emph{Proceedings of the 39th International ACM SIGIR conference
  on Research and Development in Information Retrieval}, pp.\  539--548, 2016.

\bibitem[Liu et~al.(2017)Liu, Lee, and Shroff]{liu2017change}
Liu, F., Lee, J., and Shroff, N.
\newblock A change-detection based framework for piecewise-stationary
  multi-armed bandit problem.
\newblock \emph{arXiv preprint arXiv:1711.03539}, 2017.

\bibitem[Raj \& Kalyani(2017)Raj and Kalyani]{raj2017taming}
Raj, V. and Kalyani, S.
\newblock Taming non-stationary bandits: A bayesian approach.
\newblock \emph{arXiv preprint arXiv:1707.09727}, 2017.

\bibitem[Riou \& Honda(2020)Riou and Honda]{riou2020bandit}
Riou, C. and Honda, J.
\newblock Bandit algorithms based on thompson sampling for bounded reward
  distributions.
\newblock In \emph{Algorithmic Learning Theory}, pp.\  777--826. PMLR, 2020.

\bibitem[Seznec et~al.(2020)Seznec, Menard, Lazaric, and
  Valko]{seznec2020single}
Seznec, J., Menard, P., Lazaric, A., and Valko, M.
\newblock A single algorithm for both restless and rested rotting bandits.
\newblock In \emph{International Conference on Artificial Intelligence and
  Statistics}, pp.\  3784--3794. PMLR, 2020.

\bibitem[Thompson(1933)]{thompson1933likelihood}
Thompson, W.~R.
\newblock On the likelihood that one unknown probability exceeds another in
  view of the evidence of two samples.
\newblock \emph{Biometrika}, 25\penalty0 (3/4):\penalty0 285--294, 1933.

\bibitem[Trovo et~al.(2020)Trovo, Paladino, Restelli, and
  Gatti]{trovo2020sliding}
Trovo, F., Paladino, S., Restelli, M., and Gatti, N.
\newblock Sliding-window thompson sampling for non-stationary settings.
\newblock \emph{Journal of Artificial Intelligence Research}, 68:\penalty0
  311--364, 2020.

\bibitem[Vermorel \& Mohri(2005)Vermorel and Mohri]{vermorel2005multi}
Vermorel, J. and Mohri, M.
\newblock Multi-armed bandit algorithms and empirical evaluation.
\newblock In \emph{European conference on machine learning}, pp.\  437--448.
  Springer, 2005.

\bibitem[Wu et~al.(2018)Wu, Iyer, and Wang]{wu2018learning}
Wu, Q., Iyer, N., and Wang, H.
\newblock Learning contextual bandits in a non-stationary environment.
\newblock In \emph{The 41st International ACM SIGIR Conference on Research \&
  Development in Information Retrieval}, pp.\  495--504, 2018.

\bibitem[Yue \& Joachims(2009)Yue and Joachims]{yue2009dueling}
Yue, Y. and Joachims, T.
\newblock Interactively optimizing information retrieval systems as a dueling
  bandits problem.
\newblock In \emph{Proceedings of the 26th Annual International Conference on
  Machine Learning}. ACM, 2009.

\bibitem[Zelen(1969)]{zelen1969play}
Zelen, M.
\newblock Play the winner rule and the controlled clinical trial.
\newblock \emph{Journal of the American Statistical Association}, 64\penalty0
  (325):\penalty0 131--146, 1969.

\end{thebibliography}
